\documentclass{article}

\usepackage[preprint]{main}

\usepackage[utf8]{inputenc} 
\usepackage[T1]{fontenc}    
\usepackage{hyperref}       
\usepackage{url}            
\usepackage{booktabs}       
\usepackage{amsfonts}       
\usepackage{nicefrac}       
\usepackage{microtype}      
\usepackage{xcolor}         
\usepackage{amsmath}
\usepackage{amssymb}
\usepackage{mathtools}
\usepackage{amsthm}
\usepackage{algorithm, algpseudocode}
\usepackage{titletoc}
\usepackage{wrapfig}
\usepackage{multirow}
\usepackage{xspace}
\usepackage{thmtools}
\usepackage{thm-restate}

\theoremstyle{plain}

\newtheorem{lemma}{Lemma}

\newtheorem*{proposition*}{Proposition}

\newcommand{\Ex}{\mathbb{E}}
\newcommand{\KL}{\mathrm{KL}}

\newcommand{\x}{\mathbf{x}}

\newcommand{\btheta}{\boldsymbol{\theta}}
\newcommand{\bepsilon}{\boldsymbol{\varepsilon}}

\newcommand{\bdelta}{\boldsymbol{\delta}}

\renewcommand{\hat}{\widehat}
\newcommand{\Prob}{\mathbb{P}}

\newcommand{\Normal}{\mathcal{N}}

\newcommand{\R}{\mathbb{R}}            

\newcommand{\ours}{\textsc{RRISE}\xspace}

\newcommand{\JS}{\operatorname{JS}}

\newcommand{\diag}{\operatorname{diag}}

\newcommand{\bone}{\mathbf{1}}
\newcommand{\bzero}{\mathbf{0}}
\newcommand{\bI}{\mathbf{I}}

\newcommand{\pA}{p_A}
\newcommand{\qA}{q_A}
\newcommand{\ptildeA}{\widetilde p_A}
\newcommand{\Rtilde}{\widetilde R}

\title{\ours: Robust Radius Inference \\ via a Surrogate Estimator}

\author{%
  Jong-Ik Park\thanks{Equal contribution.} \\
  Carnegie Mellon University\\
  \texttt{jongikp@andrew.cmu.edu}
  \And
  Shreyas Chaudhari\footnotemark[1] \\
  Carnegie Mellon University\\
  \texttt{shreyasc@andrew.cmu.edu}
  \AND
  Carlee Joe-Wong \\
  Carnegie Mellon University\\
  \texttt{cjoewong@andrew.cmu.edu}
  \And
  Jos\'{e} M. F. Moura \\
  Carnegie Mellon University\\
  \texttt{moura@andrew.cmu.edu}
}

\begin{document}

\maketitle

\begin{abstract}
Randomized smoothing (RS) uses a smoothed classifier to provide architecture-agnostic certificates of $\ell_2$ classification robustness, but its dependence on per-input Monte Carlo (MC) sampling undermines its use in real-time systems. We argue that this cost is structural rather than fundamental, such that it can be significantly reduced by sharing information across the deployment stream. We introduce \ours, an RS framework that compresses certification into a single forward pass through a learned surrogate. \ours trains the surrogate against precomputed MC class-count targets via a soft-label cross-entropy loss and converts surrogate predictions into provably conservative certified radii through a one-time conformal calibration step. The resulting certificate is deployment-verifiable: whenever the calibrated radius is positive, the surrogate's prediction provably matches the smoothed classifier's and the smoothed classifier is constant on a ball of that radius around the input. Across image classification benchmarks, \ours\ matches fixed-budget MC certified accuracy within $0.84$ percentage points while replacing up to $10^4$ noisy base-model evaluations per query with a single surrogate forward pass, recouping MC training cost after $\approx 10^5$ deployment queries. On CIFAR-100 and Tiny ImageNet, where the only prior offline-surrogate method collapses, \ours\ achieves $1.23$ to $1.91\times$ higher certified accuracy, establishing efficient randomized smoothing as a practical path to certified robustness in repeated-deployment settings.
\end{abstract}

\section{Introduction}
Modern AI classification systems increasingly operate in high-stakes, {real-time} settings, where performance depends not only on pointwise accuracy but also on {stability under input perturbations}~\citep{fawzi2018analysis, liu2025comprehensive}. Physically realizable perturbations---such as changes in viewpoint, lighting, or sensor noise---can, for example, trigger safety-critical failures in autonomous driving~\citep{eykholt2018robust, chi2024adversarial}, while subtle variations in medical images may compromise clinical decision-making~\citep{finlayson2019adversarial, ma2021understanding}. Similar concerns arise in real-time robotics~\citep{cao2023robust} and speech recognition~\citep{xie2020real}, where reliable decisions must be produced under strict latency constraints despite naturally occurring or adversarial input perturbations. These settings motivate a \emph{geometric} view of robustness, in which predictions should remain invariant within a neighborhood of the input, and the size of this neighborhood defines an operational safety margin~\citep{hein2017formal, wang2018efficient}. By contrast, widely used pointwise reliability measures---such as confidence scores, predictive uncertainty, and calibration metrics~\citep{guo2017calibration, lakshminarayanan2017simple, gal2016dropout, geifman2017selective}---do not directly certify neighborhood invariance.

Randomized smoothing (RS)~\citep{lecuyer2019certified, cohen2019certified, li2019certified} has emerged as a leading approach for certifying classifier robustness. RS provides instance-specific guarantees of prediction invariance under bounded perturbations. Unlike bound-propagation and convex-relaxation methods~\citep{weng2018towards, singh2019abstract} that rely on architectural assumptions and remain difficult to scale to large networks, RS is architecture-agnostic, requiring only black-box query access to the classifier and applying broadly through Monte Carlo (MC) sampling.

Despite these advantages, standard RS entails substantial computational costs, which hinder its deployment in real-time, safety-critical, risk-aware decision-making systems \cite{kumari2023rethinking}. Certification requires estimating ``smoothed'' class probabilities via MC sampling for each input~\citep{cohen2019certified}, and achieving high-confidence guarantees may require on the order of $10^5$ forward passes per input example~\citep{salman2019provably}. In latency-sensitive settings, this overhead is prohibitive. On modern GPU hardware, a single forward pass for a large RGB color image can take several milliseconds~\citep{xu2024gtp}, resulting in per-input certification times on the order of hundreds of seconds~\citep{cohen2019certified, bhardwaj2024accelerated}, far exceeding the requirements of typical latency-sensitive applications like autonomous driving or speech recognition. This gap between certified robustness guarantees and practical deployment thus motivates the development of substantially more efficient randomized smoothing certification methods.

\paragraph{Contributions.}
We introduce \ours\ (Robust Radius Inference via a Surrogate Estimator), a computationally efficient framework for randomized-smoothing certification that replaces per-input MC sampling with a single surrogate forward pass. Our contributions are twofold.

\textbf{(i) A principled surrogate-training method for computationally efficient smoothing.} We fine-tune the base classifier to predict the smoothed class distribution under Gaussian noise, supervised by soft-label cross-entropy against finite-budget MC class-count targets. Since cross-entropy is linear in its target, its gradient is an unbiased estimate of the gradient at the realized MC target. Divergence-based alternatives used in prior offline-surrogate work~\citep{bhardwaj2024accelerated} are nonlinear in their first argument and incur a curvature-induced gradient bias (Appendix~\ref{app:rrise_vs_baseline4}). Fine-tuning, rather than training from scratch, lets the surrogate inherit the noise-invariant representations the base classifier has already learned through Gaussian-noise augmentation.

\textbf{(ii) A conformal calibration layer that yields deployment-verifiable certificates.} On a held-out calibration set, we compute a single scalar offset $\delta$ that, at inference time, converts the surrogate's top-class probability into a high-probability lower bound on the smoothed top-class probability --- and thus into a certified radius computed entirely from one surrogate forward pass. When this radius is positive, the surrogate's prediction provably matches the smoothed classifier's. The standard assumption underlying amortized certification --- that the surrogate's argmax agrees with the smoothed classifier's --- becomes a condition the practitioner can check at inference time, with one calibration covering the entire deployment.

The rest of this paper is organized as follows. After giving an overview of the problem background and related work (Section~\ref{sec:background}), we present the \textsc{RRISE} methodology in Section~\ref{sec:methodology} and evaluate it in Section~\ref{sec:experiments}. We discuss potential limitations of \textsc{RRISE} in Section~\ref{sec:discussion} before concluding in Section~\ref{sec:conclusion}.
\section{Background and Related Work}
\label{sec:background}

\subsection{Preliminaries}
\label{subsec:prelim}

Randomized smoothing (RS) \citep{cohen2019certified} constructs classifiers
with provable robustness against $\ell_2$-bounded adversarial perturbations.
Unlike empirical defenses, which remain vulnerable to adaptive
attacks~\citep{carlini2017adversarial, akhtar2018threat, tramer2020adaptive},
RS yields certified guarantees that hold for \emph{any} perturbation, no matter its source, within
a prescribed radius. The core idea is to convolve a base classifier with
isotropic Gaussian noise, producing a smoothed classifier whose decision
is provably stable in a neighborhood of each input.

Let $f:\R^d \to \{1,\dots,K\}$ be a base classifier trained with a
standard supervised objective. For a smoothing parameter $\sigma>0$ and input $\x$, RS
defines the smoothed class probabilities for each class $k$:
\begin{equation}
p(k \mid \x, \sigma) \;\triangleq\;
\Prob_{\bepsilon \sim \Normal(\mathbf{0}, \sigma^2 \mathbf{I})}\!\big(f(\x + \bepsilon) = k\big),
\label{eq:smoothed_prob}
\end{equation}
and the induced \emph{smoothed classifier}
$g(\x;\sigma) \triangleq \arg\max_{k} p(k \mid \x,\sigma)$ returns the
most likely class under noise. Letting $p_A = \max_k p(k \mid \x, \sigma)$
denote the smoothed top-class probability, \citet{cohen2019certified} prove
that whenever $p_A > 1/2$, the smoothed classifier $g$ is robust within
$\ell_2$-radius
\begin{equation}
R(\x;\sigma) \;\triangleq\; \sigma\,\Phi^{-1}(p_A),
\label{eq:rs_radius}
\end{equation}
in the sense that $g(\x + \bdelta;\sigma) = g(\x;\sigma)$ for all
$\|\bdelta\|_2 \leq R(\x;\sigma)$, where $\Phi^{-1}$ is the inverse
standard Gaussian CDF. Since $p_A$ cannot be computed in closed form, the standard approach is to estimate it via Monte Carlo (MC) sampling. Drawing $n$ noise vectors
$\bepsilon_j \sim \Normal(\mathbf{0},\sigma^2\mathbf{I})$, the perturbed
inputs $\x + \bepsilon_j$ are each classified by $f$, and the
most-frequently-predicted class $\hat{c}_A$ is taken as the prediction of the smoothed classifier $g(\x;\sigma)$.
The fraction of samples voting for that class,
$\hat p_A = \tfrac{1}{n}\sum_{j=1}^{n} \mathbf{1}\{f(\x+\bepsilon_j) = \hat c_A\}$,
is an empirical estimate of $p_A$. Since $\hat p_A$ is itself noisy, a
one-sided Clopper--Pearson lower confidence bound $\underline p_A \leq \hat p_A$
is used in place of $p_A$, yielding the high probability radius
$\hat R(\x;\sigma) \triangleq \sigma\,\Phi^{-1}(\underline p_A)$.

This procedure is statistically sound and broadly applicable, but its
cost scales with the per-input MC budget. \citet{cohen2019certified} use
up to $n = 10^5$ MC samples per certified ImageNet image, amounting to
over $1{,}500$ GPU-hours\footnote{Computed from the $\sim$110\,s per-image
certification time on a single NVIDIA RTX 2080 Ti reported in
\citet{cohen2019certified}.} to certify the 50K images.
This cost is structural, as $p_A$ is estimated from scratch at every input,
with no information shared across inputs. We organize the remainder of the paper around the question this raises:
\emph{can the dependence of the certificate on $p_A$ be amortized across
inputs, so that certifying a new input $\x$ no longer requires many forward
passes through $f$?} Section~\ref{sec:methodology} answers affirmatively by
training a neural surrogate that predicts the smoothed class distribution
directly, and in particular Section~\ref{sec:calibration} shows that a one-time
conformal calibration converts surrogate predictions into certified
radii with a high-probability coverage guarantee.

\subsection{Reliability, Smoothing, and Acceleration}
\label{subsec:related_work}

\paragraph{Pointwise reliability signals.}
Calibration, predictive entropy, Bayesian approximations, ensembles, distance-aware models, selective prediction, and out-of-distribution detectors provide useful pointwise reliability information for a given classification model~\citep{gal2016dropout,guo2017calibration,lakshminarayanan2017simple,geifman2017selective,maddox2019simple,liang2018enhancing,liu2020energy,liu2023simple}. These signals can often be computed with little additional cost per input, but they do not certify neighborhood invariance and therefore do not provide an instance-specific robustness radius.

\paragraph{Certified randomized smoothing.}
Randomized smoothing has been extended beyond the original Gaussian $\ell_2$ setting to additional norms, transformations, architectures, and smoothing distributions~\citep{lecuyer2019certified,cohen2019certified,li2019certified,yang2020randomized,fischer2020certified,pfrommer2023projected}. Another line of work studies data-dependent or input-adaptive smoothing levels~\citep{alfarra2022data}. These methods improve the flexibility or quality of smoothing certificates, but the certification step still relies on expensive per-input MC estimation of the smoothed class probabilities, making them difficult to deploy for latency-sensitive applications.

\paragraph{Reducing the Monte Carlo cost.}
Several methods reduce the online sampling burden without replacing MC certification entirely. Confidence-sequence and early-stopping approaches adaptively terminate sampling once the radius estimate is sufficiently stable~\citep{voracek2024treatment}. Input-specific budgeting methods allocate fewer samples to easy inputs and more samples to ambiguous ones~\citep{seferis2024estimating}. Incremental certification methods reuse information across related classifiers~\citep{ugare2024incremental}. These approaches reduce average sampling cost but still require noisy base-model evaluations at test time. The offline surrogate approach of \citet{bhardwaj2024accelerated} is closest to ours as it also trains a surrogate on precomputed MC targets. \ours differs in two ways: it uses a cross-entropy objective whose finite-budget loss is unbiased at fixed parameters, and it adds a conformal calibration layer that converts surrogate probabilities into conservative certified radii. Appendix~\ref{app:rrise_vs_baseline4} gives a detailed comparison.

\section{Methodology}
\label{sec:methodology}

Here, we describe \ours: a computationally efficient alternative to MC-based randomized smoothing certification. At its core is a learned surrogate $q_{\btheta}$ that predicts the
smoothed class distribution from the clean input, replacing $n$ noisy forward passes through $f$ with a single forward pass
through $q_{\btheta}$. A one-time calibration procedure converts
the surrogate's predictions into certified radii with a high-probability guarantee. We proceed to describe the surrogate and its training (Section~\ref{sec:training}),
and the calibration procedure (Section~\ref{sec:calibration}).

\subsection{Training the \ours\ Surrogate}
\label{sec:training}

\ours replaces the per-input MC estimate of $p(\cdot \mid \x, \sigma)$
with a learned predictor $q_{\btheta} : \R^d \to \Delta^{K-1}$ whose
$k$-th output approximates the smoothed class probability
in~\eqref{eq:smoothed_prob}: $q_{\btheta}(\x)_k \approx p(k \mid \x, \sigma)$.
The surrogate's argmax predicts the output of the smoothed classifier, and
its top-class probabilitiy estimates $p_A$. Concretely, the surrogate's
predicted class and top-class probability are
\begin{equation}
\hat g(\x) \;\triangleq\; \arg\max_k q_{\btheta}(\x)_k,
\qquad
q_A(\x) \;\triangleq\; \max_k q_{\btheta}(\x)_k,
\label{eq:surrogate_predictions}
\end{equation}
that mirror the smoothed classifier $g(\x;\sigma)$ and top-class
probability $p_A$ in~\eqref{eq:rs_radius}, but are computed in a single
forward pass rather than from $n$ noisy evaluations of $f$. We fix
$\sigma$ throughout and treat $q_{\btheta}$ as $\sigma$-specific and the framework can be extended to the multi-$\sigma$ setting.

We train $q_{\btheta}$ on a precomputed dataset of MC targets. For each
training input $\x_i$, we draw $n$ i.i.d. (independently and identically distributed) noise samples
$\bepsilon_{i,j} \sim \Normal(\mathbf{0}, \sigma^2 \mathbf{I})$ and form
the empirical smoothed distribution $\hat p_i \in \Delta^{K-1}$ with
$\hat p_{i,k} \triangleq \tfrac{1}{n}\sum_{j=1}^{n} \mathbf{1}\{f(\x_i + \bepsilon_{i,j}) = k\}$.
The dataset $\{(\x_i, \hat p_i)\}$ is fixed and reused across epochs and
hyperparameter sweeps. We initialize $q_{\btheta}$ as $f$, and fine-tune by minimizing the soft-target cross-entropy $\widehat{\mathcal{L}}(\btheta) \;\triangleq\;
\frac{1}{|\mathcal{B}|} \sum_{i \in \mathcal{B}}
\mathsf{CE}\big(\hat p_i,\, q_{\btheta}(\x_i)\big)$
on minibatches $\mathcal{B}$. 

Since $\hat p_i$ is itself an empirical
estimate of $p(\cdot \mid \x_i, \sigma)$, training against it incurs an
$O(1/\sqrt n)$ estimation bias in the target that is easily controllable. In Appendix~\ref{app:ablation_list} we report how the surrogate performance fluctuates with the $n$ used for dataset construction, and in Appendix~\ref{app:rrise_theoretical_strength} and~\ref{app:rrise_vs_baseline4} we show that cross entropy loss provides unbiased gradients with respect to the estimated distribution $\widehat{p}$ whereas alternative divergences used by existing approaches do not. In our experiments, $q_{\btheta}$ uses the same architecture as the base classifier, is initialized from the base classifier, and trains only the estimator head; Appendix~\ref{app:ablation_list} ablates this choice against end-to-end and random-initialized variants.

\subsection{From Surrogate Predictions to Certified Radii}
\label{sec:calibration}

We now turn the surrogate's predictions into certified radii with a
provable lower-bound guarantee. A single forward pass through
$q_{\btheta}$ yields the predicted class $\hat g(\x)$ and top
probability $q_A(\x)$ in~\eqref{eq:surrogate_predictions}. Since the
true radius $R(\x;\sigma) = \sigma\,\Phi^{-1}(p_A(\x))$ is monotone in
the smoothed top probability $p_A(\x)$, any high-probability lower
bound on $p_A(\x)$ immediately yields a high-probability lower bound on
$R(\x;\sigma)$. We thus reduce the certification problem to
lower-bounding $p_A(\x)$ from $q_A(\x)$, and address it with a one-time
conformal calibration~\citep{shafer2008tutorial,lei2018distribution}
that determines a single parameter $\delta \geq 0$ on a held-out set such that
$q_A(\x) - \delta \leq p_A(\x)$ with high probability. The calibrated
radius
\begin{equation}
\widetilde R(\x;\sigma) \;\triangleq\; \sigma\,\Phi^{-1}\!\big(q_A(\x) - \delta\big)
\end{equation}
is then a lower bound on $R(\x;\sigma)$, computed from the same forward
pass that produces $\hat g(\x)$. As a result, certification at deployment is label-free, MC-free, and requires no test-time sampling.

The key observation enabling this reduction is that, by definition of the smoothed argmax, $p_A(\x) \geq p(\hat g(\x) \mid \x, \sigma)$ for
any $\x$, with equality whenever $\hat g(\x) = g(\x;\sigma)$.
Calibrating against the surrogate's own argmax $\hat g$ therefore always yields a valid lower bound on $p_A(\x)$. 

\begin{restatable}[Calibrated lower bound on the smoothed top probability]{proposition}{calibrationprop}
\label{prop:calibration}
Fix confidence parameters $\beta, \gamma \in (0,1)$. Let
$\{\x_i^{\mathrm{cal}}\}_{i=1}^M$ be a calibration set drawn i.i.d.\
from the test distribution, disjoint from surrogate training. For each
calibration point, draw $n$ noise samples
$\bepsilon_{i,j} \sim \Normal(\mathbf{0}, \sigma^2 \mathbf{I})$ and let
$\underline p_i$ be the one-sided Clopper--Pearson lower bound at
confidence $1-\beta$ on
$p(\hat g(\x_i^{\mathrm{cal}}) \mid \x_i^{\mathrm{cal}}, \sigma)$.
Define residuals $r_i = q_A(\x_i^{\mathrm{cal}}) - \underline p_i$ and
set $\delta$ to the $\lceil(M+1)(1-\gamma)\rceil$-th smallest of
$r_1,\dots,r_M$. Then for an independent test point $\x$,
\begin{equation}
\Prob\big[\,p_A(\x) \geq q_A(\x) - \delta\,\big]
\;\geq\;
1 - \beta - \gamma.
\label{eq:coverage}
\end{equation}
\end{restatable}

\begin{proof}[Proof sketch]
We have
$p(\hat g(\x) \mid \x, \sigma) \leq p_A(\x)$ for any $\x$. A
Clopper--Pearson lower bound $\underline p_{\mathrm{test}}$ on
$p(\hat g(\x) \mid \x, \sigma)$ from $n$ fresh noise samples satisfies
$\underline p_{\mathrm{test}} \leq p(\hat g(\x) \mid \x, \sigma)$ with
probability $\geq 1-\beta$. Since calibration and test residuals are
exchangeable, the conformal guarantee gives
$q_A(\x) - \delta \leq \underline p_{\mathrm{test}}$ with probability
$\geq 1-\gamma$. A union bound chains these as
$q_A(\x) - \delta \leq p(\hat g(\x) \mid \x, \sigma) \leq p_A(\x)$ with
probability $\geq 1-\beta-\gamma$. The full proof is in
Appendix~\ref{app:calibration_proof}.
\end{proof}

In Proposition~\ref{prop:calibration}, $\beta$ controls the
Clopper--Pearson noise from calibration sampling and $\gamma$ the
conformal slack absorbing the surrogate's over-prediction; their sum
$\beta + \gamma$ is the total miscoverage budget. Any allocation with
$\beta + \gamma = \alpha$ yields a $(1-\alpha)$-confidence lower bound
on $R(\x;\sigma)$ from a single forward pass through $q_{\btheta}$, in contrast to direct MC-based certification, which scales the
per-input noise budget $n$ to reach the same confidence at deployment. 

Moreover, the proposition certifies that with high probability, the smoothed classifier $g(\cdot;\sigma)$
is constant on the ball of radius $\widetilde R(\x;\sigma)$ around $\x$.
Deployment, meanwhile, returns the surrogate's prediction $\hat g(\x)$
as a fast stand-in for $g(\x;\sigma)$.
To bridge the two, we can consider the
argmax-agreement event $E(\x) \;\triangleq\; \{\hat g(\x) = g(\x;\sigma)\}$. A na\"ive guarantee for $\hat g$ would \emph{assume} $E(\x)$ at the
center and conclude that $\hat g(\x)$ inherits the smoothed
classifier's certificate. Such an assumption unfortunately cannot be verified at
deployment as checking $\hat g(\x) = g(\x;\sigma)$ requires evaluating
$g(\x;\sigma)$, which is precisely the expensive MC computation we
seek to avoid. The following corollary inverts this dependence, demonstrating that $E(\x)$ becomes a \emph{consequence} of a deployment-observable
condition, namely, the certified radius being positive.

\begin{restatable}[Surrogate prediction matches the smoothed classifier on positive radii]{corollary}{argmaxcorollary}
\label{cor:argmax_agreement}
Under the conditions of Proposition~\ref{prop:calibration}, if
$q_A(\x) - \delta > 1/2$, then with probability $\geq 1 - \beta - \gamma$
the surrogate's prediction $\hat g(\x)$ coincides with the smoothed
classifier's prediction $g(\x;\sigma)$, and $g(\cdot;\sigma)$ is
constant with value $\hat g(\x)$ on the ball
$\{\x' : \|\x' - \x\|_2 \leq \widetilde R(\x;\sigma)\}$.
\end{restatable}

Together, the clauses of the corollary turn the surrogate into a deployable
certified classifier. The first clause --- prediction-match at the
center --- closes the gap between Proposition~\ref{prop:calibration}
(which certifies $g(\cdot;\sigma)$) and deployment (which queries
$\hat g$). When the radius is positive, the surrogate's prediction at
$\x$ is provably the same prediction the smoothed classifier would
have made. The second clause --- constancy on the ball with value
$\hat g(\x)$ --- is the standard randomized-smoothing
certificate~\citep{cohen2019certified}, now anchored to the value
returned by the single forward pass of the surrogate.

\section{Experimental Evaluation}
\label{sec:experiments}

\paragraph{Evaluation Goals and Baselines}
\label{subsec:exp_goals}
We organize the evaluation around three questions that are answered explicitly in the results below.
\textbf{Q1: Certified accuracy.} Does \ours\ preserve the certified accuracy of fixed-budget MC randomized smoothing while replacing repeated noisy evaluations with one surrogate forward pass at inference?
\textbf{Q2: Boundary-radius reliability.} After calibration, does \ours\ avoid inflated radii in the boundary-confidence regime, where small probability errors can change whether an input is certified?
\textbf{Q3: Computational break-even.} After accounting for offline target construction and surrogate training, how many deployment queries are required before \ours\ becomes cheaper than MC-based certification?
We compare against four baselines.

\textbf{Baseline~1} is fixed-budget MC randomized smoothing~\citep{cohen2019certified}, which draws \(n\) noisy samples per input and certifies via a one-sided Clopper--Pearson lower bound.
\textbf{Baseline~2} is an input-specific sample-budgeting method following \citet{seferis2024estimating}, using a pilot estimate and budget-mapping rule to reduce noisy evaluations while still certifying from realized count evidence.
\textbf{Baseline~3} is a budget-prediction and early-stopping method inspired by \citet{voracek2024treatment}; the stopping rule is adaptive, but the final radius is again computed from a Clopper--Pearson lower bound.
\textbf{Baseline~4} is the offline Jensen--Shannon divergence surrogate of \citet{bhardwaj2024accelerated}, which shares \ours's MC class-count targets but has no calibration procedure. We equip it with \ours's conformal calibration for a fair radius comparison.

These baselines span the relevant design space: Baseline~1 tests whether \ours\ preserves the reference MC certificate; Baselines~2 and~3 test whether full surrogate inference provides additional savings beyond adaptive MC sampling; and Baseline~4 tests whether the calibrated surrogate design in Section~\ref{sec:methodology} improves over an offline surrogate trained from MC class-count targets.
Unless otherwise stated, all methods use an MC budget of \(n=10{,}000\), with Baselines~2--3 in their tighter \(1\%\) configuration: Baseline~2 uses decline level \(0.01\), and Baseline~3 uses stopping tolerance \(0.01\). These hyperparameters make the adaptive MC baselines more conservative by allowing less approximation slack before continuing sampling.
\ours\ is initialized from the base classifier with only the prediction head trained; Baseline~4 uses random initialization per its original setting.
Appendix~\ref{app:additional_results} ablates MC budget, training strategy, calibration level, and baseline hyperparameters. 

\paragraph{Experimental Setup}
\label{subsec:experimental_setup}
We evaluate on FashionMNIST, CIFAR-10, CIFAR-100, and Tiny ImageNet using MLP-Mixer-Tiny, ResNet-18 with a CIFAR-style stem, EfficientNet-B0, and ViT-Tiny, respectively; the \ours\ surrogate inherits each base architecture.
Following standard randomized-smoothing practice, base classifiers are trained with Gaussian noise augmentation at level \(\sigma_{\mathrm{base}}\), and certification, as well as surrogate-target construction, uses smoothing level \(\sigma\):
\[
(\sigma_{\mathrm{base}},\sigma)=(0.5,0.25)
\quad\text{on FashionMNIST and CIFAR-10,}
\]
and
\[
(\sigma_{\mathrm{base}},\sigma)=(0.25,0.10)
\quad\text{on CIFAR-100 and Tiny ImageNet.}
\]
The \ours\ offline target dataset stores the normalized class-count vector obtained by evaluating the frozen base classifier under noisy perturbations of each training input.
The surrogate is trained on clean inputs with the cross-entropy objective from Section~\ref{sec:methodology}; model selection follows the cross-validation procedure in Appendix~\ref{app:implementation_details}.

\paragraph{Calibration and Confidence Matching}
\label{subsec:exp_calibration}
For MC baselines, certificates use one-sided Clopper--Pearson lower bounds at failure level \(\alpha_{\mathrm{MC}}\).
For \ours\ and Baseline~4, we use the calibration procedure of Section~\ref{sec:calibration}.
A \(10\%\) calibration split estimates the scalar offset \(\delta\), and the same offset is used for reporting.
In the main comparison, \(\alpha_{\mathrm{MC}}=0.25\) for Baselines~1--3 and \(\beta_{\mathrm{sur}}=0.001\), \(\gamma_{\mathrm{sur}}=0.249\) for surrogate methods, so \(\beta_{\mathrm{sur}}+\gamma_{\mathrm{sur}}=0.25\).
Appendix~\ref{app:additional_results} includes stricter surrogate calibration levels \(\beta_{\mathrm{sur}}+\gamma_{\mathrm{sur}}\in\{0.10,0.05,0.01\}\). Throughout this section, we denote $\widetilde{p}_A(\mathbf{x})$ the method-specific lower bound on the smoothed top-class probability used for certification: the one-sided Clopper--Pearson lower bound for MC baselines, and the calibrated quantity $q_A(\mathbf{x}) - \delta$ for RRISE and Baseline 4. 

\paragraph{Metrics}
\label{subsec:exp_metrics}
All results are reported as mean \(\pm\) standard deviation over seeds \(\{100,200,300\}\).
Let \(\hat y(\x)\) denote the method's predicted class, \(y(\x)\) the ground-truth label, \(\ptildeA(\x)\) the method-specific lower top-probability estimate, and \(\Rtilde(\x)\) the corresponding calibrated or MC-certified radius.
Certified accuracy at threshold \(r\) is
\begin{equation}
 \mathrm{CertAcc}(r)
 =
 \frac{1}{|\mathcal D_{\mathrm{test}}|}
 \sum_{\x\in\mathcal D_{\mathrm{test}}}
 \bone\!\left[
 \hat y(\x)=y(\x),\
 \ptildeA(\x)>\frac12,\
 \Rtilde(\x)\ge r
 \right].
 \label{eq:certacc_metric}
\end{equation}
CertAcc@0 is the fraction of test inputs that are simultaneously correctly classified and certified with positive lower top probability; higher values indicate that the method certifies more of the test set with a non-trivial radius.

To study behavior near the certification boundary, we consider the boundary-confidence subset
\begin{equation}
\mathcal B
\triangleq
\{\x\in\mathcal D_{\mathrm{test}}:0.5<\ptildeA(\x)<0.75\}.
\end{equation}
Boundary Mass is \(|\mathcal B|/|\mathcal D_{\mathrm{test}}|\), the fraction of test inputs the method places in the diagnostic region just above the certification threshold.
A method that places too few inputs in \(\mathcal B\) may be over-confident, while a method that places too many may be under-confident; therefore boundary mass should be read jointly with average radius and CertAcc.
OCA denotes ordinary classification accuracy on the specified subset; in the boundary tables it is
\begin{equation}
 \mathrm{OCA}(\mathcal B)
 =
 \frac{1}{|\mathcal B|}
 \sum_{\x\in\mathcal B}
 \bone[\hat y(\x)=y(\x)],
\end{equation}
when \(|\mathcal B|>0\), and is undefined otherwise.
Avg. Radius is the average of \(\Rtilde(\x)\) on the same subset.

The certified-radius distribution (CRD) is reported in two complementary forms.
Boundary CRD measures the fraction of all test inputs that are both in \(\mathcal B\) and have radius above threshold \(t\):
\begin{equation}
 \mathrm{CRD}_{\mathrm{bdry}}(t)
 =
 \frac{1}{|\mathcal D_{\mathrm{test}}|}
 \sum_{\x\in\mathcal D_{\mathrm{test}}}
 \bone[\x\in\mathcal B,\ \Rtilde(\x)>t].
\end{equation}
Higher boundary CRD at threshold \(t\) means more borderline inputs are certified at radius at least \(t\); lower values mean fewer, either because the method has low boundary mass or because radii inside \(\mathcal B\) are compressed.
Boundary inputs are the most diagnostic of radius fidelity, since small probability errors in this region can flip whether the certificate is positive.
Certified-input CRD is instead conditioned on \(\ptildeA(\x)>1/2\) and describes the radius distribution among certified inputs.
Computational cost is measured in forward-pass equivalents, with one backward pass counted as two; break-even includes offline target construction, training to the selected checkpoint, and per-query test cost.

\paragraph{Q1: Certified accuracy.}
\label{subsec:q1_certified_accuracy}
Figure~\ref{fig:main_cert_boundary} shows that \ours\ closely preserves the fixed-budget MC certificate while replacing \(10{,}000\) noisy base-model evaluations with one surrogate forward pass at inference.
Its CertAcc@0 gaps relative to Baseline~1 are only \(0.19/0.62/0.84/0.55\) percentage points on FashionMNIST/CIFAR-10/CIFAR-100/Tiny ImageNet:
\ours\ certifies \(87.44/70.26/33.91/27.43\%\), compared with \(87.63/70.88/34.75/27.98\%\) for Baseline~1.
The small undercertification is expected because the conformal offset \(\delta\) shrinks each surrogate top-probability estimate before radius computation, converting a single surrogate prediction into a conservative lower bound on \(p_A(\x)\).

\begin{figure*}[t]
\centering
\includegraphics[width=\textwidth]{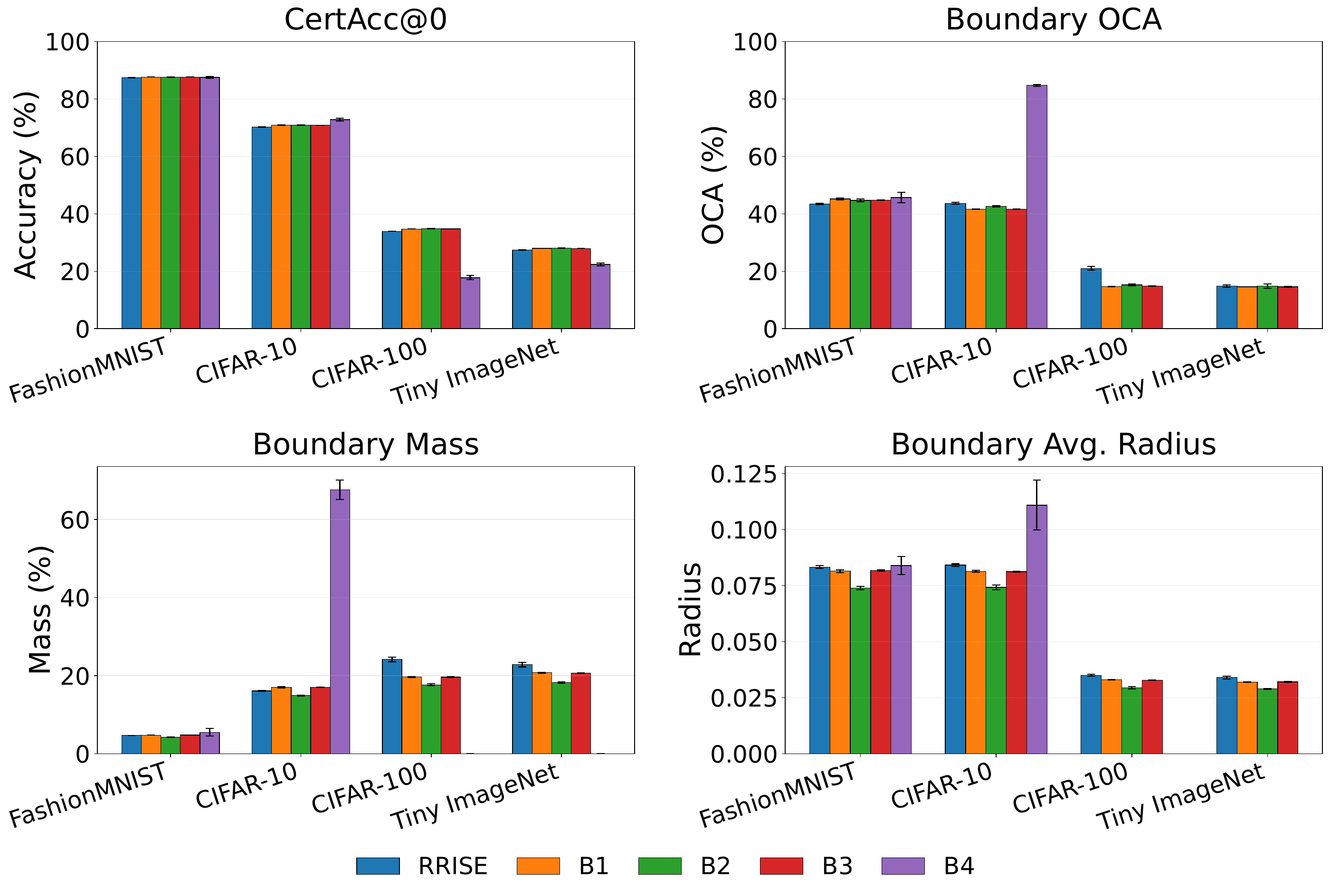}
\caption{\textbf{Certified accuracy and boundary behavior.} The four panels report: top-left, full-test CertAcc@0; top-right, ordinary classification accuracy on the boundary-confidence subset \(0.5<\ptildeA(\x)<0.75\); bottom-left, boundary mass; and bottom-right, average certified radius within the boundary subset. \ours\ closely tracks Baseline~1 in CertAcc@0 and average boundary radius across all four datasets. On CIFAR-100 and Tiny ImageNet, \ours\ places slightly more examples in the boundary region, reflecting conservative calibration: ambiguous inputs remain near the certification threshold rather than being assigned inflated confidence or radii. Computational cost is reported separately in Figure~\ref{fig:main_compute}.}
\label{fig:main_cert_boundary}
\end{figure*}

\paragraph{Q2: Boundary-radius reliability.}
\label{subsec:q2_boundary_reliability}
The boundary metrics---OCA, boundary mass, and average radius on \(\mathcal B\)---test whether calibration creates inflated radii near the certification threshold.
\ours\ matches Baseline~1's average boundary radius to two decimals on every dataset \((0.08/0.08/0.03/0.03)\), showing that the surrogate does not inflate radii in the region where probability errors are most consequential.
On the harder datasets, \ours\ places slightly more inputs in \(\mathcal B\) than Baseline~1: \(24.16\%\) vs.\ \(19.67\%\) on CIFAR-100 and \(22.85\%\) vs.\ \(20.76\%\) on Tiny ImageNet.
This is conservative: ambiguous inputs remain near the certification threshold rather than being pushed into high-confidence regions with inflated radii.
Boundary OCA remains competitive and improves over Baseline~1 by \(6.31\) percentage points on CIFAR-100, indicating that the inputs treated as borderline are often still correctly classified.
These boundary metrics are diagnostics of calibrated probability geometry; formal certificate validity follows from Proposition~\ref{prop:calibration}.

\paragraph{Comparison with the offline Jensen--Shannon surrogate.}
Baseline~4 also uses one surrogate forward pass at inference, so the comparison isolates surrogate quality rather than online cost.
On CIFAR-10, Baseline~4 attains higher CertAcc@0 \((72.79\%\) vs.\ \(70.26\%)\), but places \(67.58\%\) of test inputs in \(\mathcal B\), about four times Baseline~1's \(16.99\%\) and \ours's \(16.10\%\).
This is a warning sign rather than a strength: two-thirds of the test set remains in the diagnostic boundary region instead of being confidently certified, consistent with biased gradients in the Jensen--Shannon objective failing to converge to the smoothed distribution (Appendix~\ref{app:rrise_vs_baseline4}).
On the harder datasets the failure is sharper: CertAcc@0 drops to \(17.76\%\) on CIFAR-100 and \(22.37\%\) on Tiny ImageNet, and the calibrated outputs yield zero boundary mass.
Figure~\ref{fig:main_boundary_crd} visualizes this directly: the boundary CRD curves for Baseline~4 collapse to zero on the harder datasets, while \ours\ tracks the MC baselines closely on FashionMNIST and CIFAR-10 and remains well-defined throughout.

\begin{figure*}[t]
\centering
\includegraphics[width=\textwidth]{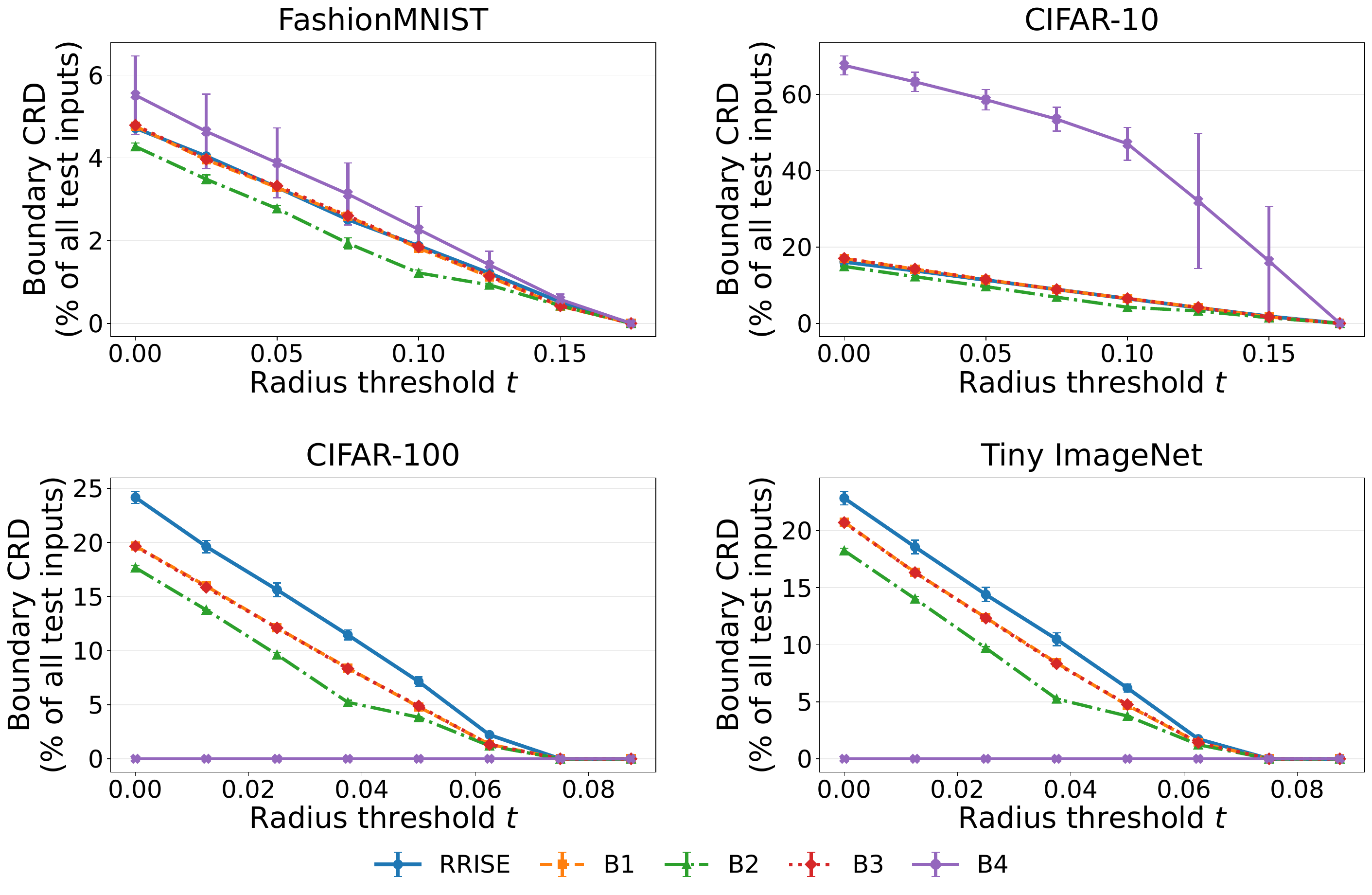}
\caption{\textbf{Boundary-confidence radius distributions.} Each curve plots the fraction of test inputs satisfying \(0.5<\ptildeA(\x)<0.75\) and \(\Rtilde(\x)>t\). Higher curves mean that more borderline inputs are certified at radius at least \(t\). \ours\ tracks the MC baselines on FashionMNIST and CIFAR-10 and remains well-defined on CIFAR-100 and Tiny ImageNet; Baseline~4 collapses on the harder datasets.}
\label{fig:main_boundary_crd}
\end{figure*}

\paragraph{Q3: Computational break-even.}
\label{subsec:q3_compute}
Figure~\ref{fig:main_compute} reports online cost and cumulative savings after accounting for offline MC target construction and surrogate training.
Let \(C_{Bk}(m)\) denote the total cost of Baseline~\(k\) after \(m\) deployment queries, and let \(C_{\mathrm{RRISE}}(m)\) denote the corresponding total cost of \ours.
We plot
\[
S_{\mathrm{RRISE}\leftarrow Bk}(m)
=
C_{Bk}(m)-C_{\mathrm{RRISE}}(m),
\]
where positive values mean that \ours\ has recovered its offline cost and is cheaper overall.

\begin{figure*}[t]
\centering
\includegraphics[width=\linewidth]{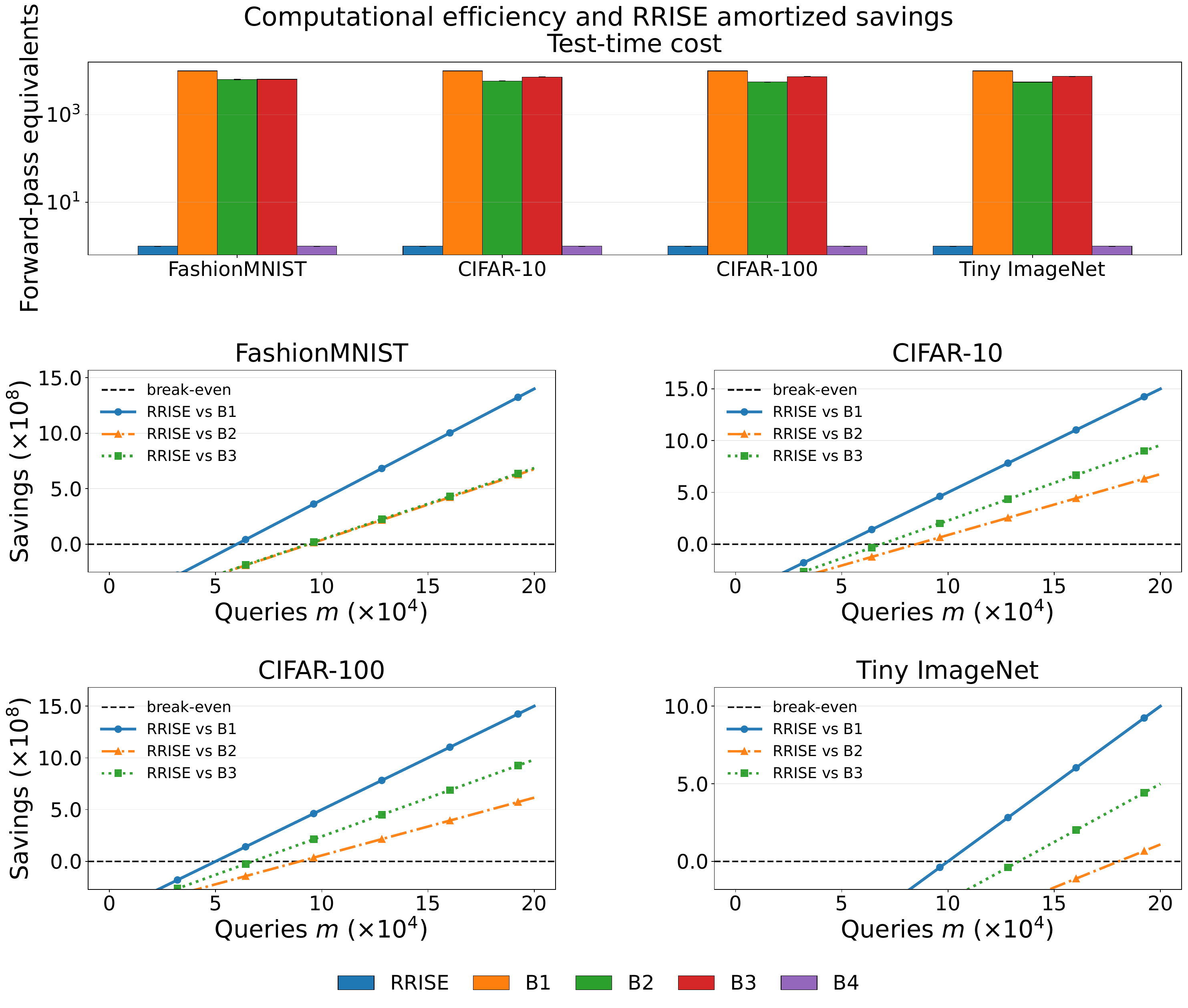}
\caption{\textbf{Computational cost and break-even.} Top: average online forward-pass equivalents per input. Bottom: cumulative savings of \ours\ over Baselines~1--3; the zero line marks break-even. Break-even against full-budget MC occurs at \(5\text{--}10\times10^4\) queries and against partial-sampling accelerators at \(1.8\text{--}2.1\times10^5\) queries.}
\label{fig:main_compute}
\end{figure*}

Against fixed-budget MC, \ours\ breaks even at approximately \(6\times10^4\) queries on FashionMNIST, \(5\times10^4\) on CIFAR-10 and CIFAR-100, and \(10^5\) on Tiny ImageNet.
Break-even is later against Baselines~2 and~3 because they already reduce average online sampling: on CIFAR-10 they use \(5883\) and \(7281\) noisy passes per input on average, and \ours\ breaks even at approximately \(2.1\times10^5\) and \(1.8\times10^5\) queries, respectively.
Thus, \ours\ is most useful in repeated-deployment settings where many inputs are certified under the same base classifier and smoothing distribution.

\paragraph{Main Results Summary.}
\label{subsec:main_results}
The experiments answer the three evaluation questions affirmatively.
First, \ours\ preserves fixed-budget MC certified accuracy: across all datasets, CertAcc@0 remains within \(0.84\) percentage points of Baseline~1 while reducing online certification from \(10{,}000\) noisy base-model evaluations to one surrogate forward pass.
Second, calibration does not inflate boundary radii: \ours\ matches Baseline~1's average boundary radius on every dataset and keeps ambiguous inputs near the certification threshold on the harder datasets, which is the conservative direction.
Third, the offline cost is recovered in repeated deployment: \ours\ breaks even against full-budget MC after roughly \(5\text{--}10\times10^4\) queries, and against adaptive MC baselines after roughly \(1.8\text{--}2.1\times10^5\) queries.
Together, these results show that \ours\ is not merely a faster proxy for randomized smoothing; it is a calibrated, sampling-free inference mechanism that preserves MC certification behavior while becoming substantially cheaper once certification is performed repeatedly.

\paragraph{Ablations.}
Appendix~\ref{app:additional_results} provides full numerical support tables and ablations over MC target budget, \ours\ training strategy, calibration level, Baseline~2 decline level, Baseline~3 stopping tolerance, and Baseline~4 initialization.
The appendix also reports the complete CRD tables over all evaluated thresholds.
\section{Discussion}
\label{sec:discussion}
\ours is designed for repeated certification under a fixed base classifier, smoothing distribution, and deployment domain. In this regime, the offline MC target-construction cost is amortized across many future queries, and each new certificate is obtained with one surrogate forward pass. This differs from adaptive MC methods, which reduce average sampling cost but still evaluate the base classifier under noisy perturbations at test time. The main limitation is the upfront cost of constructing offline targets. Although this cost is amortizable, it is still substantial for large training sets or large MC budgets. Active target selection, curriculum-style target generation, or multi-fidelity targets could reduce this cost. A second limitation is calibration under distribution shift. The conformal guarantee assumes exchangeability between calibration and deployment inputs; in practice, deployment monitoring or recalibration may be necessary when the input distribution changes. Third, our experiments use one target smoothing level per dataset. A noise-conditioned surrogate could share information across multiple $\sigma$ values and support input-adaptive smoothing, but would require additional calibration care. Finally, our empirical study focuses on image classification. 

 
\section{Conclusion}
\label{sec:conclusion}

We presented \ours, an accelerated framework for randomized-smoothing robustness-radius inference. \ours trains a surrogate on offline MC class-count targets and uses conformal calibration to turn the surrogate output into a conservative lower bound on the smoothed top-class probability. The resulting method replaces thousands of online noisy base-model evaluations with one surrogate forward pass while closely preserving fixed-budget MC certified accuracy. The empirical results show that acceleration is a practical path toward scalable smoothing certification when many inputs are certified under the same base model and smoothing distribution.



\bibliographystyle{unsrtnat}
\bibliography{references}

@inproceedings{guo2017calibration,
  title={On calibration of modern neural networks},
  author={Guo, Chuan and Pleiss, Geoff and Sun, Yu and Weinberger, Kilian Q},
  booktitle={International Conference on Machine Learning},
  pages={1321--1330},
  year={2017},
  organization={PMLR}
}

@inproceedings{gal2016dropout,
  title={Dropout as a Bayesian approximation: Representing model uncertainty in deep learning},
  author={Gal, Yarin and Ghahramani, Zoubin},
  booktitle={International Conference on Machine Learning},
  pages={1050--1059},
  year={2016},
  organization={PMLR}
}

@article{lakshminarayanan2017simple,
  title={Simple and scalable predictive uncertainty estimation using deep ensembles},
  author={Lakshminarayanan, Balaji and Pritzel, Alexander and Blundell, Charles},
  journal={Advances in Neural Information Processing Systems},
  volume={30},
  year={2017}
}

@article{maddox2019simple,
  title={A simple baseline for Bayesian uncertainty in deep learning},
  author={Maddox, Wesley J and Izmailov, Pavel and Garipov, Timur and Vetrov, Dmitry P and Wilson, Andrew Gordon},
  journal={Advances in Neural Information Processing Systems},
  volume={32},
  year={2019}
}

@article{liu2023simple,
  title={A simple approach to improve single-model deep uncertainty via distance-awareness},
  author={Liu, Jeremiah Zhe and Padhy, Shreyas and Ren, Jie and Lin, Zi and Wen, Yeming and Jerfel, Ghassen and Nado, Zachary and Snoek, Jasper and Tran, Dustin and Lakshminarayanan, Balaji},
  journal={Journal of Machine Learning Research},
  volume={24},
  number={42},
  pages={1--63},
  year={2023}
}

@article{liu2020energy,
  title={Energy-based out-of-distribution detection},
  author={Liu, Weitang and Wang, Xiaoyun and Owens, John and Li, Yixuan},
  journal={Advances in Neural Information Processing Systems},
  volume={33},
  pages={21464--21475},
  year={2020}
}

@inproceedings{liang2018enhancing,
  title={Enhancing The Reliability of Out-of-distribution Image Detection in Neural Networks},
  author={Liang, Shiyu and Li, Yixuan and Srikant, R},
  booktitle={International Conference on Learning Representations},
  year={2018}
}

@article{hein2017formal,
  title={Formal guarantees on the robustness of a classifier against adversarial manipulation},
  author={Hein, Matthias and Andriushchenko, Maksym},
  journal={Advances in Neural Information Processing Systems},
  volume={30},
  year={2017}
}

@inproceedings{cohen2019certified,
  title={Certified adversarial robustness via randomized smoothing},
  author={Cohen, Jeremy and Rosenfeld, Elan and Kolter, Zico},
  booktitle={International Conference on Machine Learning},
  pages={1310--1320},
  year={2019},
  organization={PMLR}
}

@inproceedings{lecuyer2019certified,
  title={Certified robustness to adversarial examples with differential privacy},
  author={Lecuyer, Mathias and Atlidakis, Vaggelis and Geambasu, Roxana and Hsu, Daniel and Jana, Suman},
  booktitle={2019 IEEE Symposium on Security and Privacy (SP)},
  pages={656--672},
  year={2019},
  organization={IEEE}
}

@article{seferis2024estimating,
  title={Estimating the Robustness Radius for Randomized Smoothing with 100x Sample Efficiency},
  author={Seferis, Emmanouil and Kollias, Stefanos and Cheng, Chih-Hong},
  journal={arXiv preprint arXiv:2404.17371},
  year={2024}
}

@article{voracek2024treatment,
  title={Treatment of statistical estimation problems in randomized smoothing for adversarial robustness},
  author={Voracek, Vaclav},
  journal={Advances in Neural Information Processing Systems},
  volume={37},
  pages={133464--133486},
  year={2024}
}

@inproceedings{ugare2024incremental,
  title={Incremental Randomized Smoothing Certification},
  author={Ugare, Shubham and Suresh, Tarun and Banerjee, Debangshu and Singh, Gagandeep and Misailovic, Sasa},
  booktitle={The Twelfth International Conference on Learning Representations},
  year={2024}
}

@article{salman2019provably,
  title={Provably robust deep learning via adversarially trained smoothed classifiers},
  author={Salman, Hadi and Li, Jerry and Razenshteyn, Ilya and Zhang, Pengchuan and Zhang, Huan and Bubeck, Sebastien and Yang, Greg},
  journal={Advances in Neural Information Processing Systems},
  volume={32},
  year={2019}
}

@article{fischer2020certified,
  title={Certified defense to image transformations via randomized smoothing},
  author={Fischer, Marc and Baader, Maximilian and Vechev, Martin},
  journal={Advances in Neural Information Processing Systems},
  volume={33},
  pages={8404--8417},
  year={2020}
}

@inproceedings{yang2020randomized,
  title={Randomized smoothing of all shapes and sizes},
  author={Yang, Greg and Duan, Tony and Hu, J Edward and Salman, Hadi and Razenshteyn, Ilya and Li, Jerry},
  booktitle={International Conference on Machine Learning},
  pages={10693--10705},
  year={2020},
  organization={PMLR}
}

@article{bhardwaj2024accelerated,
  title={Accelerated Smoothing: A Scalable Approach to Randomized Smoothing},
  author={Bhardwaj, Devansh and Kaushik, Kshitiz and Gupta, Sarthak},
  journal={arXiv preprint arXiv:2402.07498},
  year={2024}
}

@inproceedings{eykholt2018robust,
  title={Robust physical-world attacks on deep learning visual classification},
  author={Eykholt, Kevin and Evtimov, Ivan and Fernandes, Earlence and Li, Bo and Rahmati, Amir and Xiao, Chaowei and Prakash, Atul and Kohno, Tadayoshi and Song, Dawn},
  booktitle={Proceedings of the IEEE Conference on Computer Cision and Pattern Recognition},
  pages={1625--1634},
  year={2018}
}

@article{ma2021understanding,
  title={Understanding adversarial attacks on deep learning based medical image analysis systems},
  author={Ma, Xingjun and Niu, Yuhao and Gu, Lin and Wang, Yisen and Zhao, Yitian and Bailey, James and Lu, Feng},
  journal={Pattern Recognition},
  volume={110},
  pages={107332},
  year={2021},
  publisher={Elsevier}
}

@inproceedings{xu2024gtp,
  title={{GTP-ViT}: Efficient vision transformers via graph-based token propagation},
  author={Xu, Xuwei and Wang, Sen and Chen, Yudong and Zheng, Yanping and Wei, Zhewei and Liu, Jiajun},
  booktitle={Proceedings of the IEEE/CVF Winter Conference on Applications of Computer Vision},
  pages={86--95},
  year={2024}
}

@article{li2019certified,
  title={Certified adversarial robustness with additive noise},
  author={Li, Bai and Chen, Changyou and Wang, Wenlin and Carin, Lawrence},
  journal={Advances in Neural Information Processing Systems},
  volume={32},
  year={2019}
}

@article{pfrommer2023projected,
  title={Projected randomized smoothing for certified adversarial robustness},
  author={Pfrommer, Samuel and Anderson, Brendon G and Sojoudi, Somayeh},
  journal={Transactions on Machine Learning Research},
  year={2023}
}

@inproceedings{alfarra2022data,
  title={Data dependent randomized smoothing},
  author={Alfarra, Motasem and Bibi, Adel and Torr, Philip HS and Ghanem, Bernard},
  booktitle={Uncertainty in Artificial Intelligence},
  pages={64--74},
  year={2022},
  organization={PMLR}
}

@article{singh2019abstract,
  title={An abstract domain for certifying neural networks},
  author={Singh, Gagandeep and Gehr, Timon and P{\"u}schel, Markus and Vechev, Martin},
  journal={Proceedings of the ACM on Programming Languages},
  volume={3},
  number={POPL},
  pages={1--30},
  year={2019},
  publisher={ACM New York, NY, USA}
}

@inproceedings{weng2018towards,
  title={Towards fast computation of certified robustness for relu networks},
  author={Weng, Lily and Zhang, Huan and Chen, Hongge and Song, Zhao and Hsieh, Cho-Jui and Daniel, Luca and Boning, Duane and Dhillon, Inderjit},
  booktitle={International Conference on Machine Learning},
  pages={5276--5285},
  year={2018},
  organization={PMLR}
}

@article{wang2018efficient,
  title={Efficient formal safety analysis of neural networks},
  author={Wang, Shiqi and Pei, Kexin and Whitehouse, Justin and Yang, Junfeng and Jana, Suman},
  journal={Advances in Neural Information Processing Systems},
  volume={31},
  year={2018}
}

@inproceedings{cao2023robust,
  title={Robust trajectory prediction against adversarial attacks},
  author={Cao, Yulong and Xu, Danfei and Weng, Xinshuo and Mao, Zhuoqing and Anandkumar, Anima and Xiao, Chaowei and Pavone, Marco},
  booktitle={Conference on Robot Learning},
  pages={128--137},
  year={2023},
  organization={PMLR}
}

@inproceedings{xie2020real,
  title={Real-time, universal, and robust adversarial attacks against speaker recognition systems},
  author={Xie, Yi and Shi, Cong and Li, Zhuohang and Liu, Jian and Chen, Yingying and Yuan, Bo},
  booktitle={ICASSP 2020-2020 IEEE International Conference on Acoustics, Speech and Signal Processing (ICASSP)},
  pages={1738--1742},
  year={2020},
  organization={IEEE}
}

@article{chi2024adversarial,
  title={Adversarial attacks on autonomous driving systems in the physical world: a survey},
  author={Chi, Lijun and Msahli, Mounira and Zhang, Qingjie and Qiu, Han and Zhang, Tianwei and Memmi, Gerard and Qiu, Meikang},
  journal={IEEE Transactions on Intelligent Vehicles},
  year={2024},
  publisher={IEEE}
}

@article{fawzi2018analysis,
  title={Analysis of classifiers’ robustness to adversarial perturbations},
  author={Fawzi, Alhussein and Fawzi, Omar and Frossard, Pascal},
  journal={Machine learning},
  volume={107},
  number={3},
  pages={481--508},
  year={2018},
  publisher={Springer}
}

@article{shafer2008tutorial,
  title={A tutorial on conformal prediction.},
  author={Shafer, Glenn and Vovk, Vladimir},
  journal={Journal of Machine Learning Research},
  volume={9},
  number={3},
  year={2008}
}

@article{finlayson2019adversarial,
  title={Adversarial attacks on medical machine learning},
  author={Finlayson, Samuel G and Bowers, John D and Ito, Joichi and Zittrain, Jonathan L and Beam, Andrew L and Kohane, Isaac S},
  journal={Science},
  volume={363},
  number={6433},
  pages={1287--1289},
  year={2019},
  publisher={American Association for the Advancement of Science}
}

@article{geifman2017selective,
  title={Selective classification for deep neural networks},
  author={Geifman, Yonatan and El-Yaniv, Ran},
  journal={Advances in Neural Information Processing Systems},
  volume={30},
  year={2017}
}

@article{liu2025comprehensive,
  title={A comprehensive study on robustness of image classification models: Benchmarking and rethinking},
  author={Liu, Chang and Dong, Yinpeng and Xiang, Wenzhao and Yang, Xiao and Su, Hang and Zhu, Jun and Chen, Yuefeng and He, Yuan and Xue, Hui and Zheng, Shibao},
  journal={International Journal of Computer Vision},
  volume={133},
  number={2},
  pages={567--589},
  year={2025},
  publisher={Springer}
}

@article{kumari2023rethinking,
  title={Rethinking Randomized Smoothing from the Perspective of Scalability},
  author={Kumari, Anupriya and Bhardwaj, Devansh and Jindal, Sukrit},
  journal={arXiv preprint arXiv:2312.12608},
  year={2023}
}

@inproceedings{carlini2017adversarial,
  title={Adversarial examples are not easily detected: Bypassing ten detection methods},
  author={Carlini, Nicholas and Wagner, David},
  booktitle={Proceedings of the 10th ACM workshop on artificial intelligence and security},
  pages={3--14},
  year={2017}
}

@article{akhtar2018threat,
  title={Threat of adversarial attacks on deep learning in computer vision: A survey},
  author={Akhtar, Naveed and Mian, Ajmal},
  journal={IEEE Access},
  volume={6},
  pages={14410--14430},
  year={2018},
  publisher={IEEE}
}

@article{lei2018distribution,
  title={Distribution-free predictive inference for regression},
  author={Lei, Jing and G’Sell, Max and Rinaldo, Alessandro and Tibshirani, Ryan J and Wasserman, Larry},
  journal={Journal of the American Statistical Association},
  volume={113},
  number={523},
  pages={1094--1111},
  year={2018},
  publisher={Taylor \& Francis}
}

@article{tramer2020adaptive,
  title={On adaptive attacks to adversarial example defenses},
  author={Tramer, Florian and Carlini, Nicholas and Brendel, Wieland and Madry, Aleksander},
  journal={Advances in Neural Information Processing Systems},
  volume={33},
  pages={1633--1645},
  year={2020}
}

\newpage 
\appendix
\startcontents[appendix]
\printcontents[appendix]{}{1}{\section*{Appendix}}
\clearpage
\newpage
\section{Experimental Protocol and Implementation Details}
\label{app:implementation_details}

This appendix provides the implementation details supporting Section~\ref{sec:experiments}. We describe dataset preprocessing, network architectures, base-classifier training, offline MC target construction, RRISE surrogate training, baseline implementations, model selection, calibration and confidence matching, metric reporting details, compute accounting, and ablation settings.

\subsection{Datasets and Preprocessing}
\label{app:datasets_preprocessing}

We evaluate on FashionMNIST, CIFAR-10, CIFAR-100, and Tiny ImageNet using their standard train/test splits. All images are represented in pixel space \([0,1]\). Gaussian noise is always added in pixel space, clipped to \([0,1]\), and then normalized using dataset-specific statistics. This convention is used consistently for base-model noise augmentation, offline MC target construction, calibration sampling, and final certification.

FashionMNIST images are resized from \(28\times28\) to \(32\times32\) and replicated to three channels. CIFAR-10 and CIFAR-100 are evaluated at \(32\times32\), and Tiny ImageNet is evaluated at \(64\times64\). Stochastic data augmentation is used only for base-classifier training. Offline MC target construction, RRISE validation, calibration, and final evaluation use deterministic preprocessing so that each dataset index corresponds to a fixed clean input.

The training-time augmentations are:
\[
\begin{array}{ll}
\text{FashionMNIST:} & \text{resize to }32\times32,\ \text{random crop with padding }2,\\
\text{CIFAR-10:} & \text{random crop with padding }4,\ \text{random horizontal flip},\\
\text{CIFAR-100:} & \text{random crop with padding }4,\ \text{random horizontal flip,\ RandAugment},\\
\text{Tiny ImageNet:} & \text{random resized crop,\ random horizontal flip,\ RandAugment,\ random erasing}.
\end{array}
\]
For evaluation, we use deterministic resizing/cropping; Tiny ImageNet uses resize followed by center crop.

\subsection{Network Architectures}
\label{app:architectures}

Table~\ref{tab:architecture_details} summarizes the architecture used for each dataset. The same architecture is used for the base classifier \(f\), the RRISE surrogate \(q_{\btheta}\), and the offline surrogate baseline.

\begin{table}[ht]
\centering
\caption{Dataset--architecture pairs used in all experiments.}
\label{tab:architecture_details}
\begin{tabular}{lccc}
\toprule
Dataset & Architecture & Input size & Classes \\
\midrule
FashionMNIST & MLP-Mixer-Tiny & \(32\times32\) & 10 \\
CIFAR-10 & ResNet-18 with CIFAR stem & \(32\times32\) & 10 \\
CIFAR-100 & EfficientNet-B0 & \(32\times32\) & 100 \\
Tiny ImageNet & ViT-Tiny & \(64\times64\) & 200 \\
\bottomrule
\end{tabular}
\end{table}

For CIFAR-10, we use ResNet-18 with a CIFAR-style stem: the first convolution is replaced by a \(3\times3\) convolution with stride 1 and padding 1, and the initial max-pooling layer is removed. For Tiny ImageNet, the ViT-Tiny model uses patch size 8, embedding dimension 192, depth 9, 12 attention heads, MLP ratio 2.0, and LayerNorm with \(\epsilon=10^{-6}\).

\subsection{Base Classifier Training}
\label{app:base_training}

Base classifiers are trained with supervised cross-entropy and Gaussian noise augmentation. Noise is injected in pixel space before normalization, matching the smoothing distribution used during certification. Table~\ref{tab:base_training_details} summarizes the dataset-specific settings.

\begin{table}[ht]
\centering
\caption{Base-classifier training settings.}
\label{tab:base_training_details}
\begin{tabular}{lcccc}
\toprule
Dataset & Architecture & \(\sigma_{\mathrm{base}}\) & Epochs & Learning rate \\
\midrule
FashionMNIST & MLP-Mixer-Tiny & 0.5 & 200 & \(10^{-3}\) \\
CIFAR-10 & ResNet-18 & 0.5 & 200 & \(10^{-3}\) \\
CIFAR-100 & EfficientNet-B0 & 0.25 & 500 & \(10^{-3}\) \\
Tiny ImageNet & ViT-Tiny & 0.25 & 500 & \(2\times10^{-3}\) \\
\bottomrule
\end{tabular}
\end{table}

All base classifiers use AdamW with weight decay \(5\times10^{-4}\), batch size 512, label smoothing 0.1, gradient clipping at norm 1.0, and cosine learning-rate decay with a 5-epoch warmup. Automatic mixed precision is used on CUDA unless disabled. The base-model training seed is fixed to 0, and the best checkpoint is selected by validation/evaluation accuracy.

\subsection{Offline Monte Carlo Target Construction}
\label{app:mc_target_construction}

RRISE and Baseline~4 use offline MC target distributions. For each clean training example \(\x_i\), we draw \(n\) Gaussian perturbations using the target smoothing noise \(\sigma\), evaluate the frozen base classifier, and store the normalized class-count vector \(C_i^{(n)}\):
\[
C_i^{(n)}[c]
=
\frac{1}{n}
\sum_{j=1}^{n}
\mathbf{1}\!\left[
f(\x_i+\bepsilon_{i,j})=c
\right],
\qquad
\bepsilon_{i,j}\sim\Normal(\bzero,\sigma^2\bI).
\]
We evaluate MC target budgets
\[
n\in\{500,1000,5000,10000\}.
\]
The target smoothing noise is
\[
\sigma =
\begin{cases}
0.25, & \text{FashionMNIST and CIFAR-10},\\
0.10, & \text{CIFAR-100 and Tiny ImageNet}.
\end{cases}
\]
MC targets are constructed using deterministic preprocessing, so each training index has a stable target vector. Target datasets are constructed separately for seeds \(\{100,200,300\}\), matching the seed-averaged evaluation protocol.

\subsection{RRISE Surrogate Training}
\label{app:rrise_training}

The RRISE surrogate \(q_{\btheta}\) is trained on clean inputs to match the offline MC target vector \(C_i^{(n)}\). We use soft-label cross-entropy:
\[
\mathcal{L}_{\mathrm{RRISE}}(\btheta)
=
-\frac{1}{|\mathcal{B}|}
\sum_{i\in\mathcal{B}}
\sum_{c=1}^{K}
C_i^{(n)}[c]\log q_{\btheta}(c\mid \x_i).
\]
The surrogate uses the same architecture as the corresponding base classifier. We evaluate three training strategies:
\[
\begin{array}{ll}
\textbf{Full + base initialization:} & \text{initialize from }f\text{ and train all parameters},\\
\textbf{Full + random initialization:} & \text{randomly initialize and train all parameters},\\
\textbf{Head + base initialization:} & \text{initialize from }f\text{ and train only the classifier head}.
\end{array}
\]
The main results use head-only training initialized from the base classifier.

Table~\ref{tab:rrise_training_details} summarizes the surrogate optimization hyperparameters.

\begin{table}[ht]
\centering
\caption{RRISE surrogate training settings.}
\label{tab:rrise_training_details}
\begin{tabular}{lccc}
\toprule
Dataset & Epochs & Learning rate & Target noise \(\sigma\) \\
\midrule
FashionMNIST & 200 & \(5\times10^{-4}\) & 0.25 \\
CIFAR-10 & 200 & \(5\times10^{-4}\) & 0.25 \\
CIFAR-100 & 500 & \(5\times10^{-4}\) & 0.10 \\
Tiny ImageNet & 500 & \(10^{-3}\) & 0.10 \\
\bottomrule
\end{tabular}
\end{table}

All RRISE surrogates use AdamW with weight decay \(5\times10^{-4}\), batch size 512, gradient clipping at norm 1.0, cosine learning-rate decay with 5 warmup epochs, and automatic mixed precision on CUDA unless disabled. Checkpoint selection is described in Appendix~\ref{app:model_selection}.

\subsection{Baseline Details}
\label{app:baseline_details}

All baselines use the same trained base classifier \(f\), the same test split, the same deterministic evaluation preprocessing, and the same pixel-space Gaussian noise convention axs RRISE. In particular, noisy inputs are generated by adding Gaussian noise in pixel space, clipping to \([0,1]\), and then normalizing before model evaluation. Unless otherwise stated, all baseline results are averaged over seeds \(\{100,200,300\}\). The target smoothing noise is \(\sigma=0.25\) for FashionMNIST/CIFAR-10 and \(\sigma=0.10\) for CIFAR-100/Tiny ImageNet.

\subsubsection{Baseline~1: Fixed-Budget MC Randomized Smoothing}
\label{app:baseline1_details}

Baseline~1 implements standard randomized smoothing~\citep{cohen2019certified}. For each test input, it draws \(n\) noisy samples from the fixed target smoothing distribution, evaluates the frozen base classifier on each noisy input, and selects the empirical top class as the smoothed prediction. It then computes a one-sided Clopper--Pearson lower confidence bound for the probability of this empirical top class and certifies if the lower bound exceeds \(1/2\). We evaluate
\[
n\in\{500,1000,5000,10000\}.
\]
The clean base-model prediction is saved for analysis, but certification and certified accuracy are computed using the smoothed prediction. Although the stored output files were generated with \(\alpha=0.001\), they include the empirical top-class counts and budgets. Therefore, in the final evaluation we recompute the Clopper--Pearson lower bound at the desired failure level \(\alpha_{\mathrm{MC}}\).

\subsubsection{Baseline~2: Input-Specific Sample Budgeting}
\label{app:baseline2_details}

Baseline~2 follows the sample-efficient MC-budgeting idea of~\citet{seferis2024estimating}. It uses a pilot stage and a precomputed budget-mapping rule to reduce the number of noisy evaluations. For a maximum budget \(k_{\max}\), the pilot size is
\[
n_0=\min\{k_{\max},\max(100,\lfloor0.01k_{\max}\rfloor)\}.
\]
The pilot samples are used to estimate the empirical top class and a confidence interval for its probability. A piecewise budget mapping then determines the predicted estimation budget \(n_{\mathrm{hat}}\). To avoid degenerate probability estimates from very small sample sizes, we enforce a minimum pilot size of 100 and a minimum final estimation size of 100 whenever a positive final budget is used. The realized sample count is capped by \(k_{\max}\) and lower-bounded by this minimum estimation size. If the budget mapping returns zero, we use the pilot evidence conservatively rather than fabricating a zero probability estimate.

We evaluate
\[
k_{\max}\in\{500,1000,5000,10000\},
\qquad
\text{decline}\in\{0.01,0.05\}.
\]
The decline parameter is used in the relative-decline setting. The mapping grid uses resolution \(0.001\). The final certificate is computed from the realized count evidence using a one-sided Clopper--Pearson lower bound. We report both certified performance and the average realized sample count \(n_{\mathrm{used}}\).

\subsubsection{Baseline~3: Budget Prediction / Early Stopping}
\label{app:baseline3_details}

Baseline~3 follows the statistical-estimation and early-stopping perspective of~\citet{voracek2024treatment}. It samples sequentially up to \(n_{\max}\), using a CLT-style approximation to predict whether additional samples are likely to improve the radius. This approximation is used only for the stopping decision. The reported certificate is still a standard Clopper--Pearson certificate computed from the realized number of samples \(n_{\mathrm{used}}\). We evaluate
\[
n_{\max}\in\{500,1000,5000,10000\},
\qquad
\text{tolerance}\in\{0.01,0.05\}.
\]
The sequential step size is
\[
n_{\mathrm{step}}=\left\lceil\frac{n_{\max}}{10}\right\rceil.
\]
Sampling stops when the predicted relative radius gap is below the chosen tolerance or when the maximum budget is reached. We report the final certified radius, stopping behavior, and average/median realized sample count.

\subsubsection{Baseline~4: Jensen--Shannon Offline Surrogate}
\label{app:baseline4_details}

Baseline~4 follows the offline-surrogate setting of~\citet{bhardwaj2024accelerated}. It uses the same offline MC target vectors \(C_i^{(n)}\) as RRISE, but trains the surrogate using Jensen--Shannon divergence between the predicted class distribution and the MC target vector. This baseline isolates the effect of the surrogate-learning objective while controlling for the use of offline MC targets.

Baseline~4 uses the same architecture as the corresponding base classifier and the same deterministic clean input associated with each MC target. The MC target dataset is split into 90\% training and 10\% validation. The surrogate is trained with AdamW, weight decay \(5\times10^{-4}\), batch size 512, cosine learning-rate decay with 5 warmup epochs, gradient clipping at norm 1.0, and automatic mixed precision on CUDA unless disabled. The learning rate is \(5\times10^{-4}\) for FashionMNIST/CIFAR-10/CIFAR-100 and \(10^{-3}\) for Tiny ImageNet. The number of epochs follows the RRISE training schedule: 200 epochs for FashionMNIST/CIFAR-10 and 500 epochs for CIFAR-100/Tiny ImageNet. We evaluate both random initialization and base-model initialization. The main comparison uses random initialization and \(n=10000\), matching the original offline-surrogate setting. Checkpoint selection is described in Appendix~\ref{app:model_selection}.

\subsection{Model Selection}
\label{app:model_selection}

We use held-out validation criteria to select all trained checkpoints.

\subsubsection{Base Classifiers}

For each dataset, the base classifier \(f\) is trained with Gaussian noise augmentation using the base smoothing level \(\sigma_{\mathrm{base}}\). We save the checkpoint with the highest evaluation accuracy on the validation/evaluation split and use this checkpoint as the frozen teacher for all downstream randomized-smoothing experiments, offline MC target construction, RRISE training, and baseline evaluations.

\subsubsection{RRISE Surrogates}

For each MC target budget \(n\) and seed, the offline MC target dataset is split into a training subset and a held-out validation subset. We use a 10\% validation split stratified by the MC target top probability \(p_A=\max_c C_i^{(n)}[c]\). During training, we evaluate the surrogate on this held-out MC-target validation split and select the checkpoint with the lowest validation all-\(p_A\) mean absolute error. This metric measures how closely the surrogate probability assigned to the MC target top class matches the MC target top probability. This selection criterion is independent of the final test set used for reporting calibrated certified performance.

\subsubsection{Baseline~4 Surrogates}

Baseline~4 is also trained on offline MC target vectors. For a fair comparison, it uses the same MC target train/validation split convention as RRISE. The checkpoint is selected on the held-out MC-target validation split using the validation objective specified for the Jensen--Shannon surrogate baseline. In our implementation, the main Baseline~4 checkpoint is selected by validation certified performance under the surrogate-based rule, while additional validation statistics such as Jensen--Shannon divergence, soft-label cross-entropy, and \(p_A\) error against the MC target are logged for analysis.

\subsubsection{Final Reporting}

After model selection, the selected checkpoint is fixed. Calibration for RRISE and Baseline~4 is then performed using 10\% of the test set as described in Section~\ref{sec:calibration}, and final metrics are computed on the full test set using the selected checkpoint and calibrated probabilities. The test labels are not used to select RRISE or Baseline~4 checkpoints.

\subsection{Calibration and Confidence Levels}
\label{app:calibration_eval_details}

For MC baselines, we evaluate the Clopper--Pearson certificate at failure level \(\alpha_{\mathrm{MC}}\). For surrogate methods, we use the calibration procedure from Section~\ref{sec:calibration}. In all surrogate calibrations, 10\% of the test set is used to compute the calibration offset \(\delta\), and the same \(\delta\) is then applied to the full test set, including the calibration subset.

We fix \(\beta_{\mathrm{sur}}=0.001\) and evaluate total surrogate failure levels
\[
\beta_{\mathrm{sur}}+\gamma_{\mathrm{sur}}
\in
\{0.25,0.10,0.05,0.01\}.
\]
Thus, the corresponding conformal failure levels are
\[
\gamma_{\mathrm{sur}}
\in
\{0.249,0.099,0.049,0.009\}.
\]
The main comparison uses \(\alpha_{\mathrm{MC}}=0.25\) for Baselines~1--3 and \(\beta_{\mathrm{sur}}+\gamma_{\mathrm{sur}}=0.25\) for RRISE and Baseline~4. This matches the failure level used by MC and surrogate methods.

\subsection{Metric Reporting Details}
\label{app:metric_reporting_details}

We use the evaluation metrics defined in Section~\ref{subsec:exp_metrics}: OCA, certified accuracy at fixed radius thresholds, CRD, boundary-confidence CRD, sample reduction, and amortized break-even cost. This subsection specifies the reporting conventions used to instantiate these metrics in the tables.

All metrics are averaged over seeds \(\{100,200,300\}\) and reported as mean \(\pm\) standard deviation. For MC-based methods, \(\widetilde p_A(\x)\) is the one-sided Clopper--Pearson lower bound recomputed from the stored count evidence at the evaluation failure level \(\alpha_{\mathrm{MC}}\). For RRISE and Baseline~4, \(\widetilde p_A(\x)\) is the calibrated surrogate probability obtained using the procedure in Section~\ref{sec:calibration}.

For certified accuracy and CRD, we use the following radius thresholds:
\[
\mathcal{T}_{\mathrm{Tiny/CIFAR100}}
=
\{0,0.0125,0.025,0.0375,0.05,0.0625,0.075,0.0875\},
\]
and
\[
\mathcal{T}_{\mathrm{CIFAR10/FMNIST}}
=
\{0,0.025,0.05,0.075,0.10,0.125,0.15,0.175\}.
\]
The boundary-confidence CRD is computed over the subset satisfying
\[
0.5<\widetilde p_A(\x)<0.75,
\]
and is reported as a fraction of the full test set unless explicitly stated otherwise.

For sample-efficiency comparisons, Baseline~1 uses the fixed budget \(n\), while Baselines~2--3 use the realized sample count \(n_{\mathrm{used}}(x)\). RRISE and Baseline~4 use one surrogate forward pass per test input after training. Radius ratios and radius drops are always computed relative to Baseline~1 at the same dataset, budget, seed, and evaluation failure level.

\subsection{Compute Accounting}
\label{app:compute_accounting}

This subsection gives the implementation-level details behind the computational-efficiency metrics reported in Section~\ref{subsec:exp_metrics}. We report computational cost in forward-pass equivalents. Baseline~1 requires \(n\) noisy base-model forward passes per test input. Baselines~2--3 require the realized number of noisy forward passes \(n_{\mathrm{used}}(x)\). RRISE and Baseline~4 require one surrogate forward pass per test input after training.

For amortized methods, we include one-time offline MC target construction and surrogate training cost. Constructing the MC target dataset costs approximately
\[
|\mathcal{D}_{\mathrm{train}}|\,n
\]
base-model forward passes. For surrogate training, we count one backward pass as approximately two forward-pass equivalents, so one training update costs approximately three forward-pass equivalents. Therefore, the total amortized cost for \(m\) test queries is
\[
C(m)
=
C_{\mathrm{train}}
+
m\,C_{\mathrm{test}}.
\]
The break-even point against a baseline with average per-input cost \(\bar F\) is
\[
m^\star
=
\left\lceil
\frac{C_{\mathrm{train}}}
{\bar F-C_{\mathrm{test}}}
\right\rceil,
\]
when \(\bar F>C_{\mathrm{test}}\). This indicates how many test inputs are needed before the amortized surrogate becomes computationally cheaper than MC-based certification.

\subsection{Ablation Studies Reported in the Appendix}
\label{app:reported_ablations}

The appendix reports the following ablations:
\begin{itemize}
    \item MC target budgets \(n\in\{500,1000,5000,10000\}\).
    \item Baseline~2 decline levels \(\{0.01,0.05\}\).
    \item Baseline~3 stopping tolerances \(\{0.01,0.05\}\).
    \item RRISE training strategy: full-model training with base initialization, full-model training from random initialization, and head-only training with base initialization.
    \item Baseline~4 initialization: random initialization and base-model initialization.
    \item Surrogate calibration levels \(\beta_{\mathrm{sur}}+\gamma_{\mathrm{sur}}\in\{0.25,0.10,0.05,0.01\}\).
\end{itemize}
\clearpage
\newpage
\section{Additional Experimental Results}
\label{app:additional_results}

This appendix provides the full numerical results supporting the main experimental section.  The goal is not only to list the numbers, but also to make each ablation directly interpretable.  Appendix~\ref{app:main_support_tables} reports the exact values used to generate the main-body figures.  Appendix~\ref{app:ablation_list} reports controlled ablations over MC target budget, RRISE training strategy, calibration level, and baseline-specific hyperparameters.  Appendix~\ref{app:full_crd_results} reports the complete certified-radius distribution (CRD) results.

\paragraph{How to read the tables.}
All entries are reported as mean $\pm$ standard deviation over runs.  CertAcc@0 is computed over the full test set.  OCA, average radius, and boundary CRD are computed on the boundary-confidence subset $0.5<\widetilde p_A(\x)<0.75$ unless a table explicitly says otherwise.  If a method has zero mass in the conditioning set, the corresponding conditional metrics are undefined and reported as a dash ``--''.  A dash ``--'' also indicates that a quantity is not applicable to that method.

\begin{table*}[ht]
\centering
\small
\caption{Column glossary for Appendix~\ref{app:additional_results}.}
\label{tab:app_column_glossary}
\begin{tabular}{p{0.22\textwidth}p{0.73\textwidth}}
\toprule
Column & Description \\
\midrule
Dataset & Evaluation dataset. \\
Method & Evaluated certification method: RRISE is the proposed amortized surrogate; Baseline~1 is the fixed-budget MC estimator; Baseline~2 uses an input-specific decline-level budgeting rule; Baseline~3 uses an early-stopping rule; Baseline~4 is the Jensen--Shannon offline surrogate. \\
Variant / Training Strategy & Encodes implementation-specific settings.  For RRISE, \texttt{trainhead} trains only the estimator head and \texttt{trainall} trains all parameters; \texttt{initbase} initializes from the base classifier and \texttt{initrandom} uses random initialization.  For Baseline~4, \texttt{init0} denotes random initialization and \texttt{init1} denotes base-model initialization. \\
Budget & MC sample budget.  For MC baselines, this is the online test-time sampling budget.  For surrogate methods, this is the offline MC target budget used to train the surrogate. \\
CertAcc@0 (\%) & Certified accuracy at radius zero, measured over all test inputs. \\
OCA (\%) & Ordinary classification accuracy on the boundary-confidence subset $0.5<\widetilde p_A(\x)<0.75$. \\
Avg. Radius & Average certified radius on the same boundary-confidence subset. \\
Boundary Mass & Percentage of all test samples whose estimated top-class probability lies in $0.5<\widetilde p_A(\x)<0.75$. \\
$\delta$ & Learned or selected surrogate calibration offset.  Larger values indicate stricter post-hoc correction and can shrink radii or remove boundary mass. \\
Best Epoch & Epoch selected by validation or calibration criteria. \\
Test Cost & Average number of forward-pass equivalents per test input. \\
Train Cost & Offline cost in forward-pass equivalents, including target construction and optimization up to the selected checkpoint. \\
RRISE BE vs B1/B2/B3 & Break-even number of test queries after which RRISE's offline training cost is amortized relative to Baseline~1/2/3. \\
Sample Red. & Sampling reduction factor relative to the full fixed budget.  Larger is cheaper at test time. \\
Total Failure & Surrogate total failure level $\beta_{\mathrm{sur}}+\gamma_{\mathrm{sur}}$ used for calibration. \\
Threshold $t$ & Radius threshold used in CRD tables.  Boundary CRD entries measure $\{\x:0.5<\widetilde p_A(\x)<0.75,\ \widetilde R(\x)>t\}$ as a percentage of all test inputs.  Certified-input CRD entries are conditional fractions over certified inputs, so 1.00 means all certified inputs exceed the threshold. \\
\bottomrule
\end{tabular}
\end{table*}

\begin{table*}[ht]
\centering
\small
\caption{Variant-name glossary.}
\label{tab:app_variant_glossary}
\begin{tabular}{p{0.25\textwidth}p{0.70\textwidth}}
\toprule
Notation & Meaning \\
\midrule
\texttt{n10000}, \texttt{n5000}, etc. & Fixed MC budget $n$ used either at test time or for offline target construction, depending on the method. \\
\texttt{k*\_decline0.01} & Baseline~2 with maximum/sample cap $k$ and decline level $0.01$. \\
\texttt{k*\_decline0.05} & Baseline~2 with maximum/sample cap $k$ and decline level $0.05$. \\
\texttt{n*\_tol0.01} & Baseline~3 with MC cap $n$ and stopping tolerance $0.01$. \\
\texttt{n*\_tol0.05} & Baseline~3 with MC cap $n$ and stopping tolerance $0.05$. \\
\texttt{trainhead\_initbase\_n*} & RRISE trained only in the estimator head, initialized from the base classifier, using offline MC targets with budget $n$. \\
\texttt{trainall\_initbase\_n*} & RRISE trained end-to-end from the base-classifier initialization. \\
\texttt{trainall\_initrandom\_n*} & RRISE trained end-to-end from random initialization. \\
\texttt{init0\_n*} & Baseline~4 trained from random initialization with offline MC target budget $n$. \\
\texttt{init1\_n*} & Baseline~4 initialized from the base classifier with offline MC target budget $n$. \\
\bottomrule
\end{tabular}
\end{table*}

\subsection{Main-Figure Support Tables}
\label{app:main_support_tables}

Tables~\ref{tab:app_main_cert_boundary}--\ref{tab:app_main_crd_representative} provide the exact numerical values behind the main-body figures.  Table~\ref{tab:app_main_cert_boundary} supports Figure~\ref{fig:main_cert_boundary}, Table~\ref{tab:app_main_compute} supports Figure~\ref{fig:main_compute}, and Table~\ref{tab:app_main_crd_representative} supports Figure~\ref{fig:main_boundary_crd}.  The table captions repeat the conditioning set so that each table can be read independently.

\begin{table*}[t]
\centering
\caption{Exact values behind the main certified-performance and boundary-behavior figure. CertAcc@0 is computed over all test samples. OCA and Avg. Radius are computed over the boundary-confidence subset \(0.5<\widetilde p_A(\x)<0.75\). Boundary Mass reports the percentage of all test samples in that subset.}
\label{tab:app_main_cert_boundary}
\resizebox{\textwidth}{!}{%
\begin{tabular}{llcccc}
\toprule
Dataset & Method & CertAcc@0 (\%) & OCA (\%) & Avg. Radius & Boundary Mass \\
\midrule
FashionMNIST & RRISE & 87.44 $\pm$ 0.01 & 43.50 $\pm$ 0.21 & 0.08 $\pm$ 0.00 & 4.72 $\pm$ 0.01 \\
 & Baseline 1 & 87.63 $\pm$ 0.02 & 45.21 $\pm$ 0.35 & 0.08 $\pm$ 0.00 & 4.76 $\pm$ 0.02 \\
 & Baseline 2 & 87.61 $\pm$ 0.03 & 44.70 $\pm$ 0.54 & 0.07 $\pm$ 0.00 & 4.27 $\pm$ 0.09 \\
 & Baseline 3 & 87.64 $\pm$ 0.01 & 44.82 $\pm$ 0.08 & 0.08 $\pm$ 0.00 & 4.79 $\pm$ 0.02 \\
 & Baseline 4 & 87.50 $\pm$ 0.26 & 45.69 $\pm$ 1.88 & 0.08 $\pm$ 0.00 & 5.52 $\pm$ 0.95 \\
\midrule
CIFAR-10 & RRISE & 70.26 $\pm$ 0.07 & 43.68 $\pm$ 0.35 & 0.08 $\pm$ 0.00 & 16.10 $\pm$ 0.12 \\
 & Baseline 1 & 70.88 $\pm$ 0.03 & 41.63 $\pm$ 0.11 & 0.08 $\pm$ 0.00 & 16.99 $\pm$ 0.16 \\
 & Baseline 2 & 70.87 $\pm$ 0.05 & 42.64 $\pm$ 0.28 & 0.07 $\pm$ 0.00 & 14.92 $\pm$ 0.13 \\
 & Baseline 3 & 70.87 $\pm$ 0.01 & 41.61 $\pm$ 0.09 & 0.08 $\pm$ 0.00 & 17.04 $\pm$ 0.07 \\
 & Baseline 4 & 72.79 $\pm$ 0.46 & 84.74 $\pm$ 0.30 & 0.11 $\pm$ 0.01 & 67.58 $\pm$ 2.47 \\
\midrule
CIFAR-100 & RRISE & 33.91 $\pm$ 0.02 & 20.99 $\pm$ 0.66 & 0.03 $\pm$ 0.00 & 24.16 $\pm$ 0.56 \\
 & Baseline 1 & 34.75 $\pm$ 0.00 & 14.68 $\pm$ 0.01 & 0.03 $\pm$ 0.00 & 19.67 $\pm$ 0.09 \\
 & Baseline 2 & 34.82 $\pm$ 0.06 & 15.25 $\pm$ 0.27 & 0.03 $\pm$ 0.00 & 17.66 $\pm$ 0.23 \\
 & Baseline 3 & 34.78 $\pm$ 0.01 & 14.85 $\pm$ 0.12 & 0.03 $\pm$ 0.00 & 19.65 $\pm$ 0.11 \\
 & Baseline 4 & 17.76 $\pm$ 0.74 & -- & -- & 0.00 $\pm$ 0.00 \\
\midrule
Tiny ImageNet & RRISE & 27.43 $\pm$ 0.04 & 14.88 $\pm$ 0.38 & 0.03 $\pm$ 0.00 & 22.85 $\pm$ 0.60 \\
 & Baseline 1 & 27.98 $\pm$ 0.01 & 14.61 $\pm$ 0.02 & 0.03 $\pm$ 0.00 & 20.76 $\pm$ 0.08 \\
 & Baseline 2 & 28.02 $\pm$ 0.06 & 14.89 $\pm$ 0.72 & 0.03 $\pm$ 0.00 & 18.25 $\pm$ 0.22 \\
 & Baseline 3 & 27.94 $\pm$ 0.01 & 14.58 $\pm$ 0.15 & 0.03 $\pm$ 0.00 & 20.71 $\pm$ 0.07 \\
 & Baseline 4 & 22.37 $\pm$ 0.55 & -- & -- & 0.00 $\pm$ 0.00 \\
\midrule
\bottomrule
\end{tabular}%
}
\end{table*}

\begin{table*}[t]
\centering
\caption{Exact values behind the computational-efficiency figure. Test cost is the average number of forward-pass equivalents per input. Training cost counts offline MC target construction plus optimization up to the selected best checkpoint. Break-even columns report RRISE break-even queries relative to Baselines~1--3.}
\label{tab:app_main_compute}
\resizebox{\textwidth}{!}{%
\begin{tabular}{llccccc}
\toprule
Dataset & Method & Test Cost & Train Cost & RRISE BE vs B1 & RRISE BE vs B2 & RRISE BE vs B3 \\
\midrule
FashionMNIST & RRISE & 1.00 $\pm$ 0.00 & 600058056.00 $\pm$ 1276.37 & 60012.00 $\pm$ 0.00 & 276465.00 $\pm$ 757.91 & 396559.33 $\pm$ 812.58 \\
 & Baseline 1 & 10000.00 $\pm$ 0.00 & 0.00 $\pm$ 0.00 & -- & -- & -- \\
 & Baseline 2 & 6377.85 $\pm$ 10.99 & 0.00 $\pm$ 0.00 & -- & -- & -- \\
 & Baseline 3 & 6429.47 $\pm$ 2.31 & 0.00 $\pm$ 0.00 & -- & -- & -- \\
 & Baseline 4 & 1.00 $\pm$ 0.00 & 600013570.00 $\pm$ 9381.56 & -- & -- & -- \\
\midrule
CIFAR-10 & RRISE & 1.00 $\pm$ 0.00 & 500042924.00 $\pm$ 6447.92 & 50010.00 $\pm$ 1.00 & 210824.00 $\pm$ 904.85 & 181392.33 $\pm$ 152.67 \\
 & Baseline 1 & 10000.00 $\pm$ 0.00 & 0.00 $\pm$ 0.00 & -- & -- & -- \\
 & Baseline 2 & 5883.06 $\pm$ 11.69 & 0.00 $\pm$ 0.00 & -- & -- & -- \\
 & Baseline 3 & 7281.03 $\pm$ 1.89 & 0.00 $\pm$ 0.00 & -- & -- & -- \\
 & Baseline 4 & 1.00 $\pm$ 0.00 & 500039886.00 $\pm$ 6483.57 & -- & -- & -- \\
\midrule
CIFAR-100 & RRISE & 1.00 $\pm$ 0.00 & 500038220.00 $\pm$ 1836.03 & 50009.33 $\pm$ 0.58 & 212140.33 $\pm$ 1687.32 & 164667.33 $\pm$ 209.29 \\
 & Baseline 1 & 10000.00 $\pm$ 0.00 & 0.00 $\pm$ 0.00 & -- & -- & -- \\
 & Baseline 2 & 5579.45 $\pm$ 15.43 & 0.00 $\pm$ 0.00 & -- & -- & -- \\
 & Baseline 3 & 7413.80 $\pm$ 2.80 & 0.00 $\pm$ 0.00 & -- & -- & -- \\
 & Baseline 4 & 1.00 $\pm$ 0.00 & 500034692.00 $\pm$ 15288.00 & -- & -- & -- \\
\midrule
Tiny ImageNet & RRISE & 1.00 $\pm$ 0.00 & 1000119952.00 $\pm$ 10047.75 & 100022.33 $\pm$ 1.15 & 424920.00 $\pm$ 3328.49 & 316745.00 $\pm$ 247.57 \\
 & Baseline 1 & 10000.00 $\pm$ 0.00 & 0.00 $\pm$ 0.00 & -- & -- & -- \\
 & Baseline 2 & 5547.50 $\pm$ 5.51 & 0.00 $\pm$ 0.00 & -- & -- & -- \\
 & Baseline 3 & 7498.50 $\pm$ 1.73 & 0.00 $\pm$ 0.00 & -- & -- & -- \\
 & Baseline 4 & 1.00 $\pm$ 0.00 & 1000059780.00 $\pm$ 17386.65 & -- & -- & -- \\
\midrule
\bottomrule
\end{tabular}%
}
\end{table*}

\begin{table*}[t]
\centering
\caption{Representative boundary-confidence CRD values behind the main CRD figure. Each entry reports the percentage of all test inputs satisfying \(0.5<\widetilde p_A(\x)<0.75\) and \(\widetilde R(\x)>t\). For FashionMNIST/CIFAR-10, \((t_1,t_2,t_3)=(0.050,0.100,0.150)\). For CIFAR-100/Tiny ImageNet, \((t_1,t_2,t_3)=(0.025,0.050,0.075)\).}
\label{tab:app_main_crd_representative}
\resizebox{\textwidth}{!}{%
\begin{tabular}{llcccc}
\toprule
Dataset & Method & CRD($t_1$) & CRD($t_2$) & CRD($t_3$) & Boundary Mass \\
\midrule
FashionMNIST & RRISE & 0.03 $\pm$ 0.00 & 0.02 $\pm$ 0.00 & 0.01 $\pm$ 0.00 & 4.72 $\pm$ 0.01 \\
 & Baseline 1 & 0.03 $\pm$ 0.00 & 0.02 $\pm$ 0.00 & 0.00 $\pm$ 0.00 & 4.76 $\pm$ 0.02 \\
 & Baseline 2 & 0.03 $\pm$ 0.00 & 0.01 $\pm$ 0.00 & 0.00 $\pm$ 0.00 & 4.27 $\pm$ 0.09 \\
 & Baseline 3 & 0.03 $\pm$ 0.00 & 0.02 $\pm$ 0.00 & 0.00 $\pm$ 0.00 & 4.79 $\pm$ 0.02 \\
 & Baseline 4 & 0.04 $\pm$ 0.01 & 0.02 $\pm$ 0.01 & 0.01 $\pm$ 0.00 & 5.52 $\pm$ 0.95 \\
\midrule
CIFAR-10 & RRISE & 0.11 $\pm$ 0.00 & 0.06 $\pm$ 0.00 & 0.02 $\pm$ 0.00 & 16.10 $\pm$ 0.12 \\
 & Baseline 1 & 0.11 $\pm$ 0.00 & 0.07 $\pm$ 0.00 & 0.02 $\pm$ 0.00 & 16.99 $\pm$ 0.16 \\
 & Baseline 2 & 0.10 $\pm$ 0.00 & 0.04 $\pm$ 0.00 & 0.01 $\pm$ 0.00 & 14.92 $\pm$ 0.13 \\
 & Baseline 3 & 0.11 $\pm$ 0.00 & 0.07 $\pm$ 0.00 & 0.02 $\pm$ 0.00 & 17.04 $\pm$ 0.07 \\
 & Baseline 4 & 0.59 $\pm$ 0.03 & 0.47 $\pm$ 0.04 & 0.16 $\pm$ 0.14 & 67.58 $\pm$ 2.47 \\
\midrule
CIFAR-100 & RRISE & 0.16 $\pm$ 0.01 & 0.07 $\pm$ 0.00 & 0.00 $\pm$ 0.00 & 24.16 $\pm$ 0.56 \\
 & Baseline 1 & 0.12 $\pm$ 0.00 & 0.05 $\pm$ 0.00 & 0.00 $\pm$ 0.00 & 19.67 $\pm$ 0.09 \\
 & Baseline 2 & 0.10 $\pm$ 0.00 & 0.04 $\pm$ 0.00 & 0.00 $\pm$ 0.00 & 17.66 $\pm$ 0.23 \\
 & Baseline 3 & 0.12 $\pm$ 0.00 & 0.05 $\pm$ 0.00 & 0.00 $\pm$ 0.00 & 19.65 $\pm$ 0.11 \\
 & Baseline 4 & 0.00 $\pm$ 0.00 & 0.00 $\pm$ 0.00 & 0.00 $\pm$ 0.00 & 0.00 $\pm$ 0.00 \\
\midrule
Tiny ImageNet & RRISE & 0.14 $\pm$ 0.01 & 0.06 $\pm$ 0.00 & 0.00 $\pm$ 0.00 & 22.85 $\pm$ 0.60 \\
 & Baseline 1 & 0.12 $\pm$ 0.00 & 0.05 $\pm$ 0.00 & 0.00 $\pm$ 0.00 & 20.76 $\pm$ 0.08 \\
 & Baseline 2 & 0.10 $\pm$ 0.00 & 0.04 $\pm$ 0.00 & 0.00 $\pm$ 0.00 & 18.25 $\pm$ 0.22 \\
 & Baseline 3 & 0.12 $\pm$ 0.00 & 0.05 $\pm$ 0.00 & 0.00 $\pm$ 0.00 & 20.71 $\pm$ 0.07 \\
 & Baseline 4 & 0.00 $\pm$ 0.00 & 0.00 $\pm$ 0.00 & 0.00 $\pm$ 0.00 & 0.00 $\pm$ 0.00 \\
\midrule
\bottomrule
\end{tabular}%
}
\end{table*}

\subsection{Ablation Studies}
\label{app:ablation_list}

Unless otherwise stated, the main setting uses $n=10000$, the tighter $1\%$ setting for Baselines~2--3, head-only RRISE initialized from the base classifier, and Baseline~4 trained from random initialization following its original offline-surrogate setting.  For surrogate methods, the same calibration policy is applied to RRISE and Baseline~4 so that certified radii are comparable.

\begin{table*}[t]
\centering
\small
\caption{Ablation map.  Each ablation changes one family of choices while leaving the remaining main-setting choices fixed.}
\label{tab:app_ablation_map}
\begin{tabular}{p{0.18\textwidth}p{0.26\textwidth}p{0.23\textwidth}p{0.25\textwidth}}
\toprule
Ablation & Question answered & Varied quantity & Fixed quantities \\
\midrule
Budget & Does more MC evidence improve certification or boundary behavior? & MC budget / offline target budget & Calibration policy and method-specific main hyperparameters. \\
RRISE training strategy & Which parts of RRISE should be trained, and how important is initialization? & Trainable parameter subset and initialization & Offline target budget $n=10000$ and calibration policy. \\
Calibration level & How conservative is the surrogate after post-hoc calibration? & Total failure $\beta_{\mathrm{sur}}+\gamma_{\mathrm{sur}}$ & Trained surrogate checkpoint and evaluation protocol. \\
Baseline~2 decline & How much test-time sampling can the input-specific budget rule save? & Decline level & Maximum budget $k=10000$. \\
Baseline~3 tolerance & How aggressively can early stopping reduce MC sampling? & Stopping tolerance & Maximum budget $n=10000$. \\
Baseline~4 initialization & How sensitive is the JS surrogate to initialization? & Random vs. base-model initialization & Offline target budget $n=10000$ and calibration policy. \\
\bottomrule
\end{tabular}
\end{table*}

\paragraph{Budget ablation.}
Tables~\ref{tab:app_budget_ablation}--\ref{tab:app_budget_ablation_tiny_imagenet} vary the MC budget used by the MC baselines and the offline target budget used by the surrogate methods.  These tables are split by dataset to avoid a single long table.  RRISE and Baseline~4 keep test cost fixed at one forward-pass equivalent because their MC work is moved offline; Baseline~1 scales linearly with the online MC budget.

\begin{table*}[t]
\centering
\scriptsize
\caption{Budget ablation on FashionMNIST. The budget is the online MC sample count for Baseline~1 and the offline target-construction budget for surrogate methods. Boundary metrics use $0.5<\widetilde p_A(\x)<0.75$.}
\label{tab:app_budget_ablation}
\resizebox{\textwidth}{!}{%
\begin{tabular}{llccccc}
\toprule
Method & Budget & CertAcc@0 (\%) & OCA (\%) & Avg. Radius & Boundary Mass & Test Cost \\
\midrule
RRISE & 500 & 87.51 $\pm$ 0.01 & 45.66 $\pm$ 0.52 & 0.08 $\pm$ 0.00 & 4.69 $\pm$ 0.05 & 1.00 $\pm$ 0.00 \\
RRISE & 1000 & 87.49 $\pm$ 0.01 & 44.75 $\pm$ 0.06 & 0.09 $\pm$ 0.00 & 4.79 $\pm$ 0.01 & 1.00 $\pm$ 0.00 \\
RRISE & 5000 & 87.50 $\pm$ 0.02 & 45.55 $\pm$ 0.41 & 0.08 $\pm$ 0.00 & 4.72 $\pm$ 0.04 & 1.00 $\pm$ 0.00 \\
RRISE & 10000 & 87.44 $\pm$ 0.01 & 43.50 $\pm$ 0.21 & 0.08 $\pm$ 0.00 & 4.72 $\pm$ 0.01 & 1.00 $\pm$ 0.00 \\
\midrule
Baseline 1 & 500 & 87.65 $\pm$ 0.02 & 44.77 $\pm$ 0.11 & 0.08 $\pm$ 0.00 & 4.82 $\pm$ 0.07 & 500.00 $\pm$ 0.00 \\
Baseline 1 & 1000 & 87.64 $\pm$ 0.02 & 44.54 $\pm$ 0.26 & 0.08 $\pm$ 0.00 & 4.76 $\pm$ 0.06 & 1000.00 $\pm$ 0.00 \\
Baseline 1 & 5000 & 87.65 $\pm$ 0.01 & 44.89 $\pm$ 0.32 & 0.08 $\pm$ 0.00 & 4.80 $\pm$ 0.06 & 5000.00 $\pm$ 0.00 \\
Baseline 1 & 10000 & 87.63 $\pm$ 0.02 & 45.21 $\pm$ 0.35 & 0.08 $\pm$ 0.00 & 4.76 $\pm$ 0.02 & 10000.00 $\pm$ 0.00 \\
\midrule
Baseline 4 & 500 & 87.52 $\pm$ 0.24 & 47.41 $\pm$ 0.63 & 0.08 $\pm$ 0.00 & 6.12 $\pm$ 0.72 & 1.00 $\pm$ 0.00 \\
Baseline 4 & 1000 & 87.77 $\pm$ 0.07 & 48.32 $\pm$ 0.75 & 0.08 $\pm$ 0.00 & 6.67 $\pm$ 0.15 & 1.00 $\pm$ 0.00 \\
Baseline 4 & 5000 & 87.67 $\pm$ 0.17 & 47.53 $\pm$ 2.14 & 0.08 $\pm$ 0.00 & 6.31 $\pm$ 0.23 & 1.00 $\pm$ 0.00 \\
Baseline 4 & 10000 & 87.50 $\pm$ 0.26 & 45.69 $\pm$ 1.88 & 0.08 $\pm$ 0.00 & 5.52 $\pm$ 0.95 & 1.00 $\pm$ 0.00 \\
\bottomrule
\end{tabular}%
}
\end{table*}

\begin{table*}[t]
\centering
\scriptsize
\caption{Budget ablation on CIFAR-10. The budget is the online MC sample count for Baseline~1 and the offline target-construction budget for surrogate methods. Boundary metrics use $0.5<\widetilde p_A(\x)<0.75$.}
\label{tab:app_budget_ablation_cifar10}
\resizebox{\textwidth}{!}{%
\begin{tabular}{llccccc}
\toprule
Method & Budget & CertAcc@0 (\%) & OCA (\%) & Avg. Radius & Boundary Mass & Test Cost \\
\midrule
RRISE & 500 & 70.38 $\pm$ 0.05 & 44.72 $\pm$ 0.60 & 0.08 $\pm$ 0.00 & 16.84 $\pm$ 0.16 & 1.00 $\pm$ 0.00 \\
RRISE & 1000 & 70.36 $\pm$ 0.09 & 44.52 $\pm$ 0.30 & 0.09 $\pm$ 0.00 & 16.57 $\pm$ 0.09 & 1.00 $\pm$ 0.00 \\
RRISE & 5000 & 70.30 $\pm$ 0.04 & 44.13 $\pm$ 0.06 & 0.08 $\pm$ 0.00 & 16.26 $\pm$ 0.07 & 1.00 $\pm$ 0.00 \\
RRISE & 10000 & 70.26 $\pm$ 0.07 & 43.68 $\pm$ 0.35 & 0.08 $\pm$ 0.00 & 16.10 $\pm$ 0.12 & 1.00 $\pm$ 0.00 \\
\midrule
Baseline 1 & 500 & 70.98 $\pm$ 0.06 & 42.42 $\pm$ 0.36 & 0.08 $\pm$ 0.00 & 16.93 $\pm$ 0.09 & 500.00 $\pm$ 0.00 \\
Baseline 1 & 1000 & 70.94 $\pm$ 0.03 & 42.15 $\pm$ 0.23 & 0.08 $\pm$ 0.00 & 16.89 $\pm$ 0.11 & 1000.00 $\pm$ 0.00 \\
Baseline 1 & 5000 & 70.89 $\pm$ 0.02 & 41.92 $\pm$ 0.16 & 0.08 $\pm$ 0.00 & 17.05 $\pm$ 0.12 & 5000.00 $\pm$ 0.00 \\
Baseline 1 & 10000 & 70.88 $\pm$ 0.03 & 41.63 $\pm$ 0.11 & 0.08 $\pm$ 0.00 & 16.99 $\pm$ 0.16 & 10000.00 $\pm$ 0.00 \\
\midrule
Baseline 4 & 500 & 72.92 $\pm$ 0.65 & 86.47 $\pm$ 0.42 & 0.08 $\pm$ 0.01 & 62.45 $\pm$ 0.45 & 1.00 $\pm$ 0.00 \\
Baseline 4 & 1000 & 73.26 $\pm$ 0.46 & 85.85 $\pm$ 0.31 & 0.10 $\pm$ 0.02 & 64.92 $\pm$ 2.10 & 1.00 $\pm$ 0.00 \\
Baseline 4 & 5000 & 72.71 $\pm$ 0.29 & 81.46 $\pm$ 5.10 & 0.11 $\pm$ 0.01 & 58.54 $\pm$ 16.12 & 1.00 $\pm$ 0.00 \\
Baseline 4 & 10000 & 72.79 $\pm$ 0.46 & 84.74 $\pm$ 0.30 & 0.11 $\pm$ 0.01 & 67.58 $\pm$ 2.47 & 1.00 $\pm$ 0.00 \\
\bottomrule
\end{tabular}%
}
\end{table*}

\begin{table*}[t]
\centering
\scriptsize
\caption{Budget ablation on CIFAR-100. The budget is the online MC sample count for Baseline~1 and the offline target-construction budget for surrogate methods. Boundary metrics use $0.5<\widetilde p_A(\x)<0.75$.}
\label{tab:app_budget_ablation_cifar100}
\resizebox{\textwidth}{!}{%
\begin{tabular}{llccccc}
\toprule
Method & Budget & CertAcc@0 (\%) & OCA (\%) & Avg. Radius & Boundary Mass & Test Cost \\
\midrule
RRISE & 500 & 33.81 $\pm$ 0.03 & 22.62 $\pm$ 0.19 & 0.04 $\pm$ 0.00 & 28.13 $\pm$ 0.47 & 1.00 $\pm$ 0.00 \\
RRISE & 1000 & 33.89 $\pm$ 0.02 & 22.27 $\pm$ 0.91 & 0.04 $\pm$ 0.00 & 27.98 $\pm$ 1.31 & 1.00 $\pm$ 0.00 \\
RRISE & 5000 & 33.93 $\pm$ 0.03 & 21.51 $\pm$ 0.71 & 0.04 $\pm$ 0.00 & 24.84 $\pm$ 1.22 & 1.00 $\pm$ 0.00 \\
RRISE & 10000 & 33.91 $\pm$ 0.02 & 20.99 $\pm$ 0.66 & 0.03 $\pm$ 0.00 & 24.16 $\pm$ 0.56 & 1.00 $\pm$ 0.00 \\
\midrule
Baseline 1 & 500 & 34.83 $\pm$ 0.07 & 15.23 $\pm$ 0.03 & 0.03 $\pm$ 0.00 & 19.53 $\pm$ 0.06 & 500.00 $\pm$ 0.00 \\
Baseline 1 & 1000 & 34.83 $\pm$ 0.02 & 15.33 $\pm$ 0.20 & 0.03 $\pm$ 0.00 & 19.62 $\pm$ 0.05 & 1000.00 $\pm$ 0.00 \\
Baseline 1 & 5000 & 34.77 $\pm$ 0.03 & 14.79 $\pm$ 0.02 & 0.03 $\pm$ 0.00 & 19.63 $\pm$ 0.07 & 5000.00 $\pm$ 0.00 \\
Baseline 1 & 10000 & 34.75 $\pm$ 0.00 & 14.68 $\pm$ 0.01 & 0.03 $\pm$ 0.00 & 19.67 $\pm$ 0.09 & 10000.00 $\pm$ 0.00 \\
\midrule
Baseline 4 & 500 & 17.97 $\pm$ 0.79 & -- & -- & 0.00 $\pm$ 0.00 & 1.00 $\pm$ 0.00 \\
Baseline 4 & 1000 & 17.47 $\pm$ 1.08 & -- & -- & 0.00 $\pm$ 0.00 & 1.00 $\pm$ 0.00 \\
Baseline 4 & 5000 & 17.27 $\pm$ 0.82 & -- & -- & 0.00 $\pm$ 0.00 & 1.00 $\pm$ 0.00 \\
Baseline 4 & 10000 & 17.76 $\pm$ 0.74 & -- & -- & 0.00 $\pm$ 0.00 & 1.00 $\pm$ 0.00 \\
\bottomrule
\end{tabular}%
}
\end{table*}

\begin{table*}[t]
\centering
\scriptsize
\caption{Budget ablation on Tiny ImageNet. The budget is the online MC sample count for Baseline~1 and the offline target-construction budget for surrogate methods. Boundary metrics use $0.5<\widetilde p_A(\x)<0.75$.}
\label{tab:app_budget_ablation_tiny_imagenet}
\resizebox{\textwidth}{!}{%
\begin{tabular}{llccccc}
\toprule
Method & Budget & CertAcc@0 (\%) & OCA (\%) & Avg. Radius & Boundary Mass & Test Cost \\
\midrule
RRISE & 500 & 27.33 $\pm$ 0.01 & 14.72 $\pm$ 0.20 & 0.03 $\pm$ 0.00 & 24.15 $\pm$ 0.46 & 1.00 $\pm$ 0.00 \\
RRISE & 1000 & 27.41 $\pm$ 0.03 & 15.24 $\pm$ 0.37 & 0.03 $\pm$ 0.00 & 23.68 $\pm$ 0.57 & 1.00 $\pm$ 0.00 \\
RRISE & 5000 & 27.50 $\pm$ 0.01 & 15.31 $\pm$ 0.19 & 0.03 $\pm$ 0.00 & 22.48 $\pm$ 0.12 & 1.00 $\pm$ 0.00 \\
RRISE & 10000 & 27.43 $\pm$ 0.04 & 14.88 $\pm$ 0.38 & 0.03 $\pm$ 0.00 & 22.85 $\pm$ 0.60 & 1.00 $\pm$ 0.00 \\
\midrule
Baseline 1 & 500 & 27.96 $\pm$ 0.06 & 14.70 $\pm$ 0.28 & 0.03 $\pm$ 0.00 & 20.55 $\pm$ 0.16 & 500.00 $\pm$ 0.00 \\
Baseline 1 & 1000 & 27.94 $\pm$ 0.02 & 14.94 $\pm$ 0.33 & 0.03 $\pm$ 0.00 & 20.61 $\pm$ 0.19 & 1000.00 $\pm$ 0.00 \\
Baseline 1 & 5000 & 27.95 $\pm$ 0.01 & 14.70 $\pm$ 0.07 & 0.03 $\pm$ 0.00 & 20.77 $\pm$ 0.09 & 5000.00 $\pm$ 0.00 \\
Baseline 1 & 10000 & 27.98 $\pm$ 0.01 & 14.61 $\pm$ 0.02 & 0.03 $\pm$ 0.00 & 20.76 $\pm$ 0.08 & 10000.00 $\pm$ 0.00 \\
\midrule
Baseline 4 & 500 & 22.10 $\pm$ 0.36 & -- & -- & 0.00 $\pm$ 0.00 & 1.00 $\pm$ 0.00 \\
Baseline 4 & 1000 & 22.04 $\pm$ 0.63 & -- & -- & 0.00 $\pm$ 0.00 & 1.00 $\pm$ 0.00 \\
Baseline 4 & 5000 & 22.07 $\pm$ 0.31 & -- & -- & 0.00 $\pm$ 0.00 & 1.00 $\pm$ 0.00 \\
Baseline 4 & 10000 & 22.37 $\pm$ 0.55 & -- & -- & 0.00 $\pm$ 0.00 & 1.00 $\pm$ 0.00 \\
\bottomrule
\end{tabular}%
}
\end{table*}

\paragraph{RRISE training strategy.}
Table~\ref{tab:app_rrise_training_ablation} compares RRISE variants with different trainable parameter subsets and initialization choices.  The main-body configuration is \texttt{trainhead\_initbase\_n10000}.  The comparison is useful because high CertAcc@0 alone does not guarantee a meaningful boundary subset: some variants achieve high or low certified accuracy while producing zero boundary mass after calibration.

\begin{table*}[t]
\centering
\scriptsize
\caption{RRISE training-strategy ablation. This table isolates which parameters are trained and whether the network is initialized from the base classifier or randomly.}
\label{tab:app_rrise_training_ablation}
\resizebox{\textwidth}{!}{%
\begin{tabular}{llcccccc}
\toprule
Dataset & Training Strategy & CertAcc@0 (\%) & OCA (\%) & Avg. Radius & Boundary Mass & $\delta$ & Best Epoch \\
\midrule
FashionMNIST & trainall\_initbase\_n10000 & 87.90 $\pm$ 0.05 & 49.60 $\pm$ 1.43 & 0.08 $\pm$ 0.00 & 5.08 $\pm$ 0.23 & 0.00 $\pm$ 0.00 & 6.67 $\pm$ 4.16 \\
FashionMNIST & trainall\_initrandom\_n10000 & 87.71 $\pm$ 0.06 & 46.79 $\pm$ 0.66 & 0.08 $\pm$ 0.00 & 6.43 $\pm$ 0.12 & 0.00 $\pm$ 0.00 & 70.67 $\pm$ 15.95 \\
FashionMNIST & trainhead\_initbase\_n10000 & 87.44 $\pm$ 0.01 & 43.50 $\pm$ 0.21 & 0.08 $\pm$ 0.00 & 4.72 $\pm$ 0.01 & 0.00 $\pm$ 0.00 & 164.00 $\pm$ 3.61 \\
\midrule
CIFAR-10 & trainall\_initbase\_n10000 & 70.93 $\pm$ 0.10 & 42.51 $\pm$ 0.33 & 0.08 $\pm$ 0.00 & 16.52 $\pm$ 0.37 & 0.01 $\pm$ 0.00 & 1.67 $\pm$ 0.58 \\
CIFAR-10 & trainall\_initrandom\_n10000 & 72.56 $\pm$ 0.38 & 67.59 $\pm$ 4.91 & 0.09 $\pm$ 0.01 & 30.22 $\pm$ 5.75 & 0.19 $\pm$ 0.04 & 124.33 $\pm$ 65.24 \\
CIFAR-10 & trainhead\_initbase\_n10000 & 70.26 $\pm$ 0.07 & 43.68 $\pm$ 0.35 & 0.08 $\pm$ 0.00 & 16.10 $\pm$ 0.12 & 0.02 $\pm$ 0.01 & 146.00 $\pm$ 21.93 \\
\midrule
CIFAR-100 & trainall\_initbase\_n10000 & 35.80 $\pm$ 0.09 & 20.80 $\pm$ 0.77 & 0.03 $\pm$ 0.00 & 22.51 $\pm$ 0.19 & 0.00 $\pm$ 0.00 & 2.00 $\pm$ 0.00 \\
CIFAR-100 & trainall\_initrandom\_n10000 & 24.21 $\pm$ 0.57 & -- & -- & 0.00 $\pm$ 0.00 & 0.73 $\pm$ 0.03 & 208.33 $\pm$ 117.07 \\
CIFAR-100 & trainhead\_initbase\_n10000 & 33.91 $\pm$ 0.02 & 20.99 $\pm$ 0.66 & 0.03 $\pm$ 0.00 & 24.16 $\pm$ 0.56 & 0.16 $\pm$ 0.01 & 130.00 $\pm$ 6.24 \\
\midrule
Tiny ImageNet & trainall\_initbase\_n10000 & 27.87 $\pm$ 0.30 & 18.40 $\pm$ 0.27 & 0.03 $\pm$ 0.00 & 26.58 $\pm$ 0.37 & 0.01 $\pm$ 0.02 & 3.00 $\pm$ 0.00 \\
Tiny ImageNet & trainall\_initrandom\_n10000 & 24.71 $\pm$ 0.42 & -- & -- & 0.00 $\pm$ 0.00 & 0.57 $\pm$ 0.04 & 438.67 $\pm$ 24.95 \\
Tiny ImageNet & trainhead\_initbase\_n10000 & 27.43 $\pm$ 0.04 & 14.88 $\pm$ 0.38 & 0.03 $\pm$ 0.00 & 22.85 $\pm$ 0.60 & 0.10 $\pm$ 0.02 & 204.00 $\pm$ 17.09 \\
\bottomrule
\end{tabular}%
}
\end{table*}

\paragraph{Calibration level.}
Tables~\ref{tab:app_calibration_ablation}--\ref{tab:app_calibration_ablation_tiny_imagenet} vary the surrogate total failure level $\beta_{\mathrm{sur}}+\gamma_{\mathrm{sur}}$.  Smaller total-failure values impose stricter calibration, which increases the calibration offset $\delta$ and can shrink radii or eliminate the boundary-confidence subset.  The split tables make the calibration behavior of RRISE and Baseline~4 directly comparable on each dataset.

\begin{table*}[t]
\centering
\scriptsize
\caption{Calibration-level ablation on FashionMNIST. Total Failure denotes $\beta_{\mathrm{sur}}+\gamma_{\mathrm{sur}}$; smaller values impose stricter surrogate calibration.}
\label{tab:app_calibration_ablation}
\resizebox{\textwidth}{!}{%
\begin{tabular}{llccccc}
\toprule
Method & Total Failure & $\delta$ & CertAcc@0 (\%) & OCA (\%) & Avg. Radius & Boundary Mass \\
\midrule
RRISE & 0.01 & 0.28 $\pm$ 0.02 & 87.44 $\pm$ 0.01 & 90.14 $\pm$ 0.17 & 0.14 $\pm$ 0.01 & 94.11 $\pm$ 0.43 \\
RRISE & 0.05 & 0.07 $\pm$ 0.01 & 87.44 $\pm$ 0.01 & 46.42 $\pm$ 0.42 & 0.09 $\pm$ 0.00 & 5.18 $\pm$ 0.13 \\
RRISE & 0.1 & 0.00 $\pm$ 0.00 & 87.44 $\pm$ 0.01 & 43.42 $\pm$ 0.20 & 0.08 $\pm$ 0.00 & 4.71 $\pm$ 0.03 \\
RRISE & 0.25 & 0.00 $\pm$ 0.00 & 87.44 $\pm$ 0.01 & 43.50 $\pm$ 0.21 & 0.08 $\pm$ 0.00 & 4.72 $\pm$ 0.01 \\
\midrule
Baseline 4 & 0.01 & 0.82 $\pm$ 0.17 & 87.50 $\pm$ 0.26 & -- & -- & 0.00 $\pm$ 0.00 \\
Baseline 4 & 0.05 & 0.33 $\pm$ 0.12 & 87.50 $\pm$ 0.26 & 80.72 $\pm$ 19.08 & 0.07 $\pm$ 0.02 & 62.86 $\pm$ 46.20 \\
Baseline 4 & 0.1 & 0.03 $\pm$ 0.02 & 87.50 $\pm$ 0.26 & 46.43 $\pm$ 1.09 & 0.09 $\pm$ 0.00 & 5.70 $\pm$ 0.97 \\
Baseline 4 & 0.25 & 0.00 $\pm$ 0.00 & 87.50 $\pm$ 0.26 & 45.69 $\pm$ 1.88 & 0.08 $\pm$ 0.00 & 5.52 $\pm$ 0.95 \\
\bottomrule
\end{tabular}%
}
\end{table*}

\begin{table*}[t]
\centering
\scriptsize
\caption{Calibration-level ablation on CIFAR-10. Total Failure denotes $\beta_{\mathrm{sur}}+\gamma_{\mathrm{sur}}$; smaller values impose stricter surrogate calibration.}
\label{tab:app_calibration_ablation_cifar10}
\resizebox{\textwidth}{!}{%
\begin{tabular}{llccccc}
\toprule
Method & Total Failure & $\delta$ & CertAcc@0 (\%) & OCA (\%) & Avg. Radius & Boundary Mass \\
\midrule
RRISE & 0.01 & 0.40 $\pm$ 0.02 & 70.26 $\pm$ 0.07 & 84.12 $\pm$ 1.09 & 0.05 $\pm$ 0.01 & 65.26 $\pm$ 3.21 \\
RRISE & 0.05 & 0.22 $\pm$ 0.01 & 70.26 $\pm$ 0.07 & 56.35 $\pm$ 1.40 & 0.10 $\pm$ 0.00 & 27.59 $\pm$ 2.20 \\
RRISE & 0.1 & 0.14 $\pm$ 0.01 & 70.26 $\pm$ 0.07 & 49.11 $\pm$ 0.27 & 0.09 $\pm$ 0.00 & 19.41 $\pm$ 0.34 \\
RRISE & 0.25 & 0.02 $\pm$ 0.01 & 70.26 $\pm$ 0.07 & 43.68 $\pm$ 0.35 & 0.08 $\pm$ 0.00 & 16.10 $\pm$ 0.12 \\
\midrule
Baseline 4 & 0.01 & 0.96 $\pm$ 0.03 & 72.79 $\pm$ 0.46 & -- & -- & 0.00 $\pm$ 0.00 \\
Baseline 4 & 0.05 & 0.76 $\pm$ 0.02 & 72.79 $\pm$ 0.46 & -- & -- & 0.00 $\pm$ 0.00 \\
Baseline 4 & 0.1 & 0.61 $\pm$ 0.02 & 72.79 $\pm$ 0.46 & -- & -- & 0.00 $\pm$ 0.00 \\
Baseline 4 & 0.25 & 0.28 $\pm$ 0.02 & 72.79 $\pm$ 0.46 & 84.74 $\pm$ 0.30 & 0.11 $\pm$ 0.01 & 67.58 $\pm$ 2.47 \\
\bottomrule
\end{tabular}%
}
\end{table*}

\begin{table*}[t]
\centering
\scriptsize
\caption{Calibration-level ablation on CIFAR-100. Total Failure denotes $\beta_{\mathrm{sur}}+\gamma_{\mathrm{sur}}$; smaller values impose stricter surrogate calibration.}
\label{tab:app_calibration_ablation_cifar100}
\resizebox{\textwidth}{!}{%
\begin{tabular}{llccccc}
\toprule
Method & Total Failure & $\delta$ & CertAcc@0 (\%) & OCA (\%) & Avg. Radius & Boundary Mass \\
\midrule
RRISE & 0.01 & 0.76 $\pm$ 0.04 & 33.91 $\pm$ 0.02 & -- & -- & 0.00 $\pm$ 0.00 \\
RRISE & 0.05 & 0.57 $\pm$ 0.02 & 33.91 $\pm$ 0.02 & -- & -- & 0.00 $\pm$ 0.00 \\
RRISE & 0.1 & 0.44 $\pm$ 0.01 & 33.91 $\pm$ 0.02 & 55.18 $\pm$ 0.86 & 0.01 $\pm$ 0.00 & 42.83 $\pm$ 1.31 \\
RRISE & 0.25 & 0.16 $\pm$ 0.01 & 33.91 $\pm$ 0.02 & 20.99 $\pm$ 0.66 & 0.03 $\pm$ 0.00 & 24.16 $\pm$ 0.56 \\
\midrule
Baseline 4 & 0.01 & 1.00 $\pm$ 0.00 & 17.76 $\pm$ 0.74 & -- & -- & 0.00 $\pm$ 0.00 \\
Baseline 4 & 0.05 & 1.00 $\pm$ 0.00 & 17.76 $\pm$ 0.74 & -- & -- & 0.00 $\pm$ 0.00 \\
Baseline 4 & 0.1 & 1.00 $\pm$ 0.00 & 17.76 $\pm$ 0.74 & -- & -- & 0.00 $\pm$ 0.00 \\
Baseline 4 & 0.25 & 0.98 $\pm$ 0.01 & 17.76 $\pm$ 0.74 & -- & -- & 0.00 $\pm$ 0.00 \\
\bottomrule
\end{tabular}%
}
\end{table*}

\begin{table*}[t]
\centering
\scriptsize
\caption{Calibration-level ablation on Tiny ImageNet. Total Failure denotes $\beta_{\mathrm{sur}}+\gamma_{\mathrm{sur}}$; smaller values impose stricter surrogate calibration.}
\label{tab:app_calibration_ablation_tiny_imagenet}
\resizebox{\textwidth}{!}{%
\begin{tabular}{llccccc}
\toprule
Method & Total Failure & $\delta$ & CertAcc@0 (\%) & OCA (\%) & Avg. Radius & Boundary Mass \\
\midrule
RRISE & 0.01 & 0.66 $\pm$ 0.03 & 27.43 $\pm$ 0.04 & -- & -- & 0.00 $\pm$ 0.00 \\
RRISE & 0.05 & 0.45 $\pm$ 0.02 & 27.43 $\pm$ 0.04 & 44.10 $\pm$ 2.39 & 0.01 $\pm$ 0.00 & 43.20 $\pm$ 4.79 \\
RRISE & 0.1 & 0.33 $\pm$ 0.00 & 27.43 $\pm$ 0.04 & 36.49 $\pm$ 0.05 & 0.04 $\pm$ 0.00 & 61.54 $\pm$ 0.24 \\
RRISE & 0.25 & 0.10 $\pm$ 0.02 & 27.43 $\pm$ 0.04 & 14.88 $\pm$ 0.38 & 0.03 $\pm$ 0.00 & 22.85 $\pm$ 0.60 \\
\midrule
Baseline 4 & 0.01 & 0.99 $\pm$ 0.00 & 22.37 $\pm$ 0.55 & -- & -- & 0.00 $\pm$ 0.00 \\
Baseline 4 &  0.05 & 0.94 $\pm$ 0.01 & 22.37 $\pm$ 0.55 & -- & -- & 0.00 $\pm$ 0.00 \\
Baseline 4 & 0.1 & 0.86 $\pm$ 0.04 & 22.37 $\pm$ 0.55 & -- & -- & 0.00 $\pm$ 0.00 \\
Baseline 4 & 0.25 & 0.65 $\pm$ 0.07 & 22.37 $\pm$ 0.55 & -- & -- & 0.00 $\pm$ 0.00 \\
\bottomrule
\end{tabular}%
}
\end{table*}

\paragraph{Baseline~2 decline level.}
Table~\ref{tab:app_baseline2_ablation} reports the Baseline~2 ablation over decline levels.  The decline level controls the input-specific budgeting rule.  Larger decline levels substantially reduce test-time cost through higher sample-reduction factors, while the reported CertAcc@0 changes only mildly in these runs.

\begin{table*}[t]
\centering
\scriptsize
\caption{Baseline~2 decline-level ablation.  Sample Red. is the fixed-budget sample count divided by the realized average test cost.}
\label{tab:app_baseline2_ablation}
\resizebox{\textwidth}{!}{%
\begin{tabular}{llcccccc}
\toprule
Dataset & Variant & CertAcc@0 (\%) & OCA (\%) & Avg. Radius & Boundary Mass & Test Cost & Sample Red. \\
\midrule
FashionMNIST & k10000\_decline0.01 & 87.61 $\pm$ 0.03 & 44.70 $\pm$ 0.54 & 0.07 $\pm$ 0.00 & 4.27 $\pm$ 0.09 & 6377.85 $\pm$ 10.99 & 1.57 $\pm$ 0.00 \\
FashionMNIST & k10000\_decline0.05 & 87.65 $\pm$ 0.03 & 44.79 $\pm$ 0.83 & 0.07 $\pm$ 0.00 & 4.23 $\pm$ 0.11 & 2171.48 $\pm$ 5.94 & 4.61 $\pm$ 0.01 \\
\midrule
CIFAR-10 & k10000\_decline0.01 & 70.87 $\pm$ 0.05 & 42.64 $\pm$ 0.28 & 0.07 $\pm$ 0.00 & 14.92 $\pm$ 0.13 & 5883.06 $\pm$ 11.69 & 1.70 $\pm$ 0.00 \\
CIFAR-10 & k10000\_decline0.05 & 70.96 $\pm$ 0.20 & 42.17 $\pm$ 0.19 & 0.07 $\pm$ 0.00 & 14.95 $\pm$ 0.17 & 2372.89 $\pm$ 10.18 & 4.21 $\pm$ 0.02 \\
\midrule
CIFAR-100 & k10000\_decline0.01 & 34.82 $\pm$ 0.06 & 15.25 $\pm$ 0.27 & 0.03 $\pm$ 0.00 & 17.66 $\pm$ 0.23 & 5579.45 $\pm$ 15.43 & 1.79 $\pm$ 0.00 \\
CIFAR-100 & k10000\_decline0.05 & 34.82 $\pm$ 0.07 & 15.04 $\pm$ 0.30 & 0.03 $\pm$ 0.00 & 17.51 $\pm$ 0.17 & 2358.21 $\pm$ 18.83 & 4.24 $\pm$ 0.03 \\
\midrule
Tiny ImageNet & k10000\_decline0.01 & 28.02 $\pm$ 0.06 & 14.89 $\pm$ 0.72 & 0.03 $\pm$ 0.00 & 18.25 $\pm$ 0.22 & 5547.50 $\pm$ 5.51 & 1.80 $\pm$ 0.00 \\
Tiny ImageNet & k10000\_decline0.05 & 28.01 $\pm$ 0.09 & 14.55 $\pm$ 0.51 & 0.03 $\pm$ 0.00 & 18.19 $\pm$ 0.19 & 2354.77 $\pm$ 18.48 & 4.25 $\pm$ 0.03 \\
\bottomrule
\end{tabular}%
}
\end{table*}

\paragraph{Baseline~3 stopping tolerance.}
Table~\ref{tab:app_baseline3_ablation} reports the Baseline~3 ablation over stopping tolerances.  The stopping tolerance controls how aggressively the early-stopping rule reduces MC sampling.  Larger tolerance values reduce test cost more strongly; the table should therefore be read as a cost--certification tradeoff rather than a pure accuracy comparison.

\begin{table*}[t]
\centering
\scriptsize
\caption{Baseline~3 stopping-tolerance ablation.  Larger tolerance values allow earlier termination and therefore lower average test cost.}
\label{tab:app_baseline3_ablation}
\resizebox{\textwidth}{!}{%
\begin{tabular}{llcccccc}
\toprule
Dataset & Variant & CertAcc@0 (\%) & OCA (\%) & Avg. Radius & Boundary Mass & Test Cost & Sample Red. \\
\midrule
FashionMNIST & n10000\_tol0.01 & 87.64 $\pm$ 0.01 & 44.82 $\pm$ 0.08 & 0.08 $\pm$ 0.00 & 4.79 $\pm$ 0.02 & 6429.47 $\pm$ 2.31 & 1.56 $\pm$ 0.00 \\
FashionMNIST & n10000\_tol0.05 & 87.65 $\pm$ 0.03 & 44.45 $\pm$ 0.33 & 0.08 $\pm$ 0.00 & 4.84 $\pm$ 0.01 & 1514.17 $\pm$ 3.10 & 6.60 $\pm$ 0.01 \\
\midrule
CIFAR-10 & n10000\_tol0.01 & 70.87 $\pm$ 0.01 & 41.61 $\pm$ 0.09 & 0.08 $\pm$ 0.00 & 17.04 $\pm$ 0.07 & 7281.03 $\pm$ 1.89 & 1.37 $\pm$ 0.00 \\
CIFAR-10 & n10000\_tol0.05 & 70.86 $\pm$ 0.01 & 41.90 $\pm$ 0.16 & 0.08 $\pm$ 0.00 & 17.21 $\pm$ 0.04 & 2757.70 $\pm$ 2.35 & 3.63 $\pm$ 0.00 \\
\midrule
CIFAR-100 & n10000\_tol0.01 & 34.78 $\pm$ 0.01 & 14.85 $\pm$ 0.12 & 0.03 $\pm$ 0.00 & 19.65 $\pm$ 0.11 & 7413.80 $\pm$ 2.80 & 1.35 $\pm$ 0.00 \\
CIFAR-100 & n10000\_tol0.05 & 34.78 $\pm$ 0.03 & 15.05 $\pm$ 0.14 & 0.03 $\pm$ 0.00 & 19.80 $\pm$ 0.03 & 3037.67 $\pm$ 3.87 & 3.29 $\pm$ 0.00 \\
\midrule
Tiny ImageNet & n10000\_tol0.01 & 27.94 $\pm$ 0.01 & 14.58 $\pm$ 0.15 & 0.03 $\pm$ 0.00 & 20.71 $\pm$ 0.07 & 7498.50 $\pm$ 1.73 & 1.33 $\pm$ 0.00 \\
Tiny ImageNet & n10000\_tol0.05 & 27.94 $\pm$ 0.01 & 14.57 $\pm$ 0.12 & 0.03 $\pm$ 0.00 & 21.00 $\pm$ 0.09 & 3158.50 $\pm$ 2.45 & 3.17 $\pm$ 0.00 \\
\bottomrule
\end{tabular}%
}
\end{table*}

\paragraph{Baseline~4 initialization.}
Table~\ref{tab:app_baseline4_init_ablation} reports the Baseline~4 initialization ablation.  The main comparison uses random initialization (\texttt{init0}), following the original offline-surrogate setting.  We also report base-model initialization (\texttt{init1}) to quantify how strongly the Jensen--Shannon surrogate depends on initialization.

\begin{table*}[t]
\centering
\scriptsize
\caption{Baseline~4 initialization ablation.  	exttt{init0} is random initialization and 	exttt{init1} is base-model initialization.}
\label{tab:app_baseline4_init_ablation}
\resizebox{\textwidth}{!}{%
\begin{tabular}{llcccccc}
\toprule
Dataset & Variant & CertAcc@0 (\%) & OCA (\%) & Avg. Radius & Boundary Mass & \(\delta\) & Best Epoch \\
\midrule
FashionMNIST & init0\_n10000 & 87.50 $\pm$ 0.26 & 45.69 $\pm$ 1.88 & 0.08 $\pm$ 0.00 & 5.52 $\pm$ 0.95 & 0.00 $\pm$ 0.00 & 38.33 $\pm$ 26.50 \\
FashionMNIST & init1\_n10000 & 87.96 $\pm$ 0.05 & 47.62 $\pm$ 0.45 & 0.08 $\pm$ 0.00 & 4.54 $\pm$ 0.32 & 0.00 $\pm$ 0.00 & 13.33 $\pm$ 8.96 \\
\midrule
CIFAR-10 & init0\_n10000 & 72.79 $\pm$ 0.46 & 84.74 $\pm$ 0.30 & 0.11 $\pm$ 0.01 & 67.58 $\pm$ 2.47 & 0.28 $\pm$ 0.02 & 135.67 $\pm$ 22.05 \\
CIFAR-10 & init1\_n10000 & 72.77 $\pm$ 0.45 & 46.67 $\pm$ 0.15 & 0.08 $\pm$ 0.00 & 16.83 $\pm$ 0.10 & 0.03 $\pm$ 0.01 & 8.00 $\pm$ 4.36 \\
\midrule
CIFAR-100 & init0\_n10000 & 17.76 $\pm$ 0.74 & -- & -- & 0.00 $\pm$ 0.00 & 0.98 $\pm$ 0.01 & 118.00 $\pm$ 52.00 \\
CIFAR-100 & init1\_n10000 & 40.04 $\pm$ 0.17 & 74.27 $\pm$ 0.49 & 0.02 $\pm$ 0.00 & 29.84 $\pm$ 0.82 & 0.42 $\pm$ 0.01 & 1.00 $\pm$ 0.00 \\
\midrule
Tiny ImageNet & init0\_n10000 & 22.37 $\pm$ 0.55 & -- & -- & 0.00 $\pm$ 0.00 & 0.65 $\pm$ 0.07 & 101.67 $\pm$ 29.57 \\
Tiny ImageNet & init1\_n10000 & 29.06 $\pm$ 0.08 & 30.01 $\pm$ 8.67 & 0.04 $\pm$ 0.01 & 42.23 $\pm$ 14.02 & 0.23 $\pm$ 0.02 & 15.33 $\pm$ 6.51 \\
\bottomrule
\end{tabular}%
}
\end{table*}

\subsection{Full Certified Radius Distribution Results}
\label{app:full_crd_results}

The main text reports representative CRD values and CRD curves.  For completeness, Tables~\ref{tab:app_crd_boundary_full}--\ref{tab:app_crd_certified_full_tiny_imagenet} report the complete CRD results over all evaluated thresholds, datasets, methods, budgets, and variants.  The original long-row format has been converted into a wide threshold format: each row is one method--variant--budget setting, and the threshold columns contain the same values as the original rows.

\paragraph{Boundary-confidence CRD.}
Tables~\ref{tab:app_crd_boundary_full}--\ref{tab:app_crd_boundary_full_tiny_imagenet} report the fraction of all test samples that both fall in the boundary-confidence subset $0.5<\widetilde p_A(\x)<0.75$ and satisfy $\widetilde R(\x)>t$.  The thresholds differ by dataset family: FashionMNIST and CIFAR-10 use $t\in\{0,0.025,\ldots,0.175\}$; CIFAR-100 and Tiny ImageNet use $t\in\{0,0.0125,\ldots,0.0875\}$.

\begin{table*}[t]
\centering
\scriptsize
\caption{Full boundary-confidence CRD results on FashionMNIST. Entries are percentages of all test inputs satisfying both the boundary-confidence condition and the radius threshold.}
\label{tab:app_crd_boundary_full}
\resizebox{\textwidth}{!}{%
\begin{tabular}{lllcccccccc}
\toprule
Method & Variant & Budget & $t=0$ & $t=0.025$ & $t=0.05$ & $t=0.075$ & $t=0.1$ & $t=0.125$ & $t=0.15$ & $t=0.175$ \\
\midrule
RRISE & trainhead\_initbase\_n500 & 500 & 0.05 $\pm$ 0.00 & 0.04 $\pm$ 0.00 & 0.03 $\pm$ 0.00 & 0.03 $\pm$ 0.00 & 0.02 $\pm$ 0.00 & 0.01 $\pm$ 0.00 & 0.01 $\pm$ 0.00 & 0.00 $\pm$ 0.00 \\
RRISE & trainhead\_initbase\_n1000 & 1000 & 0.05 $\pm$ 0.00 & 0.04 $\pm$ 0.00 & 0.03 $\pm$ 0.00 & 0.03 $\pm$ 0.00 & 0.02 $\pm$ 0.00 & 0.01 $\pm$ 0.00 & 0.01 $\pm$ 0.00 & 0.00 $\pm$ 0.00 \\
RRISE & trainhead\_initbase\_n5000 & 5000 & 0.05 $\pm$ 0.00 & 0.04 $\pm$ 0.00 & 0.03 $\pm$ 0.00 & 0.03 $\pm$ 0.00 & 0.02 $\pm$ 0.00 & 0.01 $\pm$ 0.00 & 0.01 $\pm$ 0.00 & 0.00 $\pm$ 0.00 \\
RRISE & trainhead\_initbase\_n10000 & 10000 & 0.05 $\pm$ 0.00 & 0.04 $\pm$ 0.00 & 0.03 $\pm$ 0.00 & 0.03 $\pm$ 0.00 & 0.02 $\pm$ 0.00 & 0.01 $\pm$ 0.00 & 0.01 $\pm$ 0.00 & 0.00 $\pm$ 0.00 \\
RRISE & trainall\_initbase\_n500 & 500 & 0.05 $\pm$ 0.00 & 0.04 $\pm$ 0.00 & 0.04 $\pm$ 0.00 & 0.03 $\pm$ 0.00 & 0.02 $\pm$ 0.00 & 0.01 $\pm$ 0.00 & 0.01 $\pm$ 0.00 & 0.00 $\pm$ 0.00 \\
RRISE & trainall\_initbase\_n1000 & 1000 & 0.05 $\pm$ 0.00 & 0.04 $\pm$ 0.00 & 0.04 $\pm$ 0.00 & 0.03 $\pm$ 0.00 & 0.02 $\pm$ 0.00 & 0.01 $\pm$ 0.00 & 0.01 $\pm$ 0.00 & 0.00 $\pm$ 0.00 \\
RRISE & trainall\_initbase\_n5000 & 5000 & 0.05 $\pm$ 0.00 & 0.04 $\pm$ 0.00 & 0.04 $\pm$ 0.00 & 0.03 $\pm$ 0.00 & 0.02 $\pm$ 0.00 & 0.01 $\pm$ 0.00 & 0.01 $\pm$ 0.00 & 0.00 $\pm$ 0.00 \\
RRISE & trainall\_initbase\_n10000 & 10000 & 0.05 $\pm$ 0.00 & 0.04 $\pm$ 0.00 & 0.04 $\pm$ 0.00 & 0.03 $\pm$ 0.00 & 0.02 $\pm$ 0.00 & 0.01 $\pm$ 0.00 & 0.01 $\pm$ 0.00 & 0.00 $\pm$ 0.00 \\
RRISE & trainall\_initrandom\_n500 & 500 & 0.06 $\pm$ 0.00 & 0.05 $\pm$ 0.00 & 0.04 $\pm$ 0.00 & 0.03 $\pm$ 0.00 & 0.03 $\pm$ 0.00 & 0.02 $\pm$ 0.00 & 0.01 $\pm$ 0.00 & 0.00 $\pm$ 0.00 \\
RRISE & trainall\_initrandom\_n1000 & 1000 & 0.06 $\pm$ 0.00 & 0.05 $\pm$ 0.00 & 0.04 $\pm$ 0.00 & 0.04 $\pm$ 0.00 & 0.03 $\pm$ 0.00 & 0.02 $\pm$ 0.00 & 0.01 $\pm$ 0.00 & 0.00 $\pm$ 0.00 \\
RRISE & trainall\_initrandom\_n5000 & 5000 & 0.06 $\pm$ 0.00 & 0.06 $\pm$ 0.00 & 0.05 $\pm$ 0.00 & 0.04 $\pm$ 0.00 & 0.03 $\pm$ 0.00 & 0.02 $\pm$ 0.00 & 0.01 $\pm$ 0.00 & 0.00 $\pm$ 0.00 \\
RRISE & trainall\_initrandom\_n10000 & 10000 & 0.06 $\pm$ 0.00 & 0.05 $\pm$ 0.00 & 0.04 $\pm$ 0.00 & 0.04 $\pm$ 0.00 & 0.03 $\pm$ 0.00 & 0.02 $\pm$ 0.00 & 0.01 $\pm$ 0.00 & 0.00 $\pm$ 0.00 \\
\midrule
Baseline 1 & n500 & 500 & 0.05 $\pm$ 0.00 & 0.04 $\pm$ 0.00 & 0.03 $\pm$ 0.00 & 0.03 $\pm$ 0.00 & 0.02 $\pm$ 0.00 & 0.01 $\pm$ 0.00 & 0.00 $\pm$ 0.00 & 0.00 $\pm$ 0.00 \\
Baseline 1 & n1000 & 1000 & 0.05 $\pm$ 0.00 & 0.04 $\pm$ 0.00 & 0.03 $\pm$ 0.00 & 0.03 $\pm$ 0.00 & 0.02 $\pm$ 0.00 & 0.01 $\pm$ 0.00 & 0.00 $\pm$ 0.00 & 0.00 $\pm$ 0.00 \\
Baseline 1 & n5000 & 5000 & 0.05 $\pm$ 0.00 & 0.04 $\pm$ 0.00 & 0.03 $\pm$ 0.00 & 0.03 $\pm$ 0.00 & 0.02 $\pm$ 0.00 & 0.01 $\pm$ 0.00 & 0.00 $\pm$ 0.00 & 0.00 $\pm$ 0.00 \\
Baseline 1 & n10000 & 10000 & 0.05 $\pm$ 0.00 & 0.04 $\pm$ 0.00 & 0.03 $\pm$ 0.00 & 0.03 $\pm$ 0.00 & 0.02 $\pm$ 0.00 & 0.01 $\pm$ 0.00 & 0.00 $\pm$ 0.00 & 0.00 $\pm$ 0.00 \\
\midrule
Baseline 2 & k500\_decline0.01 & 500 & 0.05 $\pm$ 0.00 & 0.04 $\pm$ 0.00 & 0.03 $\pm$ 0.00 & 0.02 $\pm$ 0.00 & 0.01 $\pm$ 0.00 & 0.01 $\pm$ 0.00 & 0.00 $\pm$ 0.00 & 0.00 $\pm$ 0.00 \\
Baseline 2 & k1000\_decline0.01 & 1000 & 0.04 $\pm$ 0.00 & 0.04 $\pm$ 0.00 & 0.03 $\pm$ 0.00 & 0.02 $\pm$ 0.00 & 0.01 $\pm$ 0.00 & 0.01 $\pm$ 0.00 & 0.00 $\pm$ 0.00 & 0.00 $\pm$ 0.00 \\
Baseline 2 & k5000\_decline0.01 & 5000 & 0.04 $\pm$ 0.00 & 0.03 $\pm$ 0.00 & 0.03 $\pm$ 0.00 & 0.02 $\pm$ 0.00 & 0.01 $\pm$ 0.00 & 0.01 $\pm$ 0.00 & 0.00 $\pm$ 0.00 & 0.00 $\pm$ 0.00 \\
Baseline 2 & k10000\_decline0.01 & 10000 & 0.04 $\pm$ 0.00 & 0.03 $\pm$ 0.00 & 0.03 $\pm$ 0.00 & 0.02 $\pm$ 0.00 & 0.01 $\pm$ 0.00 & 0.01 $\pm$ 0.00 & 0.00 $\pm$ 0.00 & 0.00 $\pm$ 0.00 \\
Baseline 2 & k500\_decline0.05 & 500 & 0.05 $\pm$ 0.00 & 0.04 $\pm$ 0.00 & 0.03 $\pm$ 0.00 & 0.02 $\pm$ 0.00 & 0.02 $\pm$ 0.00 & 0.01 $\pm$ 0.00 & 0.00 $\pm$ 0.00 & 0.00 $\pm$ 0.00 \\
Baseline 2 & k1000\_decline0.05 & 1000 & 0.04 $\pm$ 0.00 & 0.04 $\pm$ 0.00 & 0.03 $\pm$ 0.00 & 0.02 $\pm$ 0.00 & 0.01 $\pm$ 0.00 & 0.01 $\pm$ 0.00 & 0.00 $\pm$ 0.00 & 0.00 $\pm$ 0.00 \\
Baseline 2 & k5000\_decline0.05 & 5000 & 0.04 $\pm$ 0.00 & 0.04 $\pm$ 0.00 & 0.03 $\pm$ 0.00 & 0.02 $\pm$ 0.00 & 0.01 $\pm$ 0.00 & 0.01 $\pm$ 0.00 & 0.00 $\pm$ 0.00 & 0.00 $\pm$ 0.00 \\
Baseline 2 & k10000\_decline0.05 & 10000 & 0.04 $\pm$ 0.00 & 0.03 $\pm$ 0.00 & 0.03 $\pm$ 0.00 & 0.02 $\pm$ 0.00 & 0.01 $\pm$ 0.00 & 0.01 $\pm$ 0.00 & 0.00 $\pm$ 0.00 & 0.00 $\pm$ 0.00 \\
\midrule
Baseline 3 & n500\_tol0.01 & 500 & 0.05 $\pm$ 0.00 & 0.04 $\pm$ 0.00 & 0.03 $\pm$ 0.00 & 0.03 $\pm$ 0.00 & 0.02 $\pm$ 0.00 & 0.01 $\pm$ 0.00 & 0.00 $\pm$ 0.00 & 0.00 $\pm$ 0.00 \\
Baseline 3 & n1000\_tol0.01 & 1000 & 0.05 $\pm$ 0.00 & 0.04 $\pm$ 0.00 & 0.03 $\pm$ 0.00 & 0.03 $\pm$ 0.00 & 0.02 $\pm$ 0.00 & 0.01 $\pm$ 0.00 & 0.00 $\pm$ 0.00 & 0.00 $\pm$ 0.00 \\
Baseline 3 & n5000\_tol0.01 & 5000 & 0.05 $\pm$ 0.00 & 0.04 $\pm$ 0.00 & 0.03 $\pm$ 0.00 & 0.03 $\pm$ 0.00 & 0.02 $\pm$ 0.00 & 0.01 $\pm$ 0.00 & 0.00 $\pm$ 0.00 & 0.00 $\pm$ 0.00 \\
Baseline 3 & n10000\_tol0.01 & 10000 & 0.05 $\pm$ 0.00 & 0.04 $\pm$ 0.00 & 0.03 $\pm$ 0.00 & 0.03 $\pm$ 0.00 & 0.02 $\pm$ 0.00 & 0.01 $\pm$ 0.00 & 0.00 $\pm$ 0.00 & 0.00 $\pm$ 0.00 \\
Baseline 3 & n500\_tol0.05 & 500 & 0.05 $\pm$ 0.00 & 0.04 $\pm$ 0.00 & 0.03 $\pm$ 0.00 & 0.03 $\pm$ 0.00 & 0.02 $\pm$ 0.00 & 0.01 $\pm$ 0.00 & 0.00 $\pm$ 0.00 & 0.00 $\pm$ 0.00 \\
Baseline 3 & n1000\_tol0.05 & 1000 & 0.05 $\pm$ 0.00 & 0.04 $\pm$ 0.00 & 0.03 $\pm$ 0.00 & 0.03 $\pm$ 0.00 & 0.02 $\pm$ 0.00 & 0.01 $\pm$ 0.00 & 0.00 $\pm$ 0.00 & 0.00 $\pm$ 0.00 \\
Baseline 3 & n5000\_tol0.05 & 5000 & 0.05 $\pm$ 0.00 & 0.04 $\pm$ 0.00 & 0.03 $\pm$ 0.00 & 0.03 $\pm$ 0.00 & 0.02 $\pm$ 0.00 & 0.01 $\pm$ 0.00 & 0.00 $\pm$ 0.00 & 0.00 $\pm$ 0.00 \\
Baseline 3 & n10000\_tol0.05 & 10000 & 0.05 $\pm$ 0.00 & 0.04 $\pm$ 0.00 & 0.03 $\pm$ 0.00 & 0.03 $\pm$ 0.00 & 0.02 $\pm$ 0.00 & 0.01 $\pm$ 0.00 & 0.00 $\pm$ 0.00 & 0.00 $\pm$ 0.00 \\
\midrule
Baseline 4 & init0\_n500 & 500 & 0.06 $\pm$ 0.01 & 0.05 $\pm$ 0.01 & 0.04 $\pm$ 0.01 & 0.03 $\pm$ 0.01 & 0.03 $\pm$ 0.00 & 0.02 $\pm$ 0.00 & 0.01 $\pm$ 0.00 & 0.00 $\pm$ 0.00 \\
Baseline 4 & init0\_n1000 & 1000 & 0.07 $\pm$ 0.00 & 0.06 $\pm$ 0.00 & 0.05 $\pm$ 0.00 & 0.04 $\pm$ 0.00 & 0.03 $\pm$ 0.00 & 0.02 $\pm$ 0.00 & 0.01 $\pm$ 0.00 & 0.00 $\pm$ 0.00 \\
Baseline 4 & init0\_n5000 & 5000 & 0.06 $\pm$ 0.00 & 0.05 $\pm$ 0.00 & 0.04 $\pm$ 0.00 & 0.03 $\pm$ 0.00 & 0.02 $\pm$ 0.00 & 0.02 $\pm$ 0.00 & 0.01 $\pm$ 0.00 & 0.00 $\pm$ 0.00 \\
Baseline 4 & init0\_n10000 & 10000 & 0.06 $\pm$ 0.01 & 0.05 $\pm$ 0.01 & 0.04 $\pm$ 0.01 & 0.03 $\pm$ 0.01 & 0.02 $\pm$ 0.01 & 0.01 $\pm$ 0.00 & 0.01 $\pm$ 0.00 & 0.00 $\pm$ 0.00 \\
Baseline 4 & init1\_n500 & 500 & 0.05 $\pm$ 0.00 & 0.04 $\pm$ 0.00 & 0.03 $\pm$ 0.00 & 0.03 $\pm$ 0.00 & 0.02 $\pm$ 0.00 & 0.01 $\pm$ 0.00 & 0.01 $\pm$ 0.00 & 0.00 $\pm$ 0.00 \\
Baseline 4 & init1\_n1000 & 1000 & 0.05 $\pm$ 0.01 & 0.04 $\pm$ 0.01 & 0.03 $\pm$ 0.00 & 0.03 $\pm$ 0.00 & 0.02 $\pm$ 0.00 & 0.01 $\pm$ 0.00 & 0.00 $\pm$ 0.00 & 0.00 $\pm$ 0.00 \\
Baseline 4 & init1\_n5000 & 5000 & 0.05 $\pm$ 0.00 & 0.04 $\pm$ 0.00 & 0.03 $\pm$ 0.00 & 0.03 $\pm$ 0.00 & 0.02 $\pm$ 0.00 & 0.01 $\pm$ 0.00 & 0.01 $\pm$ 0.00 & 0.00 $\pm$ 0.00 \\
Baseline 4 & init1\_n10000 & 10000 & 0.05 $\pm$ 0.00 & 0.04 $\pm$ 0.00 & 0.03 $\pm$ 0.00 & 0.03 $\pm$ 0.00 & 0.02 $\pm$ 0.00 & 0.01 $\pm$ 0.00 & 0.00 $\pm$ 0.00 & 0.00 $\pm$ 0.00 \\
\bottomrule
\end{tabular}%
}
\end{table*}

\begin{table*}[t]
\centering
\scriptsize
\caption{Full boundary-confidence CRD results on CIFAR-10. Entries are percentages of all test inputs satisfying both the boundary-confidence condition and the radius threshold.}
\label{tab:app_crd_boundary_full_cifar10}
\resizebox{\textwidth}{!}{%
\begin{tabular}{lllcccccccc}
\toprule
Method & Variant & Budget & $t=0$ & $t=0.025$ & $t=0.05$ & $t=0.075$ & $t=0.1$ & $t=0.125$ & $t=0.15$ & $t=0.175$ \\
\midrule
RRISE & trainhead\_initbase\_n500 & 500 & 0.17 $\pm$ 0.00 & 0.14 $\pm$ 0.00 & 0.12 $\pm$ 0.00 & 0.09 $\pm$ 0.00 & 0.07 $\pm$ 0.00 & 0.04 $\pm$ 0.00 & 0.02 $\pm$ 0.00 & 0.00 $\pm$ 0.00 \\
RRISE & trainhead\_initbase\_n1000 & 1000 & 0.17 $\pm$ 0.00 & 0.14 $\pm$ 0.00 & 0.12 $\pm$ 0.00 & 0.09 $\pm$ 0.00 & 0.07 $\pm$ 0.00 & 0.04 $\pm$ 0.00 & 0.02 $\pm$ 0.00 & 0.00 $\pm$ 0.00 \\
RRISE & trainhead\_initbase\_n5000 & 5000 & 0.16 $\pm$ 0.00 & 0.14 $\pm$ 0.00 & 0.11 $\pm$ 0.00 & 0.09 $\pm$ 0.00 & 0.06 $\pm$ 0.00 & 0.04 $\pm$ 0.00 & 0.02 $\pm$ 0.00 & 0.00 $\pm$ 0.00 \\
RRISE & trainhead\_initbase\_n10000 & 10000 & 0.16 $\pm$ 0.00 & 0.14 $\pm$ 0.00 & 0.11 $\pm$ 0.00 & 0.09 $\pm$ 0.00 & 0.06 $\pm$ 0.00 & 0.04 $\pm$ 0.00 & 0.02 $\pm$ 0.00 & 0.00 $\pm$ 0.00 \\
RRISE & trainall\_initbase\_n500 & 500 & 0.17 $\pm$ 0.00 & 0.14 $\pm$ 0.00 & 0.12 $\pm$ 0.00 & 0.09 $\pm$ 0.00 & 0.07 $\pm$ 0.00 & 0.04 $\pm$ 0.00 & 0.02 $\pm$ 0.00 & 0.00 $\pm$ 0.00 \\
RRISE & trainall\_initbase\_n1000 & 1000 & 0.16 $\pm$ 0.00 & 0.14 $\pm$ 0.00 & 0.11 $\pm$ 0.00 & 0.09 $\pm$ 0.00 & 0.07 $\pm$ 0.00 & 0.04 $\pm$ 0.00 & 0.02 $\pm$ 0.00 & 0.00 $\pm$ 0.00 \\
RRISE & trainall\_initbase\_n5000 & 5000 & 0.16 $\pm$ 0.00 & 0.14 $\pm$ 0.00 & 0.11 $\pm$ 0.00 & 0.09 $\pm$ 0.00 & 0.06 $\pm$ 0.00 & 0.04 $\pm$ 0.00 & 0.02 $\pm$ 0.00 & 0.00 $\pm$ 0.00 \\
RRISE & trainall\_initbase\_n10000 & 10000 & 0.17 $\pm$ 0.00 & 0.14 $\pm$ 0.00 & 0.11 $\pm$ 0.00 & 0.09 $\pm$ 0.00 & 0.06 $\pm$ 0.00 & 0.04 $\pm$ 0.00 & 0.02 $\pm$ 0.00 & 0.00 $\pm$ 0.00 \\
RRISE & trainall\_initrandom\_n500 & 500 & 0.59 $\pm$ 0.19 & 0.56 $\pm$ 0.19 & 0.51 $\pm$ 0.19 & 0.47 $\pm$ 0.19 & 0.42 $\pm$ 0.19 & 0.37 $\pm$ 0.18 & 0.28 $\pm$ 0.17 & 0.00 $\pm$ 0.00 \\
RRISE & trainall\_initrandom\_n1000 & 1000 & 0.33 $\pm$ 0.03 & 0.29 $\pm$ 0.03 & 0.25 $\pm$ 0.03 & 0.21 $\pm$ 0.03 & 0.17 $\pm$ 0.03 & 0.12 $\pm$ 0.02 & 0.06 $\pm$ 0.01 & 0.00 $\pm$ 0.00 \\
RRISE & trainall\_initrandom\_n5000 & 5000 & 0.29 $\pm$ 0.02 & 0.25 $\pm$ 0.02 & 0.22 $\pm$ 0.02 & 0.18 $\pm$ 0.01 & 0.14 $\pm$ 0.01 & 0.09 $\pm$ 0.01 & 0.04 $\pm$ 0.00 & 0.00 $\pm$ 0.00 \\
RRISE & trainall\_initrandom\_n10000 & 10000 & 0.30 $\pm$ 0.06 & 0.27 $\pm$ 0.06 & 0.23 $\pm$ 0.05 & 0.19 $\pm$ 0.05 & 0.15 $\pm$ 0.05 & 0.10 $\pm$ 0.04 & 0.05 $\pm$ 0.03 & 0.00 $\pm$ 0.00 \\
\midrule
Baseline 1 & n500 & 500 & 0.17 $\pm$ 0.00 & 0.14 $\pm$ 0.00 & 0.11 $\pm$ 0.00 & 0.09 $\pm$ 0.00 & 0.07 $\pm$ 0.00 & 0.04 $\pm$ 0.00 & 0.02 $\pm$ 0.00 & 0.00 $\pm$ 0.00 \\
Baseline 1 & n1000 & 1000 & 0.17 $\pm$ 0.00 & 0.14 $\pm$ 0.00 & 0.11 $\pm$ 0.00 & 0.09 $\pm$ 0.00 & 0.06 $\pm$ 0.00 & 0.04 $\pm$ 0.00 & 0.02 $\pm$ 0.00 & 0.00 $\pm$ 0.00 \\
Baseline 1 & n5000 & 5000 & 0.17 $\pm$ 0.00 & 0.14 $\pm$ 0.00 & 0.11 $\pm$ 0.00 & 0.09 $\pm$ 0.00 & 0.07 $\pm$ 0.00 & 0.04 $\pm$ 0.00 & 0.02 $\pm$ 0.00 & 0.00 $\pm$ 0.00 \\
Baseline 1 & n10000 & 10000 & 0.17 $\pm$ 0.00 & 0.14 $\pm$ 0.00 & 0.11 $\pm$ 0.00 & 0.09 $\pm$ 0.00 & 0.07 $\pm$ 0.00 & 0.04 $\pm$ 0.00 & 0.02 $\pm$ 0.00 & 0.00 $\pm$ 0.00 \\
\midrule
Baseline 2 & k500\_decline0.01 & 500 & 0.16 $\pm$ 0.00 & 0.13 $\pm$ 0.00 & 0.11 $\pm$ 0.00 & 0.08 $\pm$ 0.00 & 0.05 $\pm$ 0.00 & 0.03 $\pm$ 0.00 & 0.01 $\pm$ 0.00 & 0.00 $\pm$ 0.00 \\
Baseline 2 & k1000\_decline0.01 & 1000 & 0.16 $\pm$ 0.00 & 0.13 $\pm$ 0.00 & 0.10 $\pm$ 0.00 & 0.08 $\pm$ 0.00 & 0.05 $\pm$ 0.00 & 0.03 $\pm$ 0.00 & 0.01 $\pm$ 0.00 & 0.00 $\pm$ 0.00 \\
Baseline 2 & k5000\_decline0.01 & 5000 & 0.15 $\pm$ 0.00 & 0.12 $\pm$ 0.00 & 0.10 $\pm$ 0.00 & 0.07 $\pm$ 0.00 & 0.04 $\pm$ 0.00 & 0.03 $\pm$ 0.00 & 0.02 $\pm$ 0.00 & 0.00 $\pm$ 0.00 \\
Baseline 2 & k10000\_decline0.01 & 10000 & 0.15 $\pm$ 0.00 & 0.12 $\pm$ 0.00 & 0.10 $\pm$ 0.00 & 0.07 $\pm$ 0.00 & 0.04 $\pm$ 0.00 & 0.03 $\pm$ 0.00 & 0.01 $\pm$ 0.00 & 0.00 $\pm$ 0.00 \\
Baseline 2 & k500\_decline0.05 & 500 & 0.16 $\pm$ 0.00 & 0.13 $\pm$ 0.00 & 0.11 $\pm$ 0.00 & 0.08 $\pm$ 0.00 & 0.05 $\pm$ 0.00 & 0.03 $\pm$ 0.00 & 0.01 $\pm$ 0.00 & 0.00 $\pm$ 0.00 \\
Baseline 2 & k1000\_decline0.05 & 1000 & 0.16 $\pm$ 0.00 & 0.13 $\pm$ 0.00 & 0.10 $\pm$ 0.00 & 0.08 $\pm$ 0.00 & 0.05 $\pm$ 0.00 & 0.03 $\pm$ 0.00 & 0.01 $\pm$ 0.00 & 0.00 $\pm$ 0.00 \\
Baseline 2 & k5000\_decline0.05 & 5000 & 0.15 $\pm$ 0.00 & 0.12 $\pm$ 0.00 & 0.10 $\pm$ 0.00 & 0.07 $\pm$ 0.00 & 0.04 $\pm$ 0.00 & 0.03 $\pm$ 0.00 & 0.02 $\pm$ 0.00 & 0.00 $\pm$ 0.00 \\
Baseline 2 & k10000\_decline0.05 & 10000 & 0.15 $\pm$ 0.00 & 0.12 $\pm$ 0.00 & 0.10 $\pm$ 0.00 & 0.07 $\pm$ 0.00 & 0.04 $\pm$ 0.00 & 0.03 $\pm$ 0.00 & 0.02 $\pm$ 0.00 & 0.00 $\pm$ 0.00 \\
\midrule
Baseline 3 & n500\_tol0.01 & 500 & 0.17 $\pm$ 0.00 & 0.14 $\pm$ 0.00 & 0.11 $\pm$ 0.00 & 0.09 $\pm$ 0.00 & 0.06 $\pm$ 0.00 & 0.04 $\pm$ 0.00 & 0.02 $\pm$ 0.00 & 0.00 $\pm$ 0.00 \\
Baseline 3 & n1000\_tol0.01 & 1000 & 0.17 $\pm$ 0.00 & 0.14 $\pm$ 0.00 & 0.11 $\pm$ 0.00 & 0.09 $\pm$ 0.00 & 0.06 $\pm$ 0.00 & 0.04 $\pm$ 0.00 & 0.02 $\pm$ 0.00 & 0.00 $\pm$ 0.00 \\
Baseline 3 & n5000\_tol0.01 & 5000 & 0.17 $\pm$ 0.00 & 0.14 $\pm$ 0.00 & 0.11 $\pm$ 0.00 & 0.09 $\pm$ 0.00 & 0.07 $\pm$ 0.00 & 0.04 $\pm$ 0.00 & 0.02 $\pm$ 0.00 & 0.00 $\pm$ 0.00 \\
Baseline 3 & n10000\_tol0.01 & 10000 & 0.17 $\pm$ 0.00 & 0.14 $\pm$ 0.00 & 0.11 $\pm$ 0.00 & 0.09 $\pm$ 0.00 & 0.07 $\pm$ 0.00 & 0.04 $\pm$ 0.00 & 0.02 $\pm$ 0.00 & 0.00 $\pm$ 0.00 \\
Baseline 3 & n500\_tol0.05 & 500 & 0.17 $\pm$ 0.00 & 0.14 $\pm$ 0.00 & 0.11 $\pm$ 0.00 & 0.09 $\pm$ 0.00 & 0.07 $\pm$ 0.00 & 0.04 $\pm$ 0.00 & 0.02 $\pm$ 0.00 & 0.00 $\pm$ 0.00 \\
Baseline 3 & n1000\_tol0.05 & 1000 & 0.17 $\pm$ 0.00 & 0.14 $\pm$ 0.00 & 0.12 $\pm$ 0.00 & 0.09 $\pm$ 0.00 & 0.07 $\pm$ 0.00 & 0.04 $\pm$ 0.00 & 0.02 $\pm$ 0.00 & 0.00 $\pm$ 0.00 \\
Baseline 3 & n5000\_tol0.05 & 5000 & 0.17 $\pm$ 0.00 & 0.14 $\pm$ 0.00 & 0.12 $\pm$ 0.00 & 0.09 $\pm$ 0.00 & 0.07 $\pm$ 0.00 & 0.04 $\pm$ 0.00 & 0.02 $\pm$ 0.00 & 0.00 $\pm$ 0.00 \\
Baseline 3 & n10000\_tol0.05 & 10000 & 0.17 $\pm$ 0.00 & 0.14 $\pm$ 0.00 & 0.12 $\pm$ 0.00 & 0.09 $\pm$ 0.00 & 0.07 $\pm$ 0.00 & 0.04 $\pm$ 0.00 & 0.02 $\pm$ 0.00 & 0.00 $\pm$ 0.00 \\
\midrule
Baseline 4 & init0\_n500 & 500 & 0.62 $\pm$ 0.00 & 0.58 $\pm$ 0.01 & 0.52 $\pm$ 0.01 & 0.44 $\pm$ 0.01 & 0.11 $\pm$ 0.18 & 0.00 $\pm$ 0.00 & 0.00 $\pm$ 0.00 & 0.00 $\pm$ 0.00 \\
Baseline 4 & init0\_n1000 & 1000 & 0.65 $\pm$ 0.02 & 0.61 $\pm$ 0.03 & 0.56 $\pm$ 0.03 & 0.49 $\pm$ 0.05 & 0.30 $\pm$ 0.26 & 0.22 $\pm$ 0.19 & 0.00 $\pm$ 0.00 & 0.00 $\pm$ 0.00 \\
Baseline 4 & init0\_n5000 & 5000 & 0.59 $\pm$ 0.16 & 0.54 $\pm$ 0.16 & 0.50 $\pm$ 0.15 & 0.45 $\pm$ 0.15 & 0.39 $\pm$ 0.14 & 0.19 $\pm$ 0.20 & 0.03 $\pm$ 0.05 & 0.00 $\pm$ 0.00 \\
Baseline 4 & init0\_n10000 & 10000 & 0.68 $\pm$ 0.02 & 0.63 $\pm$ 0.03 & 0.59 $\pm$ 0.03 & 0.53 $\pm$ 0.03 & 0.47 $\pm$ 0.04 & 0.32 $\pm$ 0.18 & 0.16 $\pm$ 0.14 & 0.00 $\pm$ 0.00 \\
Baseline 4 & init1\_n500 & 500 & 0.18 $\pm$ 0.01 & 0.15 $\pm$ 0.01 & 0.13 $\pm$ 0.01 & 0.10 $\pm$ 0.01 & 0.07 $\pm$ 0.00 & 0.05 $\pm$ 0.00 & 0.02 $\pm$ 0.00 & 0.00 $\pm$ 0.00 \\
Baseline 4 & init1\_n1000 & 1000 & 0.17 $\pm$ 0.01 & 0.15 $\pm$ 0.00 & 0.12 $\pm$ 0.00 & 0.10 $\pm$ 0.00 & 0.07 $\pm$ 0.00 & 0.04 $\pm$ 0.00 & 0.02 $\pm$ 0.00 & 0.00 $\pm$ 0.00 \\
Baseline 4 & init1\_n5000 & 5000 & 0.17 $\pm$ 0.01 & 0.14 $\pm$ 0.01 & 0.12 $\pm$ 0.00 & 0.09 $\pm$ 0.00 & 0.07 $\pm$ 0.00 & 0.04 $\pm$ 0.00 & 0.02 $\pm$ 0.00 & 0.00 $\pm$ 0.00 \\
Baseline 4 & init1\_n10000 & 10000 & 0.17 $\pm$ 0.00 & 0.14 $\pm$ 0.00 & 0.12 $\pm$ 0.00 & 0.09 $\pm$ 0.00 & 0.07 $\pm$ 0.00 & 0.04 $\pm$ 0.00 & 0.02 $\pm$ 0.00 & 0.00 $\pm$ 0.00 \\
\bottomrule
\end{tabular}%
}
\end{table*}

\begin{table*}[t]
\centering
\scriptsize
\caption{Full boundary-confidence CRD results on CIFAR-100. Entries are percentages of all test inputs satisfying both the boundary-confidence condition and the radius threshold.}
\label{tab:app_crd_boundary_full_cifar100}
\resizebox{\textwidth}{!}{%
\begin{tabular}{lllcccccccc}
\toprule
Method & Variant & Budget & $t=0$ & $t=0.0125$ & $t=0.025$ & $t=0.0375$ & $t=0.05$ & $t=0.0625$ & $t=0.075$ & $t=0.0875$ \\
\midrule
RRISE & trainhead\_initbase\_n500 & 500 & 0.28 $\pm$ 0.00 & 0.24 $\pm$ 0.00 & 0.19 $\pm$ 0.00 & 0.14 $\pm$ 0.00 & 0.09 $\pm$ 0.00 & 0.03 $\pm$ 0.00 & 0.00 $\pm$ 0.00 & 0.00 $\pm$ 0.00 \\
RRISE & trainhead\_initbase\_n1000 & 1000 & 0.28 $\pm$ 0.01 & 0.24 $\pm$ 0.02 & 0.19 $\pm$ 0.01 & 0.14 $\pm$ 0.01 & 0.10 $\pm$ 0.01 & 0.03 $\pm$ 0.00 & 0.00 $\pm$ 0.00 & 0.00 $\pm$ 0.00 \\
RRISE & trainhead\_initbase\_n5000 & 5000 & 0.25 $\pm$ 0.01 & 0.21 $\pm$ 0.01 & 0.17 $\pm$ 0.01 & 0.12 $\pm$ 0.01 & 0.08 $\pm$ 0.01 & 0.02 $\pm$ 0.00 & 0.00 $\pm$ 0.00 & 0.00 $\pm$ 0.00 \\
RRISE & trainhead\_initbase\_n10000 & 10000 & 0.24 $\pm$ 0.01 & 0.20 $\pm$ 0.01 & 0.16 $\pm$ 0.01 & 0.11 $\pm$ 0.00 & 0.07 $\pm$ 0.00 & 0.02 $\pm$ 0.00 & 0.00 $\pm$ 0.00 & 0.00 $\pm$ 0.00 \\
RRISE & trainall\_initbase\_n500 & 500 & 0.22 $\pm$ 0.00 & 0.18 $\pm$ 0.00 & 0.14 $\pm$ 0.00 & 0.09 $\pm$ 0.00 & 0.05 $\pm$ 0.00 & 0.02 $\pm$ 0.00 & 0.00 $\pm$ 0.00 & 0.00 $\pm$ 0.00 \\
RRISE & trainall\_initbase\_n1000 & 1000 & 0.22 $\pm$ 0.00 & 0.17 $\pm$ 0.00 & 0.13 $\pm$ 0.00 & 0.09 $\pm$ 0.00 & 0.05 $\pm$ 0.00 & 0.01 $\pm$ 0.00 & 0.00 $\pm$ 0.00 & 0.00 $\pm$ 0.00 \\
RRISE & trainall\_initbase\_n5000 & 5000 & 0.22 $\pm$ 0.00 & 0.18 $\pm$ 0.00 & 0.13 $\pm$ 0.00 & 0.09 $\pm$ 0.00 & 0.05 $\pm$ 0.00 & 0.01 $\pm$ 0.00 & 0.00 $\pm$ 0.00 & 0.00 $\pm$ 0.00 \\
RRISE & trainall\_initbase\_n10000 & 10000 & 0.23 $\pm$ 0.00 & 0.18 $\pm$ 0.00 & 0.13 $\pm$ 0.00 & 0.09 $\pm$ 0.00 & 0.05 $\pm$ 0.00 & 0.01 $\pm$ 0.00 & 0.00 $\pm$ 0.00 & 0.00 $\pm$ 0.00 \\
RRISE & trainall\_initrandom\_n500 & 500 & 0.00 $\pm$ 0.00 & 0.00 $\pm$ 0.00 & 0.00 $\pm$ 0.00 & 0.00 $\pm$ 0.00 & 0.00 $\pm$ 0.00 & 0.00 $\pm$ 0.00 & 0.00 $\pm$ 0.00 & 0.00 $\pm$ 0.00 \\
RRISE & trainall\_initrandom\_n1000 & 1000 & 0.00 $\pm$ 0.00 & 0.00 $\pm$ 0.00 & 0.00 $\pm$ 0.00 & 0.00 $\pm$ 0.00 & 0.00 $\pm$ 0.00 & 0.00 $\pm$ 0.00 & 0.00 $\pm$ 0.00 & 0.00 $\pm$ 0.00 \\
RRISE & trainall\_initrandom\_n5000 & 5000 & 0.00 $\pm$ 0.00 & 0.00 $\pm$ 0.00 & 0.00 $\pm$ 0.00 & 0.00 $\pm$ 0.00 & 0.00 $\pm$ 0.00 & 0.00 $\pm$ 0.00 & 0.00 $\pm$ 0.00 & 0.00 $\pm$ 0.00 \\
RRISE & trainall\_initrandom\_n10000 & 10000 & 0.00 $\pm$ 0.00 & 0.00 $\pm$ 0.00 & 0.00 $\pm$ 0.00 & 0.00 $\pm$ 0.00 & 0.00 $\pm$ 0.00 & 0.00 $\pm$ 0.00 & 0.00 $\pm$ 0.00 & 0.00 $\pm$ 0.00 \\
\midrule
Baseline 1 & n500 & 500 & 0.20 $\pm$ 0.00 & 0.16 $\pm$ 0.00 & 0.12 $\pm$ 0.00 & 0.08 $\pm$ 0.00 & 0.05 $\pm$ 0.00 & 0.01 $\pm$ 0.00 & 0.00 $\pm$ 0.00 & 0.00 $\pm$ 0.00 \\
Baseline 1 & n1000 & 1000 & 0.20 $\pm$ 0.00 & 0.16 $\pm$ 0.00 & 0.12 $\pm$ 0.00 & 0.08 $\pm$ 0.00 & 0.05 $\pm$ 0.00 & 0.01 $\pm$ 0.00 & 0.00 $\pm$ 0.00 & 0.00 $\pm$ 0.00 \\
Baseline 1 & n5000 & 5000 & 0.20 $\pm$ 0.00 & 0.16 $\pm$ 0.00 & 0.12 $\pm$ 0.00 & 0.08 $\pm$ 0.00 & 0.05 $\pm$ 0.00 & 0.01 $\pm$ 0.00 & 0.00 $\pm$ 0.00 & 0.00 $\pm$ 0.00 \\
Baseline 1 & n10000 & 10000 & 0.20 $\pm$ 0.00 & 0.16 $\pm$ 0.00 & 0.12 $\pm$ 0.00 & 0.08 $\pm$ 0.00 & 0.05 $\pm$ 0.00 & 0.01 $\pm$ 0.00 & 0.00 $\pm$ 0.00 & 0.00 $\pm$ 0.00 \\
\midrule
Baseline 2 & k500\_decline0.01 & 500 & 0.19 $\pm$ 0.00 & 0.15 $\pm$ 0.00 & 0.11 $\pm$ 0.00 & 0.07 $\pm$ 0.00 & 0.03 $\pm$ 0.00 & 0.01 $\pm$ 0.00 & 0.00 $\pm$ 0.00 & 0.00 $\pm$ 0.00 \\
Baseline 2 & k1000\_decline0.01 & 1000 & 0.18 $\pm$ 0.00 & 0.14 $\pm$ 0.00 & 0.10 $\pm$ 0.00 & 0.06 $\pm$ 0.00 & 0.03 $\pm$ 0.00 & 0.01 $\pm$ 0.00 & 0.00 $\pm$ 0.00 & 0.00 $\pm$ 0.00 \\
Baseline 2 & k5000\_decline0.01 & 5000 & 0.18 $\pm$ 0.00 & 0.14 $\pm$ 0.00 & 0.10 $\pm$ 0.00 & 0.05 $\pm$ 0.00 & 0.04 $\pm$ 0.00 & 0.01 $\pm$ 0.00 & 0.00 $\pm$ 0.00 & 0.00 $\pm$ 0.00 \\
Baseline 2 & k10000\_decline0.01 & 10000 & 0.18 $\pm$ 0.00 & 0.14 $\pm$ 0.00 & 0.10 $\pm$ 0.00 & 0.05 $\pm$ 0.00 & 0.04 $\pm$ 0.00 & 0.01 $\pm$ 0.00 & 0.00 $\pm$ 0.00 & 0.00 $\pm$ 0.00 \\
Baseline 2 & k500\_decline0.05 & 500 & 0.19 $\pm$ 0.00 & 0.15 $\pm$ 0.00 & 0.11 $\pm$ 0.00 & 0.07 $\pm$ 0.00 & 0.03 $\pm$ 0.00 & 0.01 $\pm$ 0.00 & 0.00 $\pm$ 0.00 & 0.00 $\pm$ 0.00 \\
Baseline 2 & k1000\_decline0.05 & 1000 & 0.18 $\pm$ 0.00 & 0.14 $\pm$ 0.00 & 0.10 $\pm$ 0.00 & 0.06 $\pm$ 0.00 & 0.03 $\pm$ 0.00 & 0.01 $\pm$ 0.00 & 0.00 $\pm$ 0.00 & 0.00 $\pm$ 0.00 \\
Baseline 2 & k5000\_decline0.05 & 5000 & 0.17 $\pm$ 0.00 & 0.14 $\pm$ 0.00 & 0.10 $\pm$ 0.00 & 0.05 $\pm$ 0.00 & 0.04 $\pm$ 0.00 & 0.01 $\pm$ 0.00 & 0.00 $\pm$ 0.00 & 0.00 $\pm$ 0.00 \\
Baseline 2 & k10000\_decline0.05 & 10000 & 0.18 $\pm$ 0.00 & 0.14 $\pm$ 0.00 & 0.10 $\pm$ 0.00 & 0.05 $\pm$ 0.00 & 0.04 $\pm$ 0.00 & 0.01 $\pm$ 0.00 & 0.00 $\pm$ 0.00 & 0.00 $\pm$ 0.00 \\
\midrule
Baseline 3 & n500\_tol0.01 & 500 & 0.20 $\pm$ 0.00 & 0.16 $\pm$ 0.00 & 0.12 $\pm$ 0.00 & 0.08 $\pm$ 0.00 & 0.05 $\pm$ 0.00 & 0.01 $\pm$ 0.00 & 0.00 $\pm$ 0.00 & 0.00 $\pm$ 0.00 \\
Baseline 3 & n1000\_tol0.01 & 1000 & 0.20 $\pm$ 0.00 & 0.16 $\pm$ 0.00 & 0.12 $\pm$ 0.00 & 0.08 $\pm$ 0.00 & 0.05 $\pm$ 0.00 & 0.01 $\pm$ 0.00 & 0.00 $\pm$ 0.00 & 0.00 $\pm$ 0.00 \\
Baseline 3 & n5000\_tol0.01 & 5000 & 0.20 $\pm$ 0.00 & 0.16 $\pm$ 0.00 & 0.12 $\pm$ 0.00 & 0.08 $\pm$ 0.00 & 0.05 $\pm$ 0.00 & 0.01 $\pm$ 0.00 & 0.00 $\pm$ 0.00 & 0.00 $\pm$ 0.00 \\
Baseline 3 & n10000\_tol0.01 & 10000 & 0.20 $\pm$ 0.00 & 0.16 $\pm$ 0.00 & 0.12 $\pm$ 0.00 & 0.08 $\pm$ 0.00 & 0.05 $\pm$ 0.00 & 0.01 $\pm$ 0.00 & 0.00 $\pm$ 0.00 & 0.00 $\pm$ 0.00 \\
Baseline 3 & n500\_tol0.05 & 500 & 0.20 $\pm$ 0.00 & 0.16 $\pm$ 0.00 & 0.12 $\pm$ 0.00 & 0.08 $\pm$ 0.00 & 0.05 $\pm$ 0.00 & 0.01 $\pm$ 0.00 & 0.00 $\pm$ 0.00 & 0.00 $\pm$ 0.00 \\
Baseline 3 & n1000\_tol0.05 & 1000 & 0.20 $\pm$ 0.00 & 0.16 $\pm$ 0.00 & 0.12 $\pm$ 0.00 & 0.09 $\pm$ 0.00 & 0.05 $\pm$ 0.00 & 0.01 $\pm$ 0.00 & 0.00 $\pm$ 0.00 & 0.00 $\pm$ 0.00 \\
Baseline 3 & n5000\_tol0.05 & 5000 & 0.20 $\pm$ 0.00 & 0.16 $\pm$ 0.00 & 0.12 $\pm$ 0.00 & 0.08 $\pm$ 0.00 & 0.05 $\pm$ 0.00 & 0.01 $\pm$ 0.00 & 0.00 $\pm$ 0.00 & 0.00 $\pm$ 0.00 \\
Baseline 3 & n10000\_tol0.05 & 10000 & 0.20 $\pm$ 0.00 & 0.16 $\pm$ 0.00 & 0.12 $\pm$ 0.00 & 0.08 $\pm$ 0.00 & 0.05 $\pm$ 0.00 & 0.01 $\pm$ 0.00 & 0.00 $\pm$ 0.00 & 0.00 $\pm$ 0.00 \\
\midrule
Baseline 4 & init0\_n500 & 500 & 0.00 $\pm$ 0.00 & 0.00 $\pm$ 0.00 & 0.00 $\pm$ 0.00 & 0.00 $\pm$ 0.00 & 0.00 $\pm$ 0.00 & 0.00 $\pm$ 0.00 & 0.00 $\pm$ 0.00 & 0.00 $\pm$ 0.00 \\
Baseline 4 & init0\_n1000 & 1000 & 0.00 $\pm$ 0.00 & 0.00 $\pm$ 0.00 & 0.00 $\pm$ 0.00 & 0.00 $\pm$ 0.00 & 0.00 $\pm$ 0.00 & 0.00 $\pm$ 0.00 & 0.00 $\pm$ 0.00 & 0.00 $\pm$ 0.00 \\
Baseline 4 & init0\_n5000 & 5000 & 0.00 $\pm$ 0.00 & 0.00 $\pm$ 0.00 & 0.00 $\pm$ 0.00 & 0.00 $\pm$ 0.00 & 0.00 $\pm$ 0.00 & 0.00 $\pm$ 0.00 & 0.00 $\pm$ 0.00 & 0.00 $\pm$ 0.00 \\
Baseline 4 & init0\_n10000 & 10000 & 0.00 $\pm$ 0.00 & 0.00 $\pm$ 0.00 & 0.00 $\pm$ 0.00 & 0.00 $\pm$ 0.00 & 0.00 $\pm$ 0.00 & 0.00 $\pm$ 0.00 & 0.00 $\pm$ 0.00 & 0.00 $\pm$ 0.00 \\
Baseline 4 & init1\_n500 & 500 & 0.28 $\pm$ 0.00 & 0.17 $\pm$ 0.01 & 0.00 $\pm$ 0.00 & 0.00 $\pm$ 0.00 & 0.00 $\pm$ 0.00 & 0.00 $\pm$ 0.00 & 0.00 $\pm$ 0.00 & 0.00 $\pm$ 0.00 \\
Baseline 4 & init1\_n1000 & 1000 & 0.29 $\pm$ 0.02 & 0.20 $\pm$ 0.03 & 0.00 $\pm$ 0.00 & 0.00 $\pm$ 0.00 & 0.00 $\pm$ 0.00 & 0.00 $\pm$ 0.00 & 0.00 $\pm$ 0.00 & 0.00 $\pm$ 0.00 \\
Baseline 4 & init1\_n5000 & 5000 & 0.30 $\pm$ 0.01 & 0.21 $\pm$ 0.02 & 0.00 $\pm$ 0.00 & 0.00 $\pm$ 0.00 & 0.00 $\pm$ 0.00 & 0.00 $\pm$ 0.00 & 0.00 $\pm$ 0.00 & 0.00 $\pm$ 0.00 \\
Baseline 4 & init1\_n10000 & 10000 & 0.30 $\pm$ 0.01 & 0.22 $\pm$ 0.01 & 0.00 $\pm$ 0.00 & 0.00 $\pm$ 0.00 & 0.00 $\pm$ 0.00 & 0.00 $\pm$ 0.00 & 0.00 $\pm$ 0.00 & 0.00 $\pm$ 0.00 \\
\bottomrule
\end{tabular}%
}
\end{table*}

\begin{table*}[t]
\centering
\scriptsize
\caption{Full boundary-confidence CRD results on Tiny ImageNet. Entries are percentages of all test inputs satisfying both the boundary-confidence condition and the radius threshold.}
\label{tab:app_crd_boundary_full_tiny_imagenet}
\resizebox{\textwidth}{!}{%
\begin{tabular}{lllcccccccc}
\toprule
Method & Variant & Budget & $t=0$ & $t=0.0125$ & $t=0.025$ & $t=0.0375$ & $t=0.05$ & $t=0.0625$ & $t=0.075$ & $t=0.0875$ \\
\midrule
RRISE & trainhead\_initbase\_n500 & 500 & 0.24 $\pm$ 0.00 & 0.20 $\pm$ 0.01 & 0.16 $\pm$ 0.01 & 0.11 $\pm$ 0.01 & 0.07 $\pm$ 0.00 & 0.02 $\pm$ 0.00 & 0.00 $\pm$ 0.00 & 0.00 $\pm$ 0.00 \\
RRISE & trainhead\_initbase\_n1000 & 1000 & 0.24 $\pm$ 0.01 & 0.20 $\pm$ 0.00 & 0.15 $\pm$ 0.01 & 0.11 $\pm$ 0.00 & 0.07 $\pm$ 0.00 & 0.02 $\pm$ 0.00 & 0.00 $\pm$ 0.00 & 0.00 $\pm$ 0.00 \\
RRISE & trainhead\_initbase\_n5000 & 5000 & 0.22 $\pm$ 0.00 & 0.18 $\pm$ 0.00 & 0.14 $\pm$ 0.00 & 0.10 $\pm$ 0.00 & 0.06 $\pm$ 0.00 & 0.02 $\pm$ 0.00 & 0.00 $\pm$ 0.00 & 0.00 $\pm$ 0.00 \\
RRISE & trainhead\_initbase\_n10000 & 10000 & 0.23 $\pm$ 0.01 & 0.19 $\pm$ 0.01 & 0.14 $\pm$ 0.01 & 0.10 $\pm$ 0.01 & 0.06 $\pm$ 0.00 & 0.02 $\pm$ 0.00 & 0.00 $\pm$ 0.00 & 0.00 $\pm$ 0.00 \\
RRISE & trainall\_initbase\_n500 & 500 & 0.26 $\pm$ 0.00 & 0.21 $\pm$ 0.00 & 0.16 $\pm$ 0.00 & 0.11 $\pm$ 0.00 & 0.06 $\pm$ 0.00 & 0.02 $\pm$ 0.00 & 0.00 $\pm$ 0.00 & 0.00 $\pm$ 0.00 \\
RRISE & trainall\_initbase\_n1000 & 1000 & 0.26 $\pm$ 0.00 & 0.21 $\pm$ 0.00 & 0.15 $\pm$ 0.00 & 0.11 $\pm$ 0.00 & 0.06 $\pm$ 0.00 & 0.02 $\pm$ 0.00 & 0.00 $\pm$ 0.00 & 0.00 $\pm$ 0.00 \\
RRISE & trainall\_initbase\_n5000 & 5000 & 0.26 $\pm$ 0.00 & 0.20 $\pm$ 0.00 & 0.15 $\pm$ 0.00 & 0.10 $\pm$ 0.00 & 0.06 $\pm$ 0.00 & 0.02 $\pm$ 0.00 & 0.00 $\pm$ 0.00 & 0.00 $\pm$ 0.00 \\
RRISE & trainall\_initbase\_n10000 & 10000 & 0.27 $\pm$ 0.00 & 0.21 $\pm$ 0.00 & 0.16 $\pm$ 0.00 & 0.11 $\pm$ 0.00 & 0.06 $\pm$ 0.00 & 0.02 $\pm$ 0.00 & 0.00 $\pm$ 0.00 & 0.00 $\pm$ 0.00 \\
RRISE & trainall\_initrandom\_n500 & 500 & 0.00 $\pm$ 0.00 & 0.00 $\pm$ 0.00 & 0.00 $\pm$ 0.00 & 0.00 $\pm$ 0.00 & 0.00 $\pm$ 0.00 & 0.00 $\pm$ 0.00 & 0.00 $\pm$ 0.00 & 0.00 $\pm$ 0.00 \\
RRISE & trainall\_initrandom\_n1000 & 1000 & 0.00 $\pm$ 0.00 & 0.00 $\pm$ 0.00 & 0.00 $\pm$ 0.00 & 0.00 $\pm$ 0.00 & 0.00 $\pm$ 0.00 & 0.00 $\pm$ 0.00 & 0.00 $\pm$ 0.00 & 0.00 $\pm$ 0.00 \\
RRISE & trainall\_initrandom\_n5000 & 5000 & 0.00 $\pm$ 0.00 & 0.00 $\pm$ 0.00 & 0.00 $\pm$ 0.00 & 0.00 $\pm$ 0.00 & 0.00 $\pm$ 0.00 & 0.00 $\pm$ 0.00 & 0.00 $\pm$ 0.00 & 0.00 $\pm$ 0.00 \\
RRISE & trainall\_initrandom\_n10000 & 10000 & 0.00 $\pm$ 0.00 & 0.00 $\pm$ 0.00 & 0.00 $\pm$ 0.00 & 0.00 $\pm$ 0.00 & 0.00 $\pm$ 0.00 & 0.00 $\pm$ 0.00 & 0.00 $\pm$ 0.00 & 0.00 $\pm$ 0.00 \\
\midrule
Baseline 1 & n500 & 500 & 0.21 $\pm$ 0.00 & 0.16 $\pm$ 0.00 & 0.12 $\pm$ 0.00 & 0.08 $\pm$ 0.00 & 0.05 $\pm$ 0.00 & 0.01 $\pm$ 0.00 & 0.00 $\pm$ 0.00 & 0.00 $\pm$ 0.00 \\
Baseline 1 & n1000 & 1000 & 0.21 $\pm$ 0.00 & 0.16 $\pm$ 0.00 & 0.12 $\pm$ 0.00 & 0.08 $\pm$ 0.00 & 0.05 $\pm$ 0.00 & 0.01 $\pm$ 0.00 & 0.00 $\pm$ 0.00 & 0.00 $\pm$ 0.00 \\
Baseline 1 & n5000 & 5000 & 0.21 $\pm$ 0.00 & 0.16 $\pm$ 0.00 & 0.12 $\pm$ 0.00 & 0.08 $\pm$ 0.00 & 0.05 $\pm$ 0.00 & 0.01 $\pm$ 0.00 & 0.00 $\pm$ 0.00 & 0.00 $\pm$ 0.00 \\
Baseline 1 & n10000 & 10000 & 0.21 $\pm$ 0.00 & 0.16 $\pm$ 0.00 & 0.12 $\pm$ 0.00 & 0.08 $\pm$ 0.00 & 0.05 $\pm$ 0.00 & 0.01 $\pm$ 0.00 & 0.00 $\pm$ 0.00 & 0.00 $\pm$ 0.00 \\
\midrule
Baseline 2 & k500\_decline0.01 & 500 & 0.20 $\pm$ 0.00 & 0.15 $\pm$ 0.00 & 0.11 $\pm$ 0.00 & 0.07 $\pm$ 0.00 & 0.03 $\pm$ 0.00 & 0.01 $\pm$ 0.00 & 0.00 $\pm$ 0.00 & 0.00 $\pm$ 0.00 \\
Baseline 2 & k1000\_decline0.01 & 1000 & 0.19 $\pm$ 0.00 & 0.15 $\pm$ 0.00 & 0.10 $\pm$ 0.00 & 0.06 $\pm$ 0.00 & 0.03 $\pm$ 0.00 & 0.01 $\pm$ 0.00 & 0.00 $\pm$ 0.00 & 0.00 $\pm$ 0.00 \\
Baseline 2 & k5000\_decline0.01 & 5000 & 0.18 $\pm$ 0.00 & 0.14 $\pm$ 0.00 & 0.10 $\pm$ 0.00 & 0.05 $\pm$ 0.00 & 0.04 $\pm$ 0.00 & 0.01 $\pm$ 0.00 & 0.00 $\pm$ 0.00 & 0.00 $\pm$ 0.00 \\
Baseline 2 & k10000\_decline0.01 & 10000 & 0.18 $\pm$ 0.00 & 0.14 $\pm$ 0.00 & 0.10 $\pm$ 0.00 & 0.05 $\pm$ 0.00 & 0.04 $\pm$ 0.00 & 0.01 $\pm$ 0.00 & 0.00 $\pm$ 0.00 & 0.00 $\pm$ 0.00 \\
Baseline 2 & k500\_decline0.05 & 500 & 0.20 $\pm$ 0.00 & 0.16 $\pm$ 0.00 & 0.11 $\pm$ 0.00 & 0.07 $\pm$ 0.00 & 0.03 $\pm$ 0.00 & 0.01 $\pm$ 0.00 & 0.00 $\pm$ 0.00 & 0.00 $\pm$ 0.00 \\
Baseline 2 & k1000\_decline0.05 & 1000 & 0.19 $\pm$ 0.00 & 0.15 $\pm$ 0.00 & 0.11 $\pm$ 0.00 & 0.06 $\pm$ 0.00 & 0.03 $\pm$ 0.00 & 0.01 $\pm$ 0.00 & 0.00 $\pm$ 0.00 & 0.00 $\pm$ 0.00 \\
Baseline 2 & k5000\_decline0.05 & 5000 & 0.18 $\pm$ 0.00 & 0.14 $\pm$ 0.00 & 0.10 $\pm$ 0.00 & 0.05 $\pm$ 0.00 & 0.04 $\pm$ 0.00 & 0.01 $\pm$ 0.00 & 0.00 $\pm$ 0.00 & 0.00 $\pm$ 0.00 \\
Baseline 2 & k10000\_decline0.05 & 10000 & 0.18 $\pm$ 0.00 & 0.14 $\pm$ 0.00 & 0.10 $\pm$ 0.00 & 0.05 $\pm$ 0.00 & 0.04 $\pm$ 0.00 & 0.01 $\pm$ 0.00 & 0.00 $\pm$ 0.00 & 0.00 $\pm$ 0.00 \\
\midrule
Baseline 3 & n500\_tol0.01 & 500 & 0.21 $\pm$ 0.00 & 0.16 $\pm$ 0.00 & 0.12 $\pm$ 0.00 & 0.08 $\pm$ 0.00 & 0.05 $\pm$ 0.00 & 0.01 $\pm$ 0.00 & 0.00 $\pm$ 0.00 & 0.00 $\pm$ 0.00 \\
Baseline 3 & n1000\_tol0.01 & 1000 & 0.21 $\pm$ 0.00 & 0.16 $\pm$ 0.00 & 0.12 $\pm$ 0.00 & 0.08 $\pm$ 0.00 & 0.05 $\pm$ 0.00 & 0.01 $\pm$ 0.00 & 0.00 $\pm$ 0.00 & 0.00 $\pm$ 0.00 \\
Baseline 3 & n5000\_tol0.01 & 5000 & 0.21 $\pm$ 0.00 & 0.16 $\pm$ 0.00 & 0.12 $\pm$ 0.00 & 0.08 $\pm$ 0.00 & 0.05 $\pm$ 0.00 & 0.01 $\pm$ 0.00 & 0.00 $\pm$ 0.00 & 0.00 $\pm$ 0.00 \\
Baseline 3 & n10000\_tol0.01 & 10000 & 0.21 $\pm$ 0.00 & 0.16 $\pm$ 0.00 & 0.12 $\pm$ 0.00 & 0.08 $\pm$ 0.00 & 0.05 $\pm$ 0.00 & 0.01 $\pm$ 0.00 & 0.00 $\pm$ 0.00 & 0.00 $\pm$ 0.00 \\
Baseline 3 & n500\_tol0.05 & 500 & 0.21 $\pm$ 0.00 & 0.16 $\pm$ 0.00 & 0.12 $\pm$ 0.00 & 0.08 $\pm$ 0.00 & 0.05 $\pm$ 0.00 & 0.01 $\pm$ 0.00 & 0.00 $\pm$ 0.00 & 0.00 $\pm$ 0.00 \\
Baseline 3 & n1000\_tol0.05 & 1000 & 0.21 $\pm$ 0.00 & 0.16 $\pm$ 0.00 & 0.12 $\pm$ 0.00 & 0.09 $\pm$ 0.00 & 0.05 $\pm$ 0.00 & 0.01 $\pm$ 0.00 & 0.00 $\pm$ 0.00 & 0.00 $\pm$ 0.00 \\
Baseline 3 & n5000\_tol0.05 & 5000 & 0.21 $\pm$ 0.00 & 0.16 $\pm$ 0.00 & 0.12 $\pm$ 0.00 & 0.09 $\pm$ 0.00 & 0.05 $\pm$ 0.00 & 0.01 $\pm$ 0.00 & 0.00 $\pm$ 0.00 & 0.00 $\pm$ 0.00 \\
Baseline 3 & n10000\_tol0.05 & 10000 & 0.21 $\pm$ 0.00 & 0.17 $\pm$ 0.00 & 0.12 $\pm$ 0.00 & 0.08 $\pm$ 0.00 & 0.05 $\pm$ 0.00 & 0.01 $\pm$ 0.00 & 0.00 $\pm$ 0.00 & 0.00 $\pm$ 0.00 \\
\midrule
Baseline 4 & init0\_n500 & 500 & 0.00 $\pm$ 0.00 & 0.00 $\pm$ 0.00 & 0.00 $\pm$ 0.00 & 0.00 $\pm$ 0.00 & 0.00 $\pm$ 0.00 & 0.00 $\pm$ 0.00 & 0.00 $\pm$ 0.00 & 0.00 $\pm$ 0.00 \\
Baseline 4 & init0\_n1000 & 1000 & 0.00 $\pm$ 0.00 & 0.00 $\pm$ 0.00 & 0.00 $\pm$ 0.00 & 0.00 $\pm$ 0.00 & 0.00 $\pm$ 0.00 & 0.00 $\pm$ 0.00 & 0.00 $\pm$ 0.00 & 0.00 $\pm$ 0.00 \\
Baseline 4 & init0\_n5000 & 5000 & 0.00 $\pm$ 0.00 & 0.00 $\pm$ 0.00 & 0.00 $\pm$ 0.00 & 0.00 $\pm$ 0.00 & 0.00 $\pm$ 0.00 & 0.00 $\pm$ 0.00 & 0.00 $\pm$ 0.00 & 0.00 $\pm$ 0.00 \\
Baseline 4 & init0\_n10000 & 10000 & 0.00 $\pm$ 0.00 & 0.00 $\pm$ 0.00 & 0.00 $\pm$ 0.00 & 0.00 $\pm$ 0.00 & 0.00 $\pm$ 0.00 & 0.00 $\pm$ 0.00 & 0.00 $\pm$ 0.00 & 0.00 $\pm$ 0.00 \\
Baseline 4 & init1\_n500 & 500 & 0.58 $\pm$ 0.02 & 0.52 $\pm$ 0.02 & 0.47 $\pm$ 0.02 & 0.40 $\pm$ 0.03 & 0.30 $\pm$ 0.05 & 0.07 $\pm$ 0.12 & 0.00 $\pm$ 0.00 & 0.00 $\pm$ 0.00 \\
Baseline 4 & init1\_n1000 & 1000 & 0.36 $\pm$ 0.03 & 0.31 $\pm$ 0.03 & 0.26 $\pm$ 0.03 & 0.20 $\pm$ 0.03 & 0.14 $\pm$ 0.03 & 0.05 $\pm$ 0.02 & 0.00 $\pm$ 0.00 & 0.00 $\pm$ 0.00 \\
Baseline 4 & init1\_n5000 & 5000 & 0.34 $\pm$ 0.01 & 0.29 $\pm$ 0.01 & 0.24 $\pm$ 0.01 & 0.18 $\pm$ 0.01 & 0.12 $\pm$ 0.00 & 0.04 $\pm$ 0.00 & 0.00 $\pm$ 0.00 & 0.00 $\pm$ 0.00 \\
Baseline 4 & init1\_n10000 & 10000 & 0.42 $\pm$ 0.14 & 0.37 $\pm$ 0.14 & 0.32 $\pm$ 0.14 & 0.26 $\pm$ 0.13 & 0.20 $\pm$ 0.12 & 0.10 $\pm$ 0.10 & 0.00 $\pm$ 0.00 & 0.00 $\pm$ 0.00 \\
\bottomrule
\end{tabular}%
}
\end{table*}

\paragraph{Certified-input CRD.}
Tables~\ref{tab:app_crd_certified_full}--\ref{tab:app_crd_certified_full_tiny_imagenet} report the CRD conditioned on certified inputs, i.e., inputs that already satisfy $\widetilde p_A(\x)>1/2$.  These entries are conditional fractions over certified inputs: the value 1.00 at $t=0$ means that all certified inputs have radius larger than zero under that method and variant.

\begin{table*}[t]
\centering
\scriptsize
\caption{Full CRD results over certified inputs on FashionMNIST. Entries are conditional fractions over certified inputs; 1.00 means all certified inputs exceed the threshold.}
\label{tab:app_crd_certified_full}
\resizebox{\textwidth}{!}{%
\begin{tabular}{lllcccccccc}
\toprule
Method & Variant & Budget & $t=0$ & $t=0.025$ & $t=0.05$ & $t=0.075$ & $t=0.1$ & $t=0.125$ & $t=0.15$ & $t=0.175$ \\
\midrule
RRISE & trainhead\_initbase\_n500 & 500 & 1.00 $\pm$ 0.00 & 0.99 $\pm$ 0.00 & 0.99 $\pm$ 0.00 & 0.98 $\pm$ 0.00 & 0.97 $\pm$ 0.00 & 0.97 $\pm$ 0.00 & 0.96 $\pm$ 0.00 & 0.95 $\pm$ 0.00 \\
RRISE & trainhead\_initbase\_n1000 & 1000 & 1.00 $\pm$ 0.00 & 0.99 $\pm$ 0.00 & 0.99 $\pm$ 0.00 & 0.98 $\pm$ 0.00 & 0.97 $\pm$ 0.00 & 0.97 $\pm$ 0.00 & 0.96 $\pm$ 0.00 & 0.95 $\pm$ 0.00 \\
RRISE & trainhead\_initbase\_n5000 & 5000 & 1.00 $\pm$ 0.00 & 0.99 $\pm$ 0.00 & 0.99 $\pm$ 0.00 & 0.98 $\pm$ 0.00 & 0.97 $\pm$ 0.00 & 0.97 $\pm$ 0.00 & 0.96 $\pm$ 0.00 & 0.95 $\pm$ 0.00 \\
RRISE & trainhead\_initbase\_n10000 & 10000 & 1.00 $\pm$ 0.00 & 0.99 $\pm$ 0.00 & 0.99 $\pm$ 0.00 & 0.98 $\pm$ 0.00 & 0.97 $\pm$ 0.00 & 0.96 $\pm$ 0.00 & 0.96 $\pm$ 0.00 & 0.95 $\pm$ 0.00 \\
RRISE & trainall\_initbase\_n500 & 500 & 1.00 $\pm$ 0.00 & 0.99 $\pm$ 0.00 & 0.98 $\pm$ 0.00 & 0.98 $\pm$ 0.00 & 0.97 $\pm$ 0.00 & 0.96 $\pm$ 0.00 & 0.95 $\pm$ 0.00 & 0.95 $\pm$ 0.00 \\
RRISE & trainall\_initbase\_n1000 & 1000 & 1.00 $\pm$ 0.00 & 0.99 $\pm$ 0.00 & 0.99 $\pm$ 0.00 & 0.98 $\pm$ 0.00 & 0.97 $\pm$ 0.00 & 0.96 $\pm$ 0.00 & 0.95 $\pm$ 0.00 & 0.95 $\pm$ 0.00 \\
RRISE & trainall\_initbase\_n5000 & 5000 & 1.00 $\pm$ 0.00 & 0.99 $\pm$ 0.00 & 0.99 $\pm$ 0.00 & 0.98 $\pm$ 0.00 & 0.97 $\pm$ 0.00 & 0.96 $\pm$ 0.00 & 0.96 $\pm$ 0.00 & 0.95 $\pm$ 0.00 \\
RRISE & trainall\_initbase\_n10000 & 10000 & 1.00 $\pm$ 0.00 & 0.99 $\pm$ 0.00 & 0.98 $\pm$ 0.00 & 0.98 $\pm$ 0.00 & 0.97 $\pm$ 0.00 & 0.96 $\pm$ 0.00 & 0.95 $\pm$ 0.00 & 0.95 $\pm$ 0.00 \\
RRISE & trainall\_initrandom\_n500 & 500 & 1.00 $\pm$ 0.00 & 0.99 $\pm$ 0.00 & 0.98 $\pm$ 0.00 & 0.97 $\pm$ 0.00 & 0.96 $\pm$ 0.00 & 0.95 $\pm$ 0.00 & 0.94 $\pm$ 0.00 & 0.94 $\pm$ 0.00 \\
RRISE & trainall\_initrandom\_n1000 & 1000 & 1.00 $\pm$ 0.00 & 0.99 $\pm$ 0.00 & 0.98 $\pm$ 0.00 & 0.97 $\pm$ 0.00 & 0.96 $\pm$ 0.00 & 0.95 $\pm$ 0.00 & 0.94 $\pm$ 0.00 & 0.93 $\pm$ 0.00 \\
RRISE & trainall\_initrandom\_n5000 & 5000 & 1.00 $\pm$ 0.00 & 0.99 $\pm$ 0.00 & 0.98 $\pm$ 0.00 & 0.97 $\pm$ 0.00 & 0.96 $\pm$ 0.00 & 0.95 $\pm$ 0.00 & 0.94 $\pm$ 0.00 & 0.93 $\pm$ 0.00 \\
RRISE & trainall\_initrandom\_n10000 & 10000 & 1.00 $\pm$ 0.00 & 0.99 $\pm$ 0.00 & 0.98 $\pm$ 0.00 & 0.97 $\pm$ 0.00 & 0.96 $\pm$ 0.00 & 0.95 $\pm$ 0.00 & 0.94 $\pm$ 0.00 & 0.93 $\pm$ 0.00 \\
\midrule
Baseline 1 & n500 & 500 & 1.00 $\pm$ 0.00 & 0.99 $\pm$ 0.00 & 0.98 $\pm$ 0.00 & 0.98 $\pm$ 0.00 & 0.97 $\pm$ 0.00 & 0.96 $\pm$ 0.00 & 0.96 $\pm$ 0.00 & 0.95 $\pm$ 0.00 \\
Baseline 1 & n1000 & 1000 & 1.00 $\pm$ 0.00 & 0.99 $\pm$ 0.00 & 0.99 $\pm$ 0.00 & 0.98 $\pm$ 0.00 & 0.97 $\pm$ 0.00 & 0.96 $\pm$ 0.00 & 0.96 $\pm$ 0.00 & 0.95 $\pm$ 0.00 \\
Baseline 1 & n5000 & 5000 & 1.00 $\pm$ 0.00 & 0.99 $\pm$ 0.00 & 0.99 $\pm$ 0.00 & 0.98 $\pm$ 0.00 & 0.97 $\pm$ 0.00 & 0.96 $\pm$ 0.00 & 0.96 $\pm$ 0.00 & 0.95 $\pm$ 0.00 \\
Baseline 1 & n10000 & 10000 & 1.00 $\pm$ 0.00 & 0.99 $\pm$ 0.00 & 0.99 $\pm$ 0.00 & 0.98 $\pm$ 0.00 & 0.97 $\pm$ 0.00 & 0.96 $\pm$ 0.00 & 0.96 $\pm$ 0.00 & 0.95 $\pm$ 0.00 \\
\midrule
Baseline 2 & k500\_decline0.01 & 500 & 1.00 $\pm$ 0.00 & 0.99 $\pm$ 0.00 & 0.98 $\pm$ 0.00 & 0.98 $\pm$ 0.00 & 0.97 $\pm$ 0.00 & 0.96 $\pm$ 0.00 & 0.96 $\pm$ 0.00 & 0.95 $\pm$ 0.00 \\
Baseline 2 & k1000\_decline0.01 & 1000 & 1.00 $\pm$ 0.00 & 0.99 $\pm$ 0.00 & 0.98 $\pm$ 0.00 & 0.98 $\pm$ 0.00 & 0.97 $\pm$ 0.00 & 0.96 $\pm$ 0.00 & 0.96 $\pm$ 0.00 & 0.95 $\pm$ 0.00 \\
Baseline 2 & k5000\_decline0.01 & 5000 & 1.00 $\pm$ 0.00 & 0.99 $\pm$ 0.00 & 0.99 $\pm$ 0.00 & 0.98 $\pm$ 0.00 & 0.97 $\pm$ 0.00 & 0.97 $\pm$ 0.00 & 0.96 $\pm$ 0.00 & 0.96 $\pm$ 0.00 \\
Baseline 2 & k10000\_decline0.01 & 10000 & 1.00 $\pm$ 0.00 & 0.99 $\pm$ 0.00 & 0.98 $\pm$ 0.00 & 0.98 $\pm$ 0.00 & 0.97 $\pm$ 0.00 & 0.97 $\pm$ 0.00 & 0.96 $\pm$ 0.00 & 0.96 $\pm$ 0.00 \\
Baseline 2 & k500\_decline0.05 & 500 & 1.00 $\pm$ 0.00 & 0.99 $\pm$ 0.00 & 0.99 $\pm$ 0.00 & 0.98 $\pm$ 0.00 & 0.97 $\pm$ 0.00 & 0.96 $\pm$ 0.00 & 0.96 $\pm$ 0.00 & 0.95 $\pm$ 0.00 \\
Baseline 2 & k1000\_decline0.05 & 1000 & 1.00 $\pm$ 0.00 & 0.99 $\pm$ 0.00 & 0.99 $\pm$ 0.00 & 0.98 $\pm$ 0.00 & 0.97 $\pm$ 0.00 & 0.96 $\pm$ 0.00 & 0.96 $\pm$ 0.00 & 0.95 $\pm$ 0.00 \\
Baseline 2 & k5000\_decline0.05 & 5000 & 1.00 $\pm$ 0.00 & 0.99 $\pm$ 0.00 & 0.98 $\pm$ 0.00 & 0.98 $\pm$ 0.00 & 0.97 $\pm$ 0.00 & 0.97 $\pm$ 0.00 & 0.96 $\pm$ 0.00 & 0.96 $\pm$ 0.00 \\
Baseline 2 & k10000\_decline0.05 & 10000 & 1.00 $\pm$ 0.00 & 0.99 $\pm$ 0.00 & 0.98 $\pm$ 0.00 & 0.98 $\pm$ 0.00 & 0.97 $\pm$ 0.00 & 0.97 $\pm$ 0.00 & 0.96 $\pm$ 0.00 & 0.96 $\pm$ 0.00 \\
\midrule
Baseline 3 & n500\_tol0.01 & 500 & 1.00 $\pm$ 0.00 & 0.99 $\pm$ 0.00 & 0.98 $\pm$ 0.00 & 0.98 $\pm$ 0.00 & 0.97 $\pm$ 0.00 & 0.96 $\pm$ 0.00 & 0.96 $\pm$ 0.00 & 0.95 $\pm$ 0.00 \\
Baseline 3 & n1000\_tol0.01 & 1000 & 1.00 $\pm$ 0.00 & 0.99 $\pm$ 0.00 & 0.99 $\pm$ 0.00 & 0.98 $\pm$ 0.00 & 0.97 $\pm$ 0.00 & 0.96 $\pm$ 0.00 & 0.96 $\pm$ 0.00 & 0.95 $\pm$ 0.00 \\
Baseline 3 & n5000\_tol0.01 & 5000 & 1.00 $\pm$ 0.00 & 0.99 $\pm$ 0.00 & 0.99 $\pm$ 0.00 & 0.98 $\pm$ 0.00 & 0.97 $\pm$ 0.00 & 0.96 $\pm$ 0.00 & 0.96 $\pm$ 0.00 & 0.95 $\pm$ 0.00 \\
Baseline 3 & n10000\_tol0.01 & 10000 & 1.00 $\pm$ 0.00 & 0.99 $\pm$ 0.00 & 0.99 $\pm$ 0.00 & 0.98 $\pm$ 0.00 & 0.97 $\pm$ 0.00 & 0.96 $\pm$ 0.00 & 0.96 $\pm$ 0.00 & 0.95 $\pm$ 0.00 \\
Baseline 3 & n500\_tol0.05 & 500 & 1.00 $\pm$ 0.00 & 0.99 $\pm$ 0.00 & 0.98 $\pm$ 0.00 & 0.98 $\pm$ 0.00 & 0.97 $\pm$ 0.00 & 0.96 $\pm$ 0.00 & 0.96 $\pm$ 0.00 & 0.95 $\pm$ 0.00 \\
Baseline 3 & n1000\_tol0.05 & 1000 & 1.00 $\pm$ 0.00 & 0.99 $\pm$ 0.00 & 0.99 $\pm$ 0.00 & 0.98 $\pm$ 0.00 & 0.97 $\pm$ 0.00 & 0.96 $\pm$ 0.00 & 0.96 $\pm$ 0.00 & 0.95 $\pm$ 0.00 \\
Baseline 3 & n5000\_tol0.05 & 5000 & 1.00 $\pm$ 0.00 & 0.99 $\pm$ 0.00 & 0.98 $\pm$ 0.00 & 0.98 $\pm$ 0.00 & 0.97 $\pm$ 0.00 & 0.96 $\pm$ 0.00 & 0.96 $\pm$ 0.00 & 0.95 $\pm$ 0.00 \\
Baseline 3 & n10000\_tol0.05 & 10000 & 1.00 $\pm$ 0.00 & 0.99 $\pm$ 0.00 & 0.99 $\pm$ 0.00 & 0.98 $\pm$ 0.00 & 0.97 $\pm$ 0.00 & 0.96 $\pm$ 0.00 & 0.96 $\pm$ 0.00 & 0.95 $\pm$ 0.00 \\
\midrule
Baseline 4 & init0\_n500 & 500 & 1.00 $\pm$ 0.00 & 0.99 $\pm$ 0.00 & 0.98 $\pm$ 0.00 & 0.97 $\pm$ 0.00 & 0.96 $\pm$ 0.00 & 0.95 $\pm$ 0.00 & 0.95 $\pm$ 0.01 & 0.94 $\pm$ 0.01 \\
Baseline 4 & init0\_n1000 & 1000 & 1.00 $\pm$ 0.00 & 0.99 $\pm$ 0.00 & 0.98 $\pm$ 0.00 & 0.97 $\pm$ 0.00 & 0.96 $\pm$ 0.00 & 0.95 $\pm$ 0.00 & 0.94 $\pm$ 0.00 & 0.93 $\pm$ 0.00 \\
Baseline 4 & init0\_n5000 & 5000 & 1.00 $\pm$ 0.00 & 0.99 $\pm$ 0.00 & 0.98 $\pm$ 0.00 & 0.97 $\pm$ 0.00 & 0.96 $\pm$ 0.00 & 0.95 $\pm$ 0.00 & 0.94 $\pm$ 0.00 & 0.93 $\pm$ 0.00 \\
Baseline 4 & init0\_n10000 & 10000 & 1.00 $\pm$ 0.00 & 0.99 $\pm$ 0.00 & 0.98 $\pm$ 0.00 & 0.98 $\pm$ 0.00 & 0.97 $\pm$ 0.00 & 0.96 $\pm$ 0.01 & 0.95 $\pm$ 0.01 & 0.94 $\pm$ 0.01 \\
Baseline 4 & init1\_n500 & 500 & 1.00 $\pm$ 0.00 & 0.99 $\pm$ 0.00 & 0.99 $\pm$ 0.00 & 0.98 $\pm$ 0.00 & 0.97 $\pm$ 0.00 & 0.96 $\pm$ 0.00 & 0.96 $\pm$ 0.00 & 0.95 $\pm$ 0.00 \\
Baseline 4 & init1\_n1000 & 1000 & 1.00 $\pm$ 0.00 & 0.99 $\pm$ 0.00 & 0.98 $\pm$ 0.00 & 0.98 $\pm$ 0.00 & 0.97 $\pm$ 0.00 & 0.96 $\pm$ 0.00 & 0.96 $\pm$ 0.01 & 0.95 $\pm$ 0.01 \\
Baseline 4 & init1\_n5000 & 5000 & 1.00 $\pm$ 0.00 & 0.99 $\pm$ 0.00 & 0.99 $\pm$ 0.00 & 0.98 $\pm$ 0.00 & 0.97 $\pm$ 0.00 & 0.96 $\pm$ 0.00 & 0.96 $\pm$ 0.00 & 0.95 $\pm$ 0.00 \\
Baseline 4 & init1\_n10000 & 10000 & 1.00 $\pm$ 0.00 & 0.99 $\pm$ 0.00 & 0.99 $\pm$ 0.00 & 0.98 $\pm$ 0.00 & 0.97 $\pm$ 0.00 & 0.97 $\pm$ 0.00 & 0.96 $\pm$ 0.00 & 0.95 $\pm$ 0.00 \\
\bottomrule
\end{tabular}%
}
\end{table*}

\begin{table*}[t]
\centering
\scriptsize
\caption{Full CRD results over certified inputs on CIFAR-10. Entries are conditional fractions over certified inputs; 1.00 means all certified inputs exceed the threshold.}
\label{tab:app_crd_certified_full_cifar10}
\resizebox{\textwidth}{!}{%
\begin{tabular}{lllcccccccc}
\toprule
Method & Variant & Budget & $t=0$ & $t=0.025$ & $t=0.05$ & $t=0.075$ & $t=0.1$ & $t=0.125$ & $t=0.15$ & $t=0.175$ \\
\midrule
RRISE & trainhead\_initbase\_n500 & 500 & 1.00 $\pm$ 0.00 & 0.97 $\pm$ 0.00 & 0.95 $\pm$ 0.00 & 0.92 $\pm$ 0.00 & 0.89 $\pm$ 0.00 & 0.86 $\pm$ 0.00 & 0.84 $\pm$ 0.00 & 0.81 $\pm$ 0.00 \\
RRISE & trainhead\_initbase\_n1000 & 1000 & 1.00 $\pm$ 0.00 & 0.97 $\pm$ 0.00 & 0.95 $\pm$ 0.00 & 0.92 $\pm$ 0.00 & 0.90 $\pm$ 0.00 & 0.87 $\pm$ 0.00 & 0.84 $\pm$ 0.00 & 0.81 $\pm$ 0.00 \\
RRISE & trainhead\_initbase\_n5000 & 5000 & 1.00 $\pm$ 0.00 & 0.98 $\pm$ 0.00 & 0.95 $\pm$ 0.00 & 0.92 $\pm$ 0.00 & 0.90 $\pm$ 0.00 & 0.87 $\pm$ 0.00 & 0.85 $\pm$ 0.00 & 0.82 $\pm$ 0.00 \\
RRISE & trainhead\_initbase\_n10000 & 10000 & 1.00 $\pm$ 0.00 & 0.98 $\pm$ 0.00 & 0.95 $\pm$ 0.00 & 0.92 $\pm$ 0.00 & 0.90 $\pm$ 0.00 & 0.87 $\pm$ 0.00 & 0.85 $\pm$ 0.00 & 0.82 $\pm$ 0.00 \\
RRISE & trainall\_initbase\_n500 & 500 & 1.00 $\pm$ 0.00 & 0.97 $\pm$ 0.00 & 0.94 $\pm$ 0.00 & 0.92 $\pm$ 0.00 & 0.89 $\pm$ 0.00 & 0.86 $\pm$ 0.00 & 0.84 $\pm$ 0.00 & 0.81 $\pm$ 0.00 \\
RRISE & trainall\_initbase\_n1000 & 1000 & 1.00 $\pm$ 0.00 & 0.97 $\pm$ 0.00 & 0.94 $\pm$ 0.00 & 0.92 $\pm$ 0.00 & 0.89 $\pm$ 0.00 & 0.87 $\pm$ 0.00 & 0.84 $\pm$ 0.00 & 0.82 $\pm$ 0.00 \\
RRISE & trainall\_initbase\_n5000 & 5000 & 1.00 $\pm$ 0.00 & 0.97 $\pm$ 0.00 & 0.94 $\pm$ 0.00 & 0.92 $\pm$ 0.00 & 0.89 $\pm$ 0.00 & 0.87 $\pm$ 0.00 & 0.84 $\pm$ 0.00 & 0.82 $\pm$ 0.00 \\
RRISE & trainall\_initbase\_n10000 & 10000 & 1.00 $\pm$ 0.00 & 0.97 $\pm$ 0.00 & 0.95 $\pm$ 0.00 & 0.92 $\pm$ 0.00 & 0.89 $\pm$ 0.00 & 0.87 $\pm$ 0.00 & 0.84 $\pm$ 0.00 & 0.82 $\pm$ 0.00 \\
RRISE & trainall\_initrandom\_n500 & 500 & 1.00 $\pm$ 0.00 & 0.94 $\pm$ 0.00 & 0.89 $\pm$ 0.01 & 0.82 $\pm$ 0.01 & 0.76 $\pm$ 0.01 & 0.67 $\pm$ 0.02 & 0.55 $\pm$ 0.05 & 0.12 $\pm$ 0.21 \\
RRISE & trainall\_initrandom\_n1000 & 1000 & 1.00 $\pm$ 0.00 & 0.95 $\pm$ 0.00 & 0.89 $\pm$ 0.00 & 0.83 $\pm$ 0.01 & 0.78 $\pm$ 0.01 & 0.71 $\pm$ 0.02 & 0.62 $\pm$ 0.04 & 0.50 $\pm$ 0.08 \\
RRISE & trainall\_initrandom\_n5000 & 5000 & 1.00 $\pm$ 0.00 & 0.95 $\pm$ 0.00 & 0.90 $\pm$ 0.01 & 0.84 $\pm$ 0.01 & 0.79 $\pm$ 0.02 & 0.73 $\pm$ 0.02 & 0.66 $\pm$ 0.03 & 0.58 $\pm$ 0.04 \\
RRISE & trainall\_initrandom\_n10000 & 10000 & 1.00 $\pm$ 0.00 & 0.95 $\pm$ 0.01 & 0.90 $\pm$ 0.01 & 0.84 $\pm$ 0.02 & 0.79 $\pm$ 0.03 & 0.73 $\pm$ 0.04 & 0.66 $\pm$ 0.06 & 0.56 $\pm$ 0.13 \\
\midrule
Baseline 1 & n500 & 500 & 1.00 $\pm$ 0.00 & 0.97 $\pm$ 0.00 & 0.94 $\pm$ 0.00 & 0.92 $\pm$ 0.00 & 0.89 $\pm$ 0.00 & 0.86 $\pm$ 0.00 & 0.84 $\pm$ 0.00 & 0.82 $\pm$ 0.00 \\
Baseline 1 & n1000 & 1000 & 1.00 $\pm$ 0.00 & 0.97 $\pm$ 0.00 & 0.94 $\pm$ 0.00 & 0.92 $\pm$ 0.00 & 0.89 $\pm$ 0.00 & 0.87 $\pm$ 0.00 & 0.84 $\pm$ 0.00 & 0.82 $\pm$ 0.00 \\
Baseline 1 & n5000 & 5000 & 1.00 $\pm$ 0.00 & 0.97 $\pm$ 0.00 & 0.94 $\pm$ 0.00 & 0.92 $\pm$ 0.00 & 0.89 $\pm$ 0.00 & 0.87 $\pm$ 0.00 & 0.84 $\pm$ 0.00 & 0.82 $\pm$ 0.00 \\
Baseline 1 & n10000 & 10000 & 1.00 $\pm$ 0.00 & 0.97 $\pm$ 0.00 & 0.94 $\pm$ 0.00 & 0.92 $\pm$ 0.00 & 0.89 $\pm$ 0.00 & 0.87 $\pm$ 0.00 & 0.84 $\pm$ 0.00 & 0.82 $\pm$ 0.00 \\
\midrule
Baseline 2 & k500\_decline0.01 & 500 & 1.00 $\pm$ 0.00 & 0.97 $\pm$ 0.00 & 0.94 $\pm$ 0.00 & 0.91 $\pm$ 0.00 & 0.88 $\pm$ 0.00 & 0.86 $\pm$ 0.00 & 0.84 $\pm$ 0.00 & 0.82 $\pm$ 0.00 \\
Baseline 2 & k1000\_decline0.01 & 1000 & 1.00 $\pm$ 0.00 & 0.97 $\pm$ 0.00 & 0.94 $\pm$ 0.00 & 0.91 $\pm$ 0.00 & 0.88 $\pm$ 0.00 & 0.86 $\pm$ 0.00 & 0.85 $\pm$ 0.00 & 0.83 $\pm$ 0.00 \\
Baseline 2 & k5000\_decline0.01 & 5000 & 1.00 $\pm$ 0.00 & 0.97 $\pm$ 0.00 & 0.94 $\pm$ 0.00 & 0.91 $\pm$ 0.00 & 0.89 $\pm$ 0.00 & 0.88 $\pm$ 0.00 & 0.86 $\pm$ 0.00 & 0.84 $\pm$ 0.00 \\
Baseline 2 & k10000\_decline0.01 & 10000 & 1.00 $\pm$ 0.00 & 0.97 $\pm$ 0.00 & 0.94 $\pm$ 0.00 & 0.91 $\pm$ 0.00 & 0.89 $\pm$ 0.00 & 0.88 $\pm$ 0.00 & 0.86 $\pm$ 0.00 & 0.84 $\pm$ 0.00 \\
Baseline 2 & k500\_decline0.05 & 500 & 1.00 $\pm$ 0.00 & 0.97 $\pm$ 0.00 & 0.94 $\pm$ 0.00 & 0.91 $\pm$ 0.00 & 0.89 $\pm$ 0.00 & 0.86 $\pm$ 0.00 & 0.84 $\pm$ 0.00 & 0.82 $\pm$ 0.00 \\
Baseline 2 & k1000\_decline0.05 & 1000 & 1.00 $\pm$ 0.00 & 0.97 $\pm$ 0.00 & 0.94 $\pm$ 0.00 & 0.91 $\pm$ 0.00 & 0.88 $\pm$ 0.00 & 0.86 $\pm$ 0.00 & 0.85 $\pm$ 0.00 & 0.83 $\pm$ 0.00 \\
Baseline 2 & k5000\_decline0.05 & 5000 & 1.00 $\pm$ 0.00 & 0.97 $\pm$ 0.00 & 0.94 $\pm$ 0.00 & 0.91 $\pm$ 0.00 & 0.89 $\pm$ 0.00 & 0.88 $\pm$ 0.00 & 0.86 $\pm$ 0.00 & 0.83 $\pm$ 0.00 \\
Baseline 2 & k10000\_decline0.05 & 10000 & 1.00 $\pm$ 0.00 & 0.97 $\pm$ 0.00 & 0.94 $\pm$ 0.00 & 0.91 $\pm$ 0.00 & 0.89 $\pm$ 0.00 & 0.88 $\pm$ 0.00 & 0.86 $\pm$ 0.00 & 0.83 $\pm$ 0.00 \\
\midrule
Baseline 3 & n500\_tol0.01 & 500 & 1.00 $\pm$ 0.00 & 0.97 $\pm$ 0.00 & 0.94 $\pm$ 0.00 & 0.92 $\pm$ 0.00 & 0.89 $\pm$ 0.00 & 0.87 $\pm$ 0.00 & 0.84 $\pm$ 0.00 & 0.82 $\pm$ 0.00 \\
Baseline 3 & n1000\_tol0.01 & 1000 & 1.00 $\pm$ 0.00 & 0.97 $\pm$ 0.00 & 0.94 $\pm$ 0.00 & 0.91 $\pm$ 0.00 & 0.89 $\pm$ 0.00 & 0.87 $\pm$ 0.00 & 0.84 $\pm$ 0.00 & 0.82 $\pm$ 0.00 \\
Baseline 3 & n5000\_tol0.01 & 5000 & 1.00 $\pm$ 0.00 & 0.97 $\pm$ 0.00 & 0.94 $\pm$ 0.00 & 0.91 $\pm$ 0.00 & 0.89 $\pm$ 0.00 & 0.87 $\pm$ 0.00 & 0.84 $\pm$ 0.00 & 0.82 $\pm$ 0.00 \\
Baseline 3 & n10000\_tol0.01 & 10000 & 1.00 $\pm$ 0.00 & 0.97 $\pm$ 0.00 & 0.94 $\pm$ 0.00 & 0.91 $\pm$ 0.00 & 0.89 $\pm$ 0.00 & 0.87 $\pm$ 0.00 & 0.84 $\pm$ 0.00 & 0.82 $\pm$ 0.00 \\
Baseline 3 & n500\_tol0.05 & 500 & 1.00 $\pm$ 0.00 & 0.97 $\pm$ 0.00 & 0.94 $\pm$ 0.00 & 0.92 $\pm$ 0.00 & 0.89 $\pm$ 0.00 & 0.86 $\pm$ 0.00 & 0.84 $\pm$ 0.00 & 0.82 $\pm$ 0.00 \\
Baseline 3 & n1000\_tol0.05 & 1000 & 1.00 $\pm$ 0.00 & 0.97 $\pm$ 0.00 & 0.94 $\pm$ 0.00 & 0.92 $\pm$ 0.00 & 0.89 $\pm$ 0.00 & 0.86 $\pm$ 0.00 & 0.84 $\pm$ 0.00 & 0.81 $\pm$ 0.00 \\
Baseline 3 & n5000\_tol0.05 & 5000 & 1.00 $\pm$ 0.00 & 0.97 $\pm$ 0.00 & 0.94 $\pm$ 0.00 & 0.92 $\pm$ 0.00 & 0.89 $\pm$ 0.00 & 0.87 $\pm$ 0.00 & 0.84 $\pm$ 0.00 & 0.82 $\pm$ 0.00 \\
Baseline 3 & n10000\_tol0.05 & 10000 & 1.00 $\pm$ 0.00 & 0.97 $\pm$ 0.00 & 0.94 $\pm$ 0.00 & 0.91 $\pm$ 0.00 & 0.89 $\pm$ 0.00 & 0.86 $\pm$ 0.00 & 0.84 $\pm$ 0.00 & 0.82 $\pm$ 0.00 \\
\midrule
Baseline 4 & init0\_n500 & 500 & 1.00 $\pm$ 0.00 & 0.92 $\pm$ 0.00 & 0.83 $\pm$ 0.01 & 0.70 $\pm$ 0.02 & 0.17 $\pm$ 0.29 & 0.00 $\pm$ 0.00 & 0.00 $\pm$ 0.00 & 0.00 $\pm$ 0.00 \\
Baseline 4 & init0\_n1000 & 1000 & 1.00 $\pm$ 0.00 & 0.93 $\pm$ 0.01 & 0.86 $\pm$ 0.02 & 0.76 $\pm$ 0.05 & 0.45 $\pm$ 0.39 & 0.33 $\pm$ 0.28 & 0.00 $\pm$ 0.00 & 0.00 $\pm$ 0.00 \\
Baseline 4 & init0\_n5000 & 5000 & 1.00 $\pm$ 0.00 & 0.94 $\pm$ 0.01 & 0.88 $\pm$ 0.02 & 0.80 $\pm$ 0.02 & 0.71 $\pm$ 0.05 & 0.42 $\pm$ 0.37 & 0.19 $\pm$ 0.33 & 0.12 $\pm$ 0.21 \\
Baseline 4 & init0\_n10000 & 10000 & 1.00 $\pm$ 0.00 & 0.94 $\pm$ 0.00 & 0.87 $\pm$ 0.01 & 0.79 $\pm$ 0.02 & 0.70 $\pm$ 0.04 & 0.47 $\pm$ 0.25 & 0.24 $\pm$ 0.21 & 0.00 $\pm$ 0.00 \\
Baseline 4 & init1\_n500 & 500 & 1.00 $\pm$ 0.00 & 0.97 $\pm$ 0.00 & 0.94 $\pm$ 0.00 & 0.92 $\pm$ 0.00 & 0.89 $\pm$ 0.01 & 0.86 $\pm$ 0.01 & 0.83 $\pm$ 0.01 & 0.80 $\pm$ 0.01 \\
Baseline 4 & init1\_n1000 & 1000 & 1.00 $\pm$ 0.00 & 0.97 $\pm$ 0.00 & 0.95 $\pm$ 0.00 & 0.92 $\pm$ 0.00 & 0.89 $\pm$ 0.00 & 0.86 $\pm$ 0.01 & 0.83 $\pm$ 0.01 & 0.81 $\pm$ 0.01 \\
Baseline 4 & init1\_n5000 & 5000 & 1.00 $\pm$ 0.00 & 0.97 $\pm$ 0.00 & 0.95 $\pm$ 0.00 & 0.92 $\pm$ 0.00 & 0.89 $\pm$ 0.01 & 0.87 $\pm$ 0.01 & 0.84 $\pm$ 0.01 & 0.81 $\pm$ 0.01 \\
Baseline 4 & init1\_n10000 & 10000 & 1.00 $\pm$ 0.00 & 0.97 $\pm$ 0.00 & 0.94 $\pm$ 0.00 & 0.92 $\pm$ 0.00 & 0.89 $\pm$ 0.00 & 0.86 $\pm$ 0.00 & 0.84 $\pm$ 0.00 & 0.81 $\pm$ 0.00 \\
\bottomrule
\end{tabular}%
}
\end{table*}

\begin{table*}[t]
\centering
\scriptsize
\caption{Full CRD results over certified inputs on CIFAR-100. Entries are conditional fractions over certified inputs; 1.00 means all certified inputs exceed the threshold.}
\label{tab:app_crd_certified_full_cifar100}
\resizebox{\textwidth}{!}{%
\begin{tabular}{lllcccccccc}
\toprule
Method & Variant & Budget & $t=0$ & $t=0.0125$ & $t=0.025$ & $t=0.0375$ & $t=0.05$ & $t=0.0625$ & $t=0.075$ & $t=0.0875$ \\
\midrule
RRISE & trainhead\_initbase\_n500 & 500 & 1.00 $\pm$ 0.00 & 0.94 $\pm$ 0.00 & 0.87 $\pm$ 0.00 & 0.80 $\pm$ 0.00 & 0.73 $\pm$ 0.00 & 0.64 $\pm$ 0.00 & 0.50 $\pm$ 0.01 & 0.00 $\pm$ 0.00 \\
RRISE & trainhead\_initbase\_n1000 & 1000 & 1.00 $\pm$ 0.00 & 0.94 $\pm$ 0.00 & 0.87 $\pm$ 0.00 & 0.80 $\pm$ 0.00 & 0.73 $\pm$ 0.01 & 0.64 $\pm$ 0.02 & 0.50 $\pm$ 0.05 & 0.11 $\pm$ 0.19 \\
RRISE & trainhead\_initbase\_n5000 & 5000 & 1.00 $\pm$ 0.00 & 0.94 $\pm$ 0.00 & 0.89 $\pm$ 0.00 & 0.82 $\pm$ 0.01 & 0.76 $\pm$ 0.01 & 0.69 $\pm$ 0.02 & 0.59 $\pm$ 0.03 & 0.45 $\pm$ 0.09 \\
RRISE & trainhead\_initbase\_n10000 & 10000 & 1.00 $\pm$ 0.00 & 0.94 $\pm$ 0.00 & 0.88 $\pm$ 0.00 & 0.82 $\pm$ 0.00 & 0.76 $\pm$ 0.00 & 0.70 $\pm$ 0.01 & 0.62 $\pm$ 0.01 & 0.51 $\pm$ 0.03 \\
RRISE & trainall\_initbase\_n500 & 500 & 1.00 $\pm$ 0.00 & 0.94 $\pm$ 0.00 & 0.88 $\pm$ 0.00 & 0.82 $\pm$ 0.00 & 0.77 $\pm$ 0.01 & 0.72 $\pm$ 0.00 & 0.67 $\pm$ 0.00 & 0.63 $\pm$ 0.00 \\
RRISE & trainall\_initbase\_n1000 & 1000 & 1.00 $\pm$ 0.00 & 0.94 $\pm$ 0.00 & 0.88 $\pm$ 0.01 & 0.83 $\pm$ 0.00 & 0.78 $\pm$ 0.00 & 0.73 $\pm$ 0.00 & 0.68 $\pm$ 0.01 & 0.63 $\pm$ 0.01 \\
RRISE & trainall\_initbase\_n5000 & 5000 & 1.00 $\pm$ 0.00 & 0.94 $\pm$ 0.00 & 0.88 $\pm$ 0.00 & 0.82 $\pm$ 0.00 & 0.77 $\pm$ 0.00 & 0.73 $\pm$ 0.00 & 0.68 $\pm$ 0.00 & 0.64 $\pm$ 0.00 \\
RRISE & trainall\_initbase\_n10000 & 10000 & 1.00 $\pm$ 0.00 & 0.94 $\pm$ 0.00 & 0.88 $\pm$ 0.00 & 0.82 $\pm$ 0.00 & 0.77 $\pm$ 0.00 & 0.73 $\pm$ 0.00 & 0.68 $\pm$ 0.00 & 0.64 $\pm$ 0.00 \\
RRISE & trainall\_initrandom\_n500 & 500 & -- & -- & -- & -- & -- & -- & -- & -- \\
RRISE & trainall\_initrandom\_n1000 & 1000 & -- & -- & -- & -- & -- & -- & -- & -- \\
RRISE & trainall\_initrandom\_n5000 & 5000 & -- & -- & -- & -- & -- & -- & -- & -- \\
RRISE & trainall\_initrandom\_n10000 & 10000 & -- & -- & -- & -- & -- & -- & -- & -- \\
\midrule
Baseline 1 & n500 & 500 & 1.00 $\pm$ 0.00 & 0.96 $\pm$ 0.00 & 0.92 $\pm$ 0.00 & 0.88 $\pm$ 0.00 & 0.84 $\pm$ 0.00 & 0.80 $\pm$ 0.00 & 0.77 $\pm$ 0.00 & 0.73 $\pm$ 0.00 \\
Baseline 1 & n1000 & 1000 & 1.00 $\pm$ 0.00 & 0.96 $\pm$ 0.00 & 0.92 $\pm$ 0.00 & 0.88 $\pm$ 0.00 & 0.84 $\pm$ 0.00 & 0.80 $\pm$ 0.00 & 0.77 $\pm$ 0.00 & 0.73 $\pm$ 0.00 \\
Baseline 1 & n5000 & 5000 & 1.00 $\pm$ 0.00 & 0.96 $\pm$ 0.00 & 0.92 $\pm$ 0.00 & 0.88 $\pm$ 0.00 & 0.84 $\pm$ 0.00 & 0.80 $\pm$ 0.00 & 0.77 $\pm$ 0.00 & 0.73 $\pm$ 0.00 \\
Baseline 1 & n10000 & 10000 & 1.00 $\pm$ 0.00 & 0.96 $\pm$ 0.00 & 0.92 $\pm$ 0.00 & 0.88 $\pm$ 0.00 & 0.84 $\pm$ 0.00 & 0.80 $\pm$ 0.00 & 0.77 $\pm$ 0.00 & 0.73 $\pm$ 0.00 \\
\midrule
Baseline 2 & k500\_decline0.01 & 500 & 1.00 $\pm$ 0.00 & 0.96 $\pm$ 0.00 & 0.91 $\pm$ 0.00 & 0.87 $\pm$ 0.00 & 0.83 $\pm$ 0.00 & 0.80 $\pm$ 0.00 & 0.78 $\pm$ 0.00 & 0.74 $\pm$ 0.00 \\
Baseline 2 & k1000\_decline0.01 & 1000 & 1.00 $\pm$ 0.00 & 0.96 $\pm$ 0.00 & 0.91 $\pm$ 0.00 & 0.87 $\pm$ 0.00 & 0.83 $\pm$ 0.00 & 0.81 $\pm$ 0.00 & 0.78 $\pm$ 0.00 & 0.75 $\pm$ 0.00 \\
Baseline 2 & k5000\_decline0.01 & 5000 & 1.00 $\pm$ 0.00 & 0.96 $\pm$ 0.00 & 0.91 $\pm$ 0.00 & 0.86 $\pm$ 0.00 & 0.85 $\pm$ 0.00 & 0.82 $\pm$ 0.00 & 0.79 $\pm$ 0.00 & 0.75 $\pm$ 0.00 \\
Baseline 2 & k10000\_decline0.01 & 10000 & 1.00 $\pm$ 0.00 & 0.96 $\pm$ 0.00 & 0.91 $\pm$ 0.00 & 0.86 $\pm$ 0.00 & 0.85 $\pm$ 0.00 & 0.82 $\pm$ 0.00 & 0.79 $\pm$ 0.00 & 0.75 $\pm$ 0.00 \\
Baseline 2 & k500\_decline0.05 & 500 & 1.00 $\pm$ 0.00 & 0.96 $\pm$ 0.00 & 0.91 $\pm$ 0.00 & 0.87 $\pm$ 0.00 & 0.83 $\pm$ 0.00 & 0.80 $\pm$ 0.00 & 0.77 $\pm$ 0.00 & 0.74 $\pm$ 0.00 \\
Baseline 2 & k1000\_decline0.05 & 1000 & 1.00 $\pm$ 0.00 & 0.96 $\pm$ 0.00 & 0.91 $\pm$ 0.00 & 0.87 $\pm$ 0.00 & 0.83 $\pm$ 0.00 & 0.81 $\pm$ 0.00 & 0.78 $\pm$ 0.00 & 0.75 $\pm$ 0.00 \\
Baseline 2 & k5000\_decline0.05 & 5000 & 1.00 $\pm$ 0.00 & 0.96 $\pm$ 0.00 & 0.91 $\pm$ 0.00 & 0.86 $\pm$ 0.00 & 0.85 $\pm$ 0.00 & 0.82 $\pm$ 0.00 & 0.79 $\pm$ 0.00 & 0.75 $\pm$ 0.00 \\
Baseline 2 & k10000\_decline0.05 & 10000 & 1.00 $\pm$ 0.00 & 0.96 $\pm$ 0.00 & 0.91 $\pm$ 0.00 & 0.87 $\pm$ 0.00 & 0.85 $\pm$ 0.00 & 0.82 $\pm$ 0.00 & 0.79 $\pm$ 0.00 & 0.75 $\pm$ 0.00 \\
\midrule
Baseline 3 & n500\_tol0.01 & 500 & 1.00 $\pm$ 0.00 & 0.96 $\pm$ 0.00 & 0.92 $\pm$ 0.00 & 0.88 $\pm$ 0.00 & 0.84 $\pm$ 0.00 & 0.80 $\pm$ 0.00 & 0.77 $\pm$ 0.00 & 0.73 $\pm$ 0.00 \\
Baseline 3 & n1000\_tol0.01 & 1000 & 1.00 $\pm$ 0.00 & 0.96 $\pm$ 0.00 & 0.92 $\pm$ 0.00 & 0.88 $\pm$ 0.00 & 0.84 $\pm$ 0.00 & 0.80 $\pm$ 0.00 & 0.77 $\pm$ 0.00 & 0.73 $\pm$ 0.00 \\
Baseline 3 & n5000\_tol0.01 & 5000 & 1.00 $\pm$ 0.00 & 0.96 $\pm$ 0.00 & 0.92 $\pm$ 0.00 & 0.88 $\pm$ 0.00 & 0.84 $\pm$ 0.00 & 0.80 $\pm$ 0.00 & 0.77 $\pm$ 0.00 & 0.73 $\pm$ 0.00 \\
Baseline 3 & n10000\_tol0.01 & 10000 & 1.00 $\pm$ 0.00 & 0.96 $\pm$ 0.00 & 0.92 $\pm$ 0.00 & 0.88 $\pm$ 0.00 & 0.84 $\pm$ 0.00 & 0.80 $\pm$ 0.00 & 0.77 $\pm$ 0.00 & 0.73 $\pm$ 0.00 \\
Baseline 3 & n500\_tol0.05 & 500 & 1.00 $\pm$ 0.00 & 0.96 $\pm$ 0.00 & 0.92 $\pm$ 0.00 & 0.88 $\pm$ 0.00 & 0.84 $\pm$ 0.00 & 0.80 $\pm$ 0.00 & 0.77 $\pm$ 0.00 & 0.73 $\pm$ 0.00 \\
Baseline 3 & n1000\_tol0.05 & 1000 & 1.00 $\pm$ 0.00 & 0.96 $\pm$ 0.00 & 0.92 $\pm$ 0.00 & 0.88 $\pm$ 0.00 & 0.84 $\pm$ 0.00 & 0.80 $\pm$ 0.00 & 0.77 $\pm$ 0.00 & 0.73 $\pm$ 0.00 \\
Baseline 3 & n5000\_tol0.05 & 5000 & 1.00 $\pm$ 0.00 & 0.96 $\pm$ 0.00 & 0.92 $\pm$ 0.00 & 0.88 $\pm$ 0.00 & 0.84 $\pm$ 0.00 & 0.80 $\pm$ 0.00 & 0.77 $\pm$ 0.00 & 0.73 $\pm$ 0.00 \\
Baseline 3 & n10000\_tol0.05 & 10000 & 1.00 $\pm$ 0.00 & 0.96 $\pm$ 0.00 & 0.92 $\pm$ 0.00 & 0.88 $\pm$ 0.00 & 0.84 $\pm$ 0.00 & 0.80 $\pm$ 0.00 & 0.77 $\pm$ 0.00 & 0.73 $\pm$ 0.00 \\
\midrule
Baseline 4 & init0\_n500 & 500 & -- & -- & -- & -- & -- & -- & -- & -- \\
Baseline 4 & init0\_n1000 & 1000 & -- & -- & -- & -- & -- & -- & -- & -- \\
Baseline 4 & init0\_n5000 & 5000 & -- & -- & -- & -- & -- & -- & -- & -- \\
Baseline 4 & init0\_n10000 & 10000 & -- & -- & -- & -- & -- & -- & -- & -- \\
Baseline 4 & init1\_n500 & 500 & 1.00 $\pm$ 0.00 & 0.61 $\pm$ 0.03 & 0.00 $\pm$ 0.00 & 0.00 $\pm$ 0.00 & 0.00 $\pm$ 0.00 & 0.00 $\pm$ 0.00 & 0.00 $\pm$ 0.00 & 0.00 $\pm$ 0.00 \\
Baseline 4 & init1\_n1000 & 1000 & 1.00 $\pm$ 0.00 & 0.70 $\pm$ 0.06 & 0.00 $\pm$ 0.00 & 0.00 $\pm$ 0.00 & 0.00 $\pm$ 0.00 & 0.00 $\pm$ 0.00 & 0.00 $\pm$ 0.00 & 0.00 $\pm$ 0.00 \\
Baseline 4 & init1\_n5000 & 5000 & 1.00 $\pm$ 0.00 & 0.72 $\pm$ 0.03 & 0.00 $\pm$ 0.00 & 0.00 $\pm$ 0.00 & 0.00 $\pm$ 0.00 & 0.00 $\pm$ 0.00 & 0.00 $\pm$ 0.00 & 0.00 $\pm$ 0.00 \\
Baseline 4 & init1\_n10000 & 10000 & 1.00 $\pm$ 0.00 & 0.73 $\pm$ 0.03 & 0.00 $\pm$ 0.00 & 0.00 $\pm$ 0.00 & 0.00 $\pm$ 0.00 & 0.00 $\pm$ 0.00 & 0.00 $\pm$ 0.00 & 0.00 $\pm$ 0.00 \\
\bottomrule
\end{tabular}%
}
\end{table*}

\begin{table*}[t]
\centering
\scriptsize
\caption{Full CRD results over certified inputs on Tiny ImageNet. Entries are conditional fractions over certified inputs; 1.00 means all certified inputs exceed the threshold.}
\label{tab:app_crd_certified_full_tiny_imagenet}
\resizebox{\textwidth}{!}{%
\begin{tabular}{lllcccccccc}
\toprule
Method & Variant & Budget & $t=0$ & $t=0.0125$ & $t=0.025$ & $t=0.0375$ & $t=0.05$ & $t=0.0625$ & $t=0.075$ & $t=0.0875$ \\
\midrule
RRISE & trainhead\_initbase\_n500 & 500 & 1.00 $\pm$ 0.00 & 0.95 $\pm$ 0.00 & 0.89 $\pm$ 0.00 & 0.84 $\pm$ 0.00 & 0.78 $\pm$ 0.00 & 0.72 $\pm$ 0.01 & 0.65 $\pm$ 0.01 & 0.56 $\pm$ 0.02 \\
RRISE & trainhead\_initbase\_n1000 & 1000 & 1.00 $\pm$ 0.00 & 0.95 $\pm$ 0.00 & 0.90 $\pm$ 0.00 & 0.84 $\pm$ 0.00 & 0.79 $\pm$ 0.01 & 0.73 $\pm$ 0.01 & 0.67 $\pm$ 0.01 & 0.59 $\pm$ 0.02 \\
RRISE & trainhead\_initbase\_n5000 & 5000 & 1.00 $\pm$ 0.00 & 0.95 $\pm$ 0.00 & 0.90 $\pm$ 0.00 & 0.85 $\pm$ 0.00 & 0.80 $\pm$ 0.00 & 0.75 $\pm$ 0.00 & 0.69 $\pm$ 0.01 & 0.63 $\pm$ 0.01 \\
RRISE & trainhead\_initbase\_n10000 & 10000 & 1.00 $\pm$ 0.00 & 0.95 $\pm$ 0.00 & 0.90 $\pm$ 0.00 & 0.85 $\pm$ 0.00 & 0.80 $\pm$ 0.01 & 0.74 $\pm$ 0.01 & 0.69 $\pm$ 0.01 & 0.63 $\pm$ 0.02 \\
RRISE & trainall\_initbase\_n500 & 500 & 1.00 $\pm$ 0.00 & 0.92 $\pm$ 0.00 & 0.85 $\pm$ 0.00 & 0.79 $\pm$ 0.00 & 0.72 $\pm$ 0.01 & 0.66 $\pm$ 0.01 & 0.59 $\pm$ 0.01 & 0.53 $\pm$ 0.02 \\
RRISE & trainall\_initbase\_n1000 & 1000 & 1.00 $\pm$ 0.00 & 0.92 $\pm$ 0.00 & 0.85 $\pm$ 0.01 & 0.78 $\pm$ 0.01 & 0.72 $\pm$ 0.01 & 0.65 $\pm$ 0.02 & 0.59 $\pm$ 0.03 & 0.53 $\pm$ 0.04 \\
RRISE & trainall\_initbase\_n5000 & 5000 & 1.00 $\pm$ 0.00 & 0.92 $\pm$ 0.00 & 0.85 $\pm$ 0.00 & 0.79 $\pm$ 0.00 & 0.73 $\pm$ 0.00 & 0.67 $\pm$ 0.00 & 0.61 $\pm$ 0.00 & 0.56 $\pm$ 0.01 \\
RRISE & trainall\_initbase\_n10000 & 10000 & 1.00 $\pm$ 0.00 & 0.92 $\pm$ 0.00 & 0.85 $\pm$ 0.00 & 0.79 $\pm$ 0.00 & 0.73 $\pm$ 0.00 & 0.67 $\pm$ 0.00 & 0.61 $\pm$ 0.00 & 0.56 $\pm$ 0.01 \\
RRISE & trainall\_initrandom\_n500 & 500 & -- & -- & -- & -- & -- & -- & -- & -- \\
RRISE & trainall\_initrandom\_n1000 & 1000 & -- & -- & -- & -- & -- & -- & -- & -- \\
RRISE & trainall\_initrandom\_n5000 & 5000 & -- & -- & -- & -- & -- & -- & -- & -- \\
RRISE & trainall\_initrandom\_n10000 & 10000 & -- & -- & -- & -- & -- & -- & -- & -- \\
\midrule
Baseline 1 & n500 & 500 & 1.00 $\pm$ 0.00 & 0.95 $\pm$ 0.00 & 0.91 $\pm$ 0.00 & 0.87 $\pm$ 0.00 & 0.83 $\pm$ 0.00 & 0.79 $\pm$ 0.00 & 0.76 $\pm$ 0.00 & 0.72 $\pm$ 0.00 \\
Baseline 1 & n1000 & 1000 & 1.00 $\pm$ 0.00 & 0.95 $\pm$ 0.00 & 0.91 $\pm$ 0.00 & 0.87 $\pm$ 0.00 & 0.83 $\pm$ 0.00 & 0.79 $\pm$ 0.00 & 0.76 $\pm$ 0.00 & 0.72 $\pm$ 0.00 \\
Baseline 1 & n5000 & 5000 & 1.00 $\pm$ 0.00 & 0.95 $\pm$ 0.00 & 0.91 $\pm$ 0.00 & 0.87 $\pm$ 0.00 & 0.83 $\pm$ 0.00 & 0.79 $\pm$ 0.00 & 0.76 $\pm$ 0.00 & 0.72 $\pm$ 0.00 \\
Baseline 1 & n10000 & 10000 & 1.00 $\pm$ 0.00 & 0.95 $\pm$ 0.00 & 0.91 $\pm$ 0.00 & 0.87 $\pm$ 0.00 & 0.83 $\pm$ 0.00 & 0.79 $\pm$ 0.00 & 0.76 $\pm$ 0.00 & 0.72 $\pm$ 0.00 \\
\midrule
Baseline 2 & k500\_decline0.01 & 500 & 1.00 $\pm$ 0.00 & 0.95 $\pm$ 0.00 & 0.91 $\pm$ 0.00 & 0.86 $\pm$ 0.00 & 0.82 $\pm$ 0.00 & 0.79 $\pm$ 0.00 & 0.77 $\pm$ 0.00 & 0.73 $\pm$ 0.00 \\
Baseline 2 & k1000\_decline0.01 & 1000 & 1.00 $\pm$ 0.00 & 0.95 $\pm$ 0.00 & 0.91 $\pm$ 0.00 & 0.86 $\pm$ 0.00 & 0.82 $\pm$ 0.00 & 0.80 $\pm$ 0.00 & 0.77 $\pm$ 0.00 & 0.74 $\pm$ 0.00 \\
Baseline 2 & k5000\_decline0.01 & 5000 & 1.00 $\pm$ 0.00 & 0.95 $\pm$ 0.00 & 0.91 $\pm$ 0.00 & 0.86 $\pm$ 0.00 & 0.84 $\pm$ 0.00 & 0.81 $\pm$ 0.00 & 0.78 $\pm$ 0.00 & 0.74 $\pm$ 0.00 \\
Baseline 2 & k10000\_decline0.01 & 10000 & 1.00 $\pm$ 0.00 & 0.95 $\pm$ 0.00 & 0.91 $\pm$ 0.00 & 0.86 $\pm$ 0.00 & 0.84 $\pm$ 0.00 & 0.81 $\pm$ 0.00 & 0.78 $\pm$ 0.00 & 0.74 $\pm$ 0.00 \\
Baseline 2 & k500\_decline0.05 & 500 & 1.00 $\pm$ 0.00 & 0.95 $\pm$ 0.00 & 0.91 $\pm$ 0.00 & 0.86 $\pm$ 0.00 & 0.82 $\pm$ 0.00 & 0.79 $\pm$ 0.00 & 0.77 $\pm$ 0.00 & 0.73 $\pm$ 0.00 \\
Baseline 2 & k1000\_decline0.05 & 1000 & 1.00 $\pm$ 0.00 & 0.95 $\pm$ 0.00 & 0.91 $\pm$ 0.00 & 0.86 $\pm$ 0.00 & 0.82 $\pm$ 0.00 & 0.80 $\pm$ 0.00 & 0.77 $\pm$ 0.00 & 0.74 $\pm$ 0.00 \\
Baseline 2 & k5000\_decline0.05 & 5000 & 1.00 $\pm$ 0.00 & 0.95 $\pm$ 0.00 & 0.91 $\pm$ 0.00 & 0.86 $\pm$ 0.00 & 0.84 $\pm$ 0.00 & 0.81 $\pm$ 0.00 & 0.78 $\pm$ 0.00 & 0.74 $\pm$ 0.00 \\
Baseline 2 & k10000\_decline0.05 & 10000 & 1.00 $\pm$ 0.00 & 0.96 $\pm$ 0.00 & 0.91 $\pm$ 0.00 & 0.86 $\pm$ 0.00 & 0.84 $\pm$ 0.00 & 0.81 $\pm$ 0.00 & 0.78 $\pm$ 0.00 & 0.74 $\pm$ 0.00 \\
\midrule
Baseline 3 & n500\_tol0.01 & 500 & 1.00 $\pm$ 0.00 & 0.95 $\pm$ 0.00 & 0.91 $\pm$ 0.00 & 0.87 $\pm$ 0.00 & 0.83 $\pm$ 0.00 & 0.79 $\pm$ 0.00 & 0.75 $\pm$ 0.00 & 0.72 $\pm$ 0.00 \\
Baseline 3 & n1000\_tol0.01 & 1000 & 1.00 $\pm$ 0.00 & 0.95 $\pm$ 0.00 & 0.91 $\pm$ 0.00 & 0.87 $\pm$ 0.00 & 0.83 $\pm$ 0.00 & 0.79 $\pm$ 0.00 & 0.76 $\pm$ 0.00 & 0.72 $\pm$ 0.00 \\
Baseline 3 & n5000\_tol0.01 & 5000 & 1.00 $\pm$ 0.00 & 0.95 $\pm$ 0.00 & 0.91 $\pm$ 0.00 & 0.87 $\pm$ 0.00 & 0.83 $\pm$ 0.00 & 0.79 $\pm$ 0.00 & 0.76 $\pm$ 0.00 & 0.72 $\pm$ 0.00 \\
Baseline 3 & n10000\_tol0.01 & 10000 & 1.00 $\pm$ 0.00 & 0.95 $\pm$ 0.00 & 0.91 $\pm$ 0.00 & 0.87 $\pm$ 0.00 & 0.83 $\pm$ 0.00 & 0.79 $\pm$ 0.00 & 0.76 $\pm$ 0.00 & 0.72 $\pm$ 0.00 \\
Baseline 3 & n500\_tol0.05 & 500 & 1.00 $\pm$ 0.00 & 0.95 $\pm$ 0.00 & 0.91 $\pm$ 0.00 & 0.87 $\pm$ 0.00 & 0.83 $\pm$ 0.00 & 0.79 $\pm$ 0.00 & 0.75 $\pm$ 0.00 & 0.72 $\pm$ 0.00 \\
Baseline 3 & n1000\_tol0.05 & 1000 & 1.00 $\pm$ 0.00 & 0.95 $\pm$ 0.00 & 0.91 $\pm$ 0.00 & 0.87 $\pm$ 0.00 & 0.83 $\pm$ 0.00 & 0.79 $\pm$ 0.00 & 0.75 $\pm$ 0.00 & 0.72 $\pm$ 0.00 \\
Baseline 3 & n5000\_tol0.05 & 5000 & 1.00 $\pm$ 0.00 & 0.95 $\pm$ 0.00 & 0.91 $\pm$ 0.00 & 0.87 $\pm$ 0.00 & 0.83 $\pm$ 0.00 & 0.79 $\pm$ 0.00 & 0.75 $\pm$ 0.00 & 0.72 $\pm$ 0.00 \\
Baseline 3 & n10000\_tol0.05 & 10000 & 1.00 $\pm$ 0.00 & 0.95 $\pm$ 0.00 & 0.91 $\pm$ 0.00 & 0.87 $\pm$ 0.00 & 0.82 $\pm$ 0.00 & 0.79 $\pm$ 0.00 & 0.75 $\pm$ 0.00 & 0.72 $\pm$ 0.00 \\
\midrule
Baseline 4 & init0\_n500 & 500 & -- & -- & -- & -- & -- & -- & -- & -- \\
Baseline 4 & init0\_n1000 & 1000 & -- & -- & -- & -- & -- & -- & -- & -- \\
Baseline 4 & init0\_n5000 & 5000 & -- & -- & -- & -- & -- & -- & -- & -- \\
Baseline 4 & init0\_n10000 & 10000 & -- & -- & -- & -- & -- & -- & -- & -- \\
Baseline 4 & init1\_n500 & 500 & 1.00 $\pm$ 0.00 & 0.91 $\pm$ 0.00 & 0.81 $\pm$ 0.01 & 0.69 $\pm$ 0.03 & 0.52 $\pm$ 0.07 & 0.12 $\pm$ 0.21 & 0.00 $\pm$ 0.00 & 0.00 $\pm$ 0.00 \\
Baseline 4 & init1\_n1000 & 1000 & 1.00 $\pm$ 0.00 & 0.92 $\pm$ 0.00 & 0.83 $\pm$ 0.00 & 0.74 $\pm$ 0.00 & 0.64 $\pm$ 0.01 & 0.50 $\pm$ 0.03 & 0.11 $\pm$ 0.19 & 0.00 $\pm$ 0.00 \\
Baseline 4 & init1\_n5000 & 5000 & 1.00 $\pm$ 0.00 & 0.92 $\pm$ 0.00 & 0.84 $\pm$ 0.00 & 0.75 $\pm$ 0.01 & 0.65 $\pm$ 0.01 & 0.53 $\pm$ 0.01 & 0.31 $\pm$ 0.02 & 0.00 $\pm$ 0.00 \\
Baseline 4 & init1\_n10000 & 10000 & 1.00 $\pm$ 0.00 & 0.92 $\pm$ 0.01 & 0.83 $\pm$ 0.01 & 0.74 $\pm$ 0.02 & 0.63 $\pm$ 0.04 & 0.47 $\pm$ 0.09 & 0.14 $\pm$ 0.12 & 0.00 $\pm$ 0.00 \\
\bottomrule
\end{tabular}%
}
\end{table*}

\clearpage
\newpage
\section{Theoretical Properties of the Offline \ours Objective}
\label{app:rrise_theoretical_strength}

This appendix formalizes the statistical role of the offline MC targets used by \ours. The key fact is simple but important: for any fixed surrogate parameter vector $\btheta$, the soft-label cross-entropy computed with the empirical MC target is an unbiased estimator of the ideal cross-entropy to the true smoothed distribution. The finite MC budget $n$ controls the variance of this supervision. The learned surrogate can still contain finite-target, approximation, and optimization error, which is why calibrated lower-bound correction is applied before reporting certified radii.

\subsection{Offline MC Targets}
\label{app:rrise_mc_targets}

For a fixed input vector $\x$, smoothing level $\sigma$, and base classifier $f$, define the true smoothed class distribution
\begin{equation}
 p(k\mid\x,\sigma)
 =
 \Prob_{\bepsilon\sim\Normal(\bzero,\sigma^2\bI)}
 \!\left[f(\x+\bepsilon)=k\right].
\end{equation}
\ours constructs the empirical MC target
\begin{equation}
 \widehat p_n(k\mid\x,\sigma)
 =
 \frac{1}{n}\sum_{j=1}^n
 \bone\!\left[f(\x+\bepsilon_j)=k\right],
 \qquad
 \bepsilon_j\stackrel{\mathrm{i.i.d.}}{\sim}\Normal(\bzero,\sigma^2\bI).
\end{equation}
For every class $k$,
\begin{equation}
 \Ex\!\left[\widehat p_n(k\mid\x,\sigma)\right]
 =
 p(k\mid\x,\sigma),
 \qquad
 \operatorname{Var}\!\left[\widehat p_n(k\mid\x,\sigma)\right]
 =
 \frac{p(k\mid\x,\sigma)(1-p(k\mid\x,\sigma))}{n}.
\end{equation}
Equivalently,
\begin{equation}
 n\widehat p_n(\cdot\mid\x,\sigma)
 \sim
 \operatorname{Multinomial}\!\left(n,p(\cdot\mid\x,\sigma)\right),
\end{equation}
and
\begin{equation}
 \operatorname{Cov}\!\left[\widehat p_n(\cdot\mid\x,\sigma)\right]
 =
 \frac{1}{n}\left(\diag(p)-pp^\top\right),
\end{equation}
where $p=p(\cdot\mid\x,\sigma)$. Thus larger $n$ gives lower-variance supervision, while smaller $n$ reduces offline cost at the price of noisier targets.

\subsection{Objective Alignment}
\label{app:rrise_objective_alignment}

The ideal population objective for learning the smoothed distribution is
\begin{equation}
 \mathcal L^\star(\btheta)
 =
 \Ex_{\x}\!\left[
 -\sum_{k=1}^K p(k\mid\x,\sigma)\log q_{\btheta}(k\mid\x)
 \right].
\end{equation}
The empirical-target objective is
\begin{equation}
 \widehat{\mathcal L}_n(\btheta)
 =
 \Ex_{\x}\!\left[
 -\sum_{k=1}^K \widehat p_n(k\mid\x,\sigma)\log q_{\btheta}(k\mid\x)
 \right].
\end{equation}
For fixed $\btheta$, the objective is linear in $\widehat p_n$. Therefore,
\begin{align}
 \Ex_{\widehat p_n}\!\left[\widehat{\mathcal L}_n(\btheta)\right]
 & =
 \Ex_{\x}\!\left[
 -\sum_{k=1}^K \Ex\!\left[\widehat p_n(k\mid\x,\sigma)\right]
 \log q_{\btheta}(k\mid\x)
 \right] \\
 & =
 \Ex_{\x}\!\left[
 -\sum_{k=1}^K p(k\mid\x,\sigma)\log q_{\btheta}(k\mid\x)
 \right]
 =
 \mathcal L^\star(\btheta).
\end{align}
Under regularity conditions that permit exchanging differentiation and expectation,
\begin{equation}
 \Ex_{\widehat p_n}\!\left[\nabla_{\btheta}\widehat{\mathcal L}_n(\btheta)\right]
 =
 \nabla_{\btheta}\mathcal L^\star(\btheta).
\end{equation}
Thus the finite-$n$ loss and gradient are unbiased at any fixed parameter vector. The optimized parameter vector is a nonlinear function of the finite-$n$ targets, so the final learned model can still reflect target noise, model approximation, and optimization effects.

\subsection{Population Optimum}
\label{app:rrise_population_optimum}

For a fixed input $\x$, the ideal cross-entropy is
\begin{equation}
 \mathcal L_{\x}(q)
 =
 -\sum_{k=1}^K p(k\mid\x,\sigma)\log q(k\mid\x).
\end{equation}
It decomposes as
\begin{equation}
 \mathcal L_{\x}(q)
 =
 H(p(\cdot\mid\x,\sigma))
 +
 \KL\!\left(p(\cdot\mid\x,\sigma)\middle\|q(\cdot\mid\x)\right).
\end{equation}
The entropy term does not depend on $q$, and the KL divergence is minimized exactly when $q(\cdot\mid\x)=p(\cdot\mid\x,\sigma)$. Therefore, with sufficient model capacity and successful optimization, the population objective targets the same smoothed distribution used by randomized smoothing.

\subsection{Calibration and Amortization}
\label{app:rrise_calibration_strength}

The surrogate output $q_{\btheta}(\cdot\mid\x)$ is a fast estimate of the smoothed distribution, not itself a certificate. Calibration converts it into the lower-bound quantity
\begin{equation}
 \ptildeA(\x)=\min\{1,\max\{0,\qA(\x)-\delta\}\},
 \qquad \delta\ge0.
\end{equation}
Therefore $\ptildeA(\x)\le\qA(\x)$, and calibration can only decrease the top probability used for radius certification. This separates the learning problem from the certification problem: learning provides a one-forward-pass approximation, while calibration controls overestimation before a radius is reported.

The computational benefit is amortization. Fixed-budget RS certifies $m$ inputs with $n$ samples at cost
\begin{equation}
 C_{\mathrm{MC}}(m)=mn
\end{equation}
base-model forward passes. \ours pays the target-construction cost once,
\begin{equation}
 C_{\mathrm{target}}\approx |\mathcal D_{\mathrm{train}}|n,
\end{equation}
then certifies $m$ inputs at cost
\begin{equation}
 C_{\mathrm{RRISE}}(m)=C_{\mathrm{target}}+C_{\mathrm{train}}+m.
\end{equation}
Thus \ours is most useful when many future certificates are required under the same base classifier and smoothing distribution.

\clearpage
\newpage
\section{Comparison Between \ours and Baseline~4}
\label{app:rrise_vs_baseline4}

Baseline~4~\citep{bhardwaj2024accelerated} is the closest prior offline-surrogate method for accelerating randomized smoothing. Like \ours, it constructs an offline dataset of MC class-count targets from a frozen base classifier and trains a surrogate to predict the smoothed class distribution. The two methods differ along the following components: (i) the training objective --- \ours uses soft-label cross-entropy, whereas Baseline~4 uses Jensen--Shannon divergence; (ii) the surrogate architecture and initialization --- \ours fine-tunes the base classifier itself, whereas Baseline~4 trains a separate ResNet surrogate from random initialization regardless of the base architecture; and (iii) the certification post-processing --- \ours applies a one-time conformal calibration before reporting any radius, whereas Baseline~4, as originally proposed, has no calibration step. This appendix discusses each in turn, with particular attention to the gradient consequences of the objective choice.

\subsection{Shared Offline Target}
\label{app:rrise_b4_common_target}

Both methods use the same offline target. For each training input $\x_i$, draw $\bepsilon_{i,j}\stackrel{\mathrm{i.i.d.}}{\sim}\Normal(\bzero,\sigma^2\bI)$ for $j=1,\dots,n$ and form the empirical class-count vector
\begin{equation}
 \widehat p_i(k)
 \;\triangleq\;
 \frac{1}{n}\sum_{j=1}^n
 \bone\left[f(\x_i+\bepsilon_{i,j})=k\right],
 \qquad k=1,\dots,K.
\end{equation}
Writing $p_i(k)\triangleq p(k\mid\x_i,\sigma)$ for the true smoothed probability from~\eqref{eq:smoothed_prob}, the target satisfies $\Ex[\widehat p_i(k)]=p_i(k)$ but is itself a noisy estimator with finite-$n$ deviation from $p_i$. Both methods inherit the same $\widehat p_i$ and therefore the same target-construction bias relative to the population $p_i$. The question this appendix addresses is not whether $\widehat p_i$ is noisy --- it is, equally, for both methods --- but whether each method's gradient is an unbiased estimator of the gradient of its own ideal objective, taken with respect to $\widehat p_i$.

\subsection{\ours: Unbiased Gradients with Respect to \texorpdfstring{$\widehat p_i$}{p-hat-i}}
\label{app:rrise_ce_vs_b4}

\ours minimizes the soft-label cross-entropy
\begin{equation}
 \widehat{\mathcal L}_{\mathrm{RRISE}}(\btheta)
 \;=\;
 -\frac{1}{|\mathcal B|}
 \sum_{i\in\mathcal B}\sum_{k=1}^K
 \widehat p_i(k)\,\log q_{\btheta}(\x_i)_k.
\end{equation}
Cross-entropy is \emph{linear} in its first argument. Consequently, for any fixed $\btheta$ the gradient
\begin{equation}
 \nabla_{\btheta}\widehat{\mathcal L}_{\mathrm{RRISE}}(\btheta)
 \;=\;
 -\frac{1}{|\mathcal B|}
 \sum_{i\in\mathcal B}\sum_{k=1}^K
 \widehat p_i(k)\,\nabla_{\btheta}\log q_{\btheta}(\x_i)_k
\end{equation}
is itself linear in $\widehat p_i$. Linearity is the essential property: at fixed $\btheta$, the finite-$n$ gradient is the exact gradient of the cross-entropy objective evaluated at the realized target $\widehat p_i$, with no curvature correction. The MC budget $n$ enters only through the variance of $\widehat p_i$ and hence the variance of the gradient; it does not introduce any nonlinear distortion between $\widehat p_i$ and the gradient direction.

\subsection{Baseline~4: Biased Gradients with Respect to \texorpdfstring{$\widehat p_i$}{p-hat-i}}
\label{app:b4_js_objective}

Baseline~4 minimizes
\begin{equation}
 \widehat{\mathcal L}_{\mathrm{B4}}(\btheta)
 \;=\;
 \frac{1}{|\mathcal B|}
 \sum_{i\in\mathcal B}
 \JS\!\left(\widehat p_i\,\middle\|\,q_{\btheta}(\x_i)\right).
\end{equation}
Jensen--Shannon divergence is \emph{nonlinear} in its first argument. To make the gradient consequence precise, fix $\btheta$ and let $\widetilde p_i$ denote a second i.i.d.\ MC realization with the same $n$, drawn independently from the same distribution as $\widehat p_i$. Then $\Ex[\widetilde p_i]=\Ex[\widehat p_i]=\bar p_i$, but in general
\begin{equation}
 \Ex_{\widetilde p_i}\!\left[\nabla_{\btheta}\JS(\widetilde p_i\,\|\,q_{\btheta}(\x_i))\right]
 \;\neq\;
 \nabla_{\btheta}\JS\!\left(\bar p_i\,\big\|\,q_{\btheta}(\x_i)\right).
\end{equation}
That is, even within the family of finite-$n$ targets the method itself uses, the expected gradient at $\btheta$ is offset from the gradient evaluated at the mean target by a curvature-dependent term. A second-order Taylor expansion of any smooth divergence $D(\cdot,q)$ around $\bar p_i$ formalizes this: writing $\widetilde p_i=\bar p_i+\x_i$ with $\Ex[\x_i]=\bzero$ and $\mathrm{Cov}(\x_i)=\tfrac{1}{n}\!\left(\diag(\bar p_i)-\bar p_i\bar p_i^\top\right)$, differentiation of the second-order term yields an $O(1/n)$ gradient offset whose magnitude scales with $\nabla^2_{pp}D$. For cross-entropy, $\nabla^2_{pp}D\equiv\bzero$ and the offset vanishes identically; for Jensen--Shannon divergence, $\nabla^2_{pp}D$ is generally nonzero on the simplex and the offset persists at every $\btheta$ until $n\to\infty$.

\subsection{Surrogate Architecture and Initialization}
\label{app:rrise_b4_architecture}

The two methods also differ in how the surrogate is constructed. \ours fine-tunes the base classifier directly: $q_{\btheta}$ inherits the base architecture --- MLP-Mixer-Tiny, ResNet-18, EfficientNet-B0, or ViT-Tiny in our experiments --- and is initialized from base-classifier weights, with only the estimator head trained by default (Section~\ref{sec:training}, with end-to-end and random-initialization variants ablated in Appendix~\ref{app:ablation_list}). Baseline~4, in contrast, fixes a ResNet surrogate trained from random initialization regardless of the base classifier, following the protocol in~\citet{bhardwaj2024accelerated}. Two practical consequences follow. First, \ours reuses representations the base classifier has already learned to be approximately invariant to Gaussian noise (since base classifiers are trained with noise augmentation), giving the surrogate a strong initialization and reducing the offline training cost. Second, the base architecture and surrogate architecture are guaranteed to match, so any architectural prior that helps the base classifier on a given dataset --- patch-mixing on FashionMNIST, attention on Tiny ImageNet --- transfers to the surrogate at no additional design cost.

\subsection{Calibration}
\label{app:rrise_b4_certification_treatment}

The third difference concerns how each method translates surrogate outputs into certified radii. Baseline~4, as originally proposed, has no calibration step and treats surrogate probabilities directly as smoothed probabilities; this is unsafe for certification, because a point estimate that overestimates the smoothed top probability will inflate the radius. \ours instead applies the conformal calibration of Section~\ref{sec:calibration}, which yields a finite-sample lower bound on $p_A(\x)$ before any radius is reported. For a fair empirical comparison, we apply the same calibration protocol to Baseline~4, so that any radius gap reflects the surrogates themselves under a common post-processing rule rather than the absence of calibration in the original method.

\subsection{Summary}
\label{app:rrise_b4_summary}

\begin{table}[h]
\centering
\small
\caption{Comparison between \ours and Baseline~4 under the shared offline-target construction. The three groups correspond to the three axes of difference: training objective, surrogate construction, and calibration.}
\label{tab:rrise_vs_baseline4}
\begin{tabular}{lll}
\toprule
Property & \ours & Baseline~4 \\
\midrule
\multicolumn{3}{l}{\textit{Shared offline target}}\\
Offline MC targets & Yes & Yes \\
Target distribution & Normalized MC class counts & Normalized MC class counts \\
Test-time MC sampling & No & No \\
\midrule
\multicolumn{3}{l}{\textit{Training objective}}\\
Loss function & Soft-label cross-entropy & Jensen--Shannon divergence \\
Linear in target $\widehat p_i$ & Yes & No \\
Gradient unbiased w.r.t.\ $\widehat p_i$ at fixed $\btheta$ & Yes & Generally no \\
\midrule
\multicolumn{3}{l}{\textit{Surrogate architecture and initialization}}\\
Architecture & Matches base classifier & Fixed ResNet \\
Initialization & Base-classifier weights & Random \\
\midrule
\multicolumn{3}{l}{\textit{Calibration}}\\
In original method & Yes & No \\
In our evaluation & Yes & Yes (added for fair comparison) \\
\bottomrule
\end{tabular}
\end{table}
\clearpage
\newpage
\section{Proofs}
\label{app:calibration_proof}

This appendix proves Proposition~\ref{prop:calibration} and Corollary~\ref{cor:argmax_agreement}. The proof separates three ingredients: exchangeability of conformal residuals, the one-sided Clopper--Pearson guarantee for the calibration sampling step, and the randomized-smoothing certificate.

Throughout, the base classifier $f$, surrogate $q_{\btheta}$, smoothing level $\sigma$, calibration sample budget $n$, and failure levels $\beta,\gamma$ are fixed before calibration and test inputs are drawn. Calibration inputs $\x_1^{\mathrm{cal}},\dots,\x_M^{\mathrm{cal}}$ and the test input $\x$ are i.i.d. from the deployment distribution $\mathcal D$.

\subsection{Exchangeability of Calibration Residuals}

\begin{lemma}[Exchangeability]
\label{lem:exchangeability}
Let each calibration or test input be paired with independent Gaussian perturbations $\bepsilon_{1:n}\stackrel{\mathrm{i.i.d.}}{\sim}\Normal(\bzero,\sigma^2\bI)$. Define the residual at any input $\x'$ by
\begin{equation}
 r(\x')
 =
 \qA(\x')-\underline p(\x'),
\end{equation}
where $\underline p(\x')$ is the Clopper--Pearson lower bound, computed from the perturbations at $\x'$, for the probability of the surrogate-predicted class $\hat g(\x')$. Then the residuals of the $M$ calibration inputs and the residual of the independent test input are exchangeable.
\end{lemma}

\begin{proof}
Each residual is obtained by applying the same deterministic mapping to an input--noise pair: evaluate $q_{\btheta}$ on the clean input, count noisy base-model predictions agreeing with the resulting surrogate class, compute the Clopper--Pearson lower bound, and subtract it from $\qA$. The input--noise pairs are i.i.d., hence exchangeable. Applying the same measurable mapping to each coordinate preserves exchangeability.
\end{proof}

\subsection{Proof of Proposition~\ref{prop:calibration}}

\calibrationprop*

\begin{proof}
For analysis only, imagine drawing $n$ fresh perturbations at the independent test input $\x$ and computing the corresponding one-sided Clopper--Pearson lower bound $\underline p_{\mathrm{test}}$ for the probability of the surrogate-predicted class $\hat g(\x)$. Since the agreement indicators are Bernoulli with parameter $p(\hat g(\x)\mid\x,\sigma)$, the Clopper--Pearson construction gives
\begin{equation}
 \Prob\!\left[
 \underline p_{\mathrm{test}}
 \le
 p(\hat g(\x)\mid\x,\sigma)
 \right]
 \ge 1-\beta.
 \label{eq:cp_bound}
\end{equation}
Let
\begin{equation}
 r_{\mathrm{test}}
 =
 \qA(\x)-\underline p_{\mathrm{test}}.
\end{equation}
By Lemma~\ref{lem:exchangeability}, the calibration residuals and $r_{\mathrm{test}}$ are exchangeable. If $r_{(k)}$ is the $k$-th order statistic of the calibration residuals with $k=\lceil(M+1)(1-\gamma)\rceil$, split conformal prediction yields
\begin{equation}
 \Prob\!\left[r_{\mathrm{test}}\le r_{(k)}\right]
 \ge 1-\gamma.
\end{equation}
The deployed offset is $\delta=\max\{0,r_{(k)}\}$, so $\delta\ge r_{(k)}$ and therefore
\begin{equation}
 \Prob\!\left[
 \qA(\x)-\delta
 \le
 \underline p_{\mathrm{test}}
 \right]
 \ge 1-\gamma.
 \label{eq:conformal_bound}
\end{equation}
For every input $\x$, the probability of the surrogate-chosen class is bounded by the smoothed top-class probability:
\begin{equation}
 p(\hat g(\x)\mid\x,\sigma)
 \le
 \max_c p(c\mid\x,\sigma)
 =
 \pA(\x;\sigma).
 \label{eq:top_probability_dominates}
\end{equation}
Combining~\eqref{eq:cp_bound} and~\eqref{eq:conformal_bound} by a union bound gives an event of probability at least $1-\beta-\gamma$ on which
\begin{equation}
 \qA(\x)-\delta
 \le
 \underline p_{\mathrm{test}}
 \le
 p(\hat g(\x)\mid\x,\sigma)
 \le
 \pA(\x;\sigma).
 \label{eq:proof_chain}
\end{equation}
This proves~\eqref{eq:coverage}.
\end{proof}

\subsection{Proof of Corollary~\ref{cor:argmax_agreement}}

\argmaxcorollary*

\begin{proof}
Assume $\qA(\x)-\delta>1/2$. On the event in~\eqref{eq:proof_chain}, which has probability at least $1-\beta-\gamma$, we have
\begin{equation}
 p(\hat g(\x)\mid\x,\sigma)
 \ge
 \qA(\x)-\delta
 >
 \frac12.
 \label{eq:surrogate_class_above_half}
\end{equation}
A class distribution cannot assign probability greater than $1/2$ to two distinct classes. Therefore the class $\hat g(\x)$ must be the unique smoothed argmax, and $\hat g(\x)=g(\x;\sigma)$.

On the same event, $\pA(\x;\sigma)\ge \qA(\x)-\delta>1/2$. By the standard randomized-smoothing certificate of \citet{cohen2019certified}, $g(\cdot;\sigma)$ is constant on the ball of radius $\sigma\Phi^{-1}(\pA(\x;\sigma))$ around $\x$. Since $\Phi^{-1}$ is monotone and $\pA(\x;\sigma)\ge \qA(\x)-\delta$, this certified ball contains
\begin{equation}
 \left\{\x':\|\x'-\x\|_2\le \sigma\Phi^{-1}(\qA(\x)-\delta)\right\}
 =
 \left\{\x':\|\x'-\x\|_2\le \Rtilde(\x;\sigma)\right\}.
\end{equation}
The constant value on this smaller ball is $g(\x;\sigma)=\hat g(\x)$.
\end{proof}

\clearpage
\newpage



\end{document}